\documentclass[journal]{IEEEtran}
\ifCLASSINFOpdf
 \usepackage[pdftex]{graphicx}
\else
\fi
\usepackage{graphicx}
\usepackage{amsmath,amssymb}
\usepackage{amsthm}
\usepackage{algorithmic}
\usepackage{array}
\usepackage{multirow}
\usepackage{fixltx2e}
\usepackage{url}
\usepackage{ifpdf}
\usepackage{ragged2e}
\usepackage{lipsum}
\usepackage[colorlinks]{hyperref}
\usepackage{mathtools}
\usepackage{subfigure}
\usepackage{epstopdf}
\DeclarePairedDelimiter\ceil{\lceil}{\rceil}
\DeclarePairedDelimiter\floor{\lfloor}{\rfloor}

\begin{document}

\title{Self-paced Resistance Learning against Overfitting on Noisy Labels}
%
%
%
%

\author{Xiaoshuang~Shi,
		Zhenhua~Guo,
         Kang~Li,
         Yun~ Liang,
         and Xiaofeng~ Zhu
\thanks{X. Shi and X. Zhu are with School of Computer Science and Technology, University of Electronic Science and Technology of China, Chengdu, Sichuan, China, email:(xsshi2013@gmail.com and seanzhuxf@gmail.com) }
\thanks{Z. Guo is with Tsinghua Shenzhen International Graduate School, Tsinghua University, Shenzhen, Guangdong, China, e-mail: (zhenhua.guo@sz.tsinghua.edu.cn).}
\thanks{Kang Li is with West China Medical Center, Sichuan University, Chengdu, Sichuan, Chian, e-mail: (likang@wchscu.cn).}
\thanks{Y. Liang is with the J. Crayton Pruitt Family Department of Biomedical Engineering, University of Florida, Gainesville, FL, USA, e-mail: (yunliang@ufl.edu).}}

\maketitle

\begin{abstract}
Noisy labels composed of correct and corrupted ones are pervasive in practice. They might significantly deteriorate the performance of convolutional neural networks (CNNs), because CNNs are easily overfitted on corrupted labels. To address this issue, inspired by an observation, deep neural networks might first memorize the probably correct-label data and then corrupt-label samples, we propose a novel yet simple self-paced resistance framework to resist corrupted labels, without using any clean validation data. The proposed framework first utilizes the memorization effect of CNNs to learn a curriculum, which contains confident samples and provides meaningful supervision for other training samples. Then it adopts selected confident samples and a proposed resistance loss to update model parameters;  the resistance loss tends to  smooth model parameters' update or attain equivalent prediction over each class, thereby resisting model overfitting on corrupted labels. Finally, we unify these two modules into a single loss function and optimize it in an alternative learning. Extensive experiments  demonstrate the significantly superior performance of the proposed framework over recent state-of-the-art methods on noisy-label data.  \emph{Source codes of the proposed method are available on \textcolor{blue}{https://github.com/xsshi2015/Self-paced-Resistance-Learning}.}

\end{abstract}


\IEEEpeerreviewmaketitle

\vspace{-0.5em}
\section{Introduction}
\label{sec:introduction}

\IEEEPARstart{R}ecently, convolutional neural networks (CNNs) have achieved tremendous success on various different tasks, such as image classification \cite{krizhevsky2012imagenet}  \cite{szegedy2016rethinking} \cite{he2016deep} \cite{Feng_2018_CVPR} \cite{feng2019hypergraph}, retrieval \cite{shi2018pairwise} \cite{shi2020anchor},  detection \cite{girshick2014rich} and segmentation \cite{long2015fully}. However, most CNNs usually require large-scale high-quality labels to obtain desired accuracy, because deep CNNs are capable of memorizing the entire training data even with completely random labels \cite{zhang2016understanding}. This infers that noisy labels might significantly deteriorate the performance of CNNs during training. Unfortunately, noisy labels are pervasive in practice and it is expensive to obtain accurate labeled data. 

To tackle noisy labels for effectively and robustly training CNNs, some methods \cite{reed2014training} \cite{laine2016temporal} utilize regularization terms for label correction to alleviate the deterioration of deep networks during training, but they often fail to attain the optimal accuracy. Another popular way is to estimate a label transition matrix without using regularizations for loss correction \cite{patrini2017making}. However, it is usually difficult to accurately estimate the label transition matrix, especially for a large number of classes. The third promising direction is to select confident samples based on small-loss distances in order to update networks robustly, without estimating the label transition matrix. MentorNet \cite{jiang2017mentornet} and Co-teaching \cite{han2018co} are two representative methods.  When no clean validation data is available, self-paced MentorNet learns a neural network to approximate a predefined curriculum to provide meaningful supervision for StudentNet, so that it can focus on the samples with probably correct labels. Self-paced MentorNet is similar to the self-training method \cite{chapelle2009semi}, and it inherits the same inferiority of accumulated errors generated by sample-selection bias. To address the issue, Co-teaching utilizes the memorization effect of deep neural networks \cite{arpit2017closer}, which might first memorize training data with correct labels and then those with corrupted labels (please refer to Fig. A1 in the supplemental material), and symmetrically trains two networks, each of which filters corrupted labels and selects the samples with small-loss to update the peer network. However, with the increasing number of training epochs, the two networks will gradually form consensus predictions and Co-teaching will functionally deteriorate to self-paced MentorNet. Although the strategy of ``Update by Disagreement'' \cite{malach2017decoupling} can slow down the two networks of Co-teaching to form consensus predictions, it still cannot prevent the effect of sample-selection bias in many cases \cite{yu2019does}. Additionally, when training data is with extremely noisy labels, MentorNet and Co-teaching easily select the corrupt-label data as confident samples so that the networks are overfitted on corrupted labels, thereby decreasing their performance. Moreover, Co-teaching aims to filter corrupt-label training samples and thus might fail to explore their correct semantic information.

To address the performance deterioration of CNNs generated by model overfitting on corrupted labels, and meanwhile explore the correct semantic information of training samples with corrupted labels, in this paper, we propose a novel self-paced resistance framework using a resistance loss to robustly train CNNs on noisy labels, without using any clean validation data. The proposed framework is mainly inspired by: (i) Deep neural networks might first memorize the probably correct-label data and then samples with corrupted labels or outliers \cite{han2018co}; (ii) A curriculum consisting of confident samples can provide meaningful supervision for other training data \cite{jiang2017mentornet}; (iii) A resisting model overfitted on corrupted labels can reduce the deterioration of model performance. We summarize three major contributions as follows:

\begin{itemize}
\item We propose a novel resistance loss to significantly alleviate model overfitting on corrupted labels, by smoothing model parameters' update or attaining equivalent prediction on each class. For clarity, we present the difference between the resistance loss and the traditional cross-entropy loss in Fig. \ref{fig:idea}.  

\item We propose a novel yet simple framework, self-paced resistance learning (SPRL), by effectively using the memorization effect of deep neural networks, curriculum learning and a resistance loss to robustly train CNNs on noisy labels.

\item Extensive experiments on four image datasets demonstrate that (i) The proposed framework can prevent the accuracy deterioration of CNNs on noisy labels, leading to superior classification accuracy over recent state-of-the-art methods on multiple types of label noise; (ii) With clean training data only, the proposed method usually obtains better results than standard networks.
\end{itemize}

The rest of the paper is organized as follows. Section 2 briefly reviews some popular methods to tackle noisy labels; Section 3 introduces the preliminaries on curriculum learning; Section 4 presents the proposed framework, SPRL. Section 5 shows and analyzes experimental results of various methods, and points out the future work; Finally, Section 6 concludes this paper.

\begin{figure}[tbp]
	\includegraphics[trim={0em 0em 0em 0em}, width=0.48\textwidth]{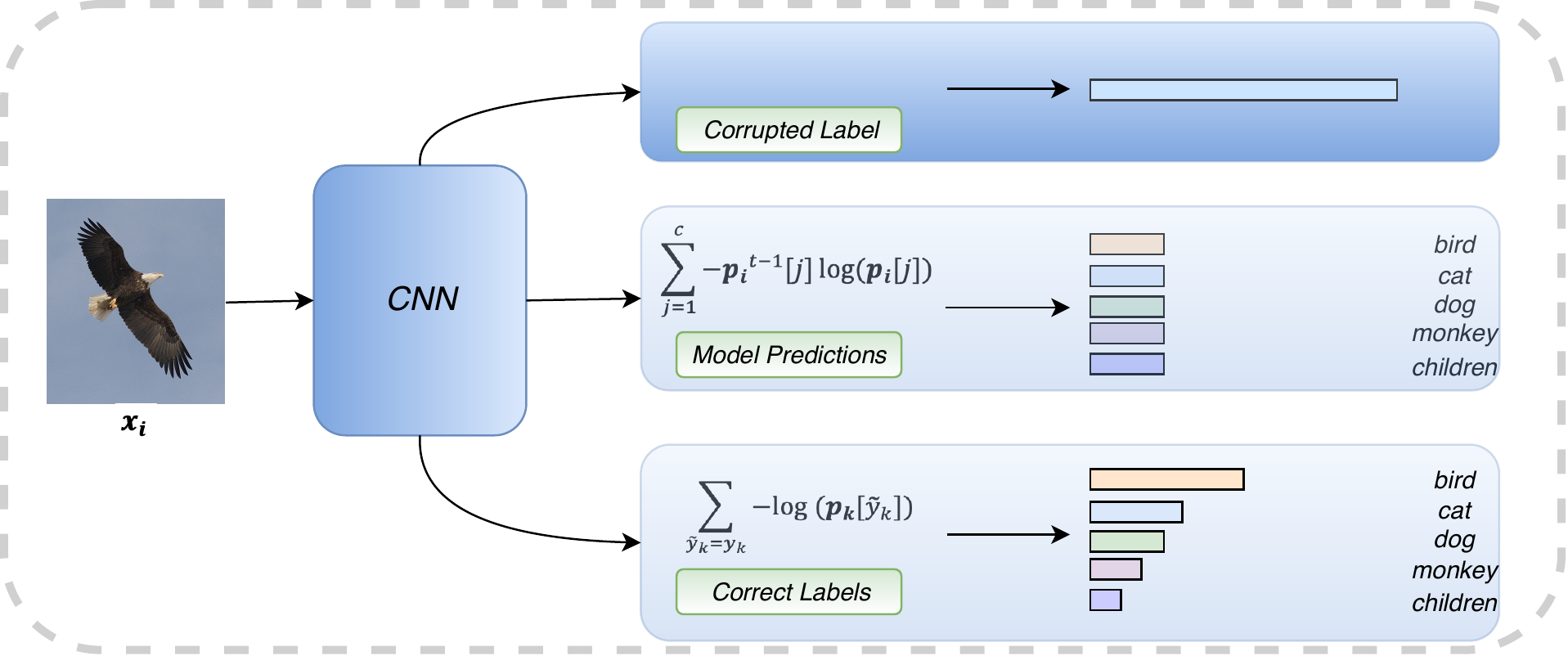}
	\vspace{-1em}
	\caption{The difference between the resistance loss and the traditional cross-entropy loss. The middle is the resistance loss, which employs \(c\) weighted cross-entropy losses to learn model parameters (\(\mathbf{p}_i^{t-1}[j]\) (\(1\leq j\leq c\)) is the weight of each cross-entropy), tending to make the prediction over each class be equivalent.  The top cross-entropy loss utilizes a corrupted label \(\tilde{y}_i\) of \(\mathbf{x}_{i}\) to update model parameters and make a wrong prediction or tends to be overfitting by an outlier. The bottom cross-entropy loss utilizes the samples with correct labels to update model parameters for providing a correct prediction.  \(\mathbf{p}_i^{t-1}\) is the model prediction of the \(i^{th}\) training sample \(\mathbf{x}_{i}\) in the \(t-1^{th}\) epoch before using \(\mathbf{x}_{i}\) to update model parameters, \(\mathbf{p}_i\) represents the model prediction in the \(t^{th}\) epoch, \(t\) is the current number of training epochs, \(c\) is the number of classes, \(y_{k}\) denotes the correct label of the \(k^{th}\) (\(1\leq k \leq n\)) sample and \(n\) is the total number of training samples.} 
	\vspace{-1em}
	\label{fig:idea}
\end{figure}

\vspace{-1em}
\section{Related Work}
\label{sec:relatedwork}
Here, we briefly review some popular statistical learning methods for tackling noisy labels and deep neural networks with noisy labels. 

\textbf{Statistical learning methods.} There are numerous statistical learning algorithms to handle noisy labels \cite{frenay2013classification}. They can be roughly categorized into three groups: probabilistic modeling, surrogate losses and noise rate estimation. One popular probabilistic modeling method is \cite{raykar2010learning}, which proposes a two-coin model to handle noisy labels provided by multiple annotators. For surrogate losses based methods, \cite{natarajan2013learning} proposes an unbiased estimator to provide the noise corrected loss and then presents a weighted loss function for handling class-dependent noisy labels; \cite{masnadi2009design} introduces a robust non-convex loss for tackling the contamination of data with outliers and a boosting algorithm, SavageBoost, to minimize the loss; \cite{van2015learning} presents a convex loss modified from the hinged loss and proves its robustness to symmetric label noise. In the noise rate estimation category, \cite{scott2013classification} designs consistent estimators for classification with asymmetric (class-dependent) label noise; \cite{ramaswamy2016mixture} utilizes kernel embeddings onto reproducing kernel Hilbert space for mixture proportion estimation; \cite{sanderson2014class} estimates class proportions when the distributions of training and test samples are different; \cite{sanderson2014class} and \cite{liu2015classification}  introduce class-probability estimators using order statistics on the range of scores. Most of these statistical learning methods are proposed for traditional algorithms on relatively small datasets. Thus they usually fail to obtain promising performance on real applications, especially large datasets. 

\textbf{Deep neural networks with noisy labels.} Because deep neural networks are sensitive to noisy labels, a few methods have been proposed to handle noisy labels for robust network training. \cite{mnih2012learning} proposes two robust loss functions for binary classification of aerial image patches to handle omission and wrong location of training labels. \cite{ghosh2017robust} \cite{ghosh2015making} \cite{manwani2013noise} investigate noise-tolerant of loss functions under risk minimization. \cite{reed2014training} \cite{laine2016temporal} \cite{shi2020graph} consider the prediction consistency via adding a regularization term for robustly training deep neural networks. This strategy cannot prevent the performance deterioration of CNNs in many cases and it usually fails to obtain optimal accuracy. \cite{sukhbaatar2014training} and \cite{patrini2017making} estimate a label transition matrix, which summarizes the probability of one class being flipped into another, to correct loss functions, and \cite{ma2018dimensionality} employs a dimensionality-driven learning strategy to estimate the correct labels of samples during training and adapt the loss function. However, it is difficult to accurately estimate the label transition matrix or the labels of training samples. \cite{wang2018iterative} proposes an iterative learning framework to handle open-set noisy labels.  \cite{xiao2015learning}\cite{li2017learning} \cite{veit2017learning} and \cite{vahdat2017toward} adopt a small clean dataset to leverage samples with noisy labels; \cite{ren2018learning} adopts a small clean dataset to assign weights for training samples based on their gradient directions to reduce the effect of corrupted labels. These methods usually require an additional clean dataset to alleviate the overfitting of CNNs on noisy labels. \cite{northcutt2017learning} and \cite{zhang2018generalized} adopt the confident samples for training by cleaning up corrupted labels, and thus they fail to exploit the semantic information of the samples with corrupted labels.  \cite{malach2017decoupling} introduces a strategy, ``Update by Disagreement'', that updates the parameters of two networks by using the samples with different predictions. This strategy cannot handle noisy labels explicitly, because the disagreement predictions usually contain corrupted labels. MentorNet \cite{jiang2017mentornet} and Co-teaching \cite{han2018co} are two popular learn-to-teach methods to handle noisy labels. They select confident samples based on small-loss distances to teach the student or other network.  \cite{yu2019does} extends Co-teaching to alleviate the performance deterioration of deep neural networks. However, these learn-to-teach methods easily select corrupt-label samples as confident ones and then make CNNs be overfitted on corrupted labels in many cases, especially on extremely noisy labels (please refer to Fig. A2 in the supplemental material), thereby deteriorating and decreasing the accuracy of CNNs during training. \cite{yao2020searching} formulates the sample selection from noisy labels as a function approximation problem, and proposes a novel Newton algorithm to solve the problem. However, its selection performance is still far from satisfying on extremely noisy labels.

Similar to previous learn-to-teach methods, the proposed method utilizes the memorization effect of deep neural networks to select confident samples as a curriculum to provide supervision of other training samples. However, unlike previous learn-to-teach methods that are very likely to deteriorate with the increasing number of training epochs, the proposed method can prevent the performance degradation during training. This is because the proposed resistance loss can significantly reduce the effect of corrupted labels by alleviating model overfitting. Additionally, the proposed framework does not require the noise rate and only trains a single network, differing from Co-teaching \cite{han2018co} and its variant \cite{yu2019does} that need to know or estimate the rate of label noise and train two networks in a symmetric way. Overall, the proposed method is easy to utilize and can obtain good performance for image classification.

\begin{figure*}[tbp]
\center
\includegraphics[trim={0em 0em 0em 0em}, clip, width=0.98\textwidth]{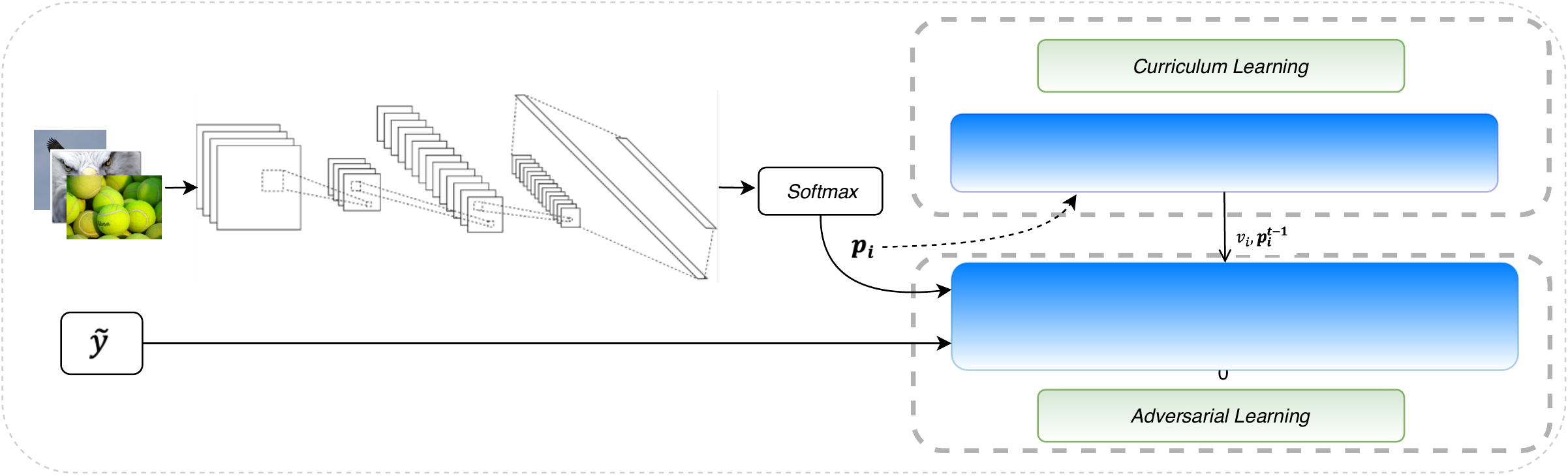}
\vspace{-0.5em}
\caption{The flowchart of the proposed SPRL, which alternatively learns a curriculum \(\mathbf{v}\) and updates model parameters \(\mathbf{w}\) during training. } 
\label{fig:framework}
\vspace{-1em}
\end{figure*}

\section{Preliminaries on Curriculum Learning}

Curriculum learning (CL) \cite{bengio2009curriculum} is a training strategy inspired by the learning process of humans and animals that gradually proceeds easy to difficult samples. CL predetermines the curriculum based on the prior knowledge so that training data is ranked in a meaningful order to facilitate learning. In the following, we briefly introduce three major variants of CL that are related to our proposed method.

\textbf{Self-paced learning (SPL) \cite{kumar2010self}:} CL heavily relies on the prior knowledge and ignores the feedback of the learner (model); to address this issue, SPL dynamically determines the curriculum based on the learner abilities. Given training data \(\mathbf{X}=\left \{ \mathbf{x}_{i} \right \}_{i=1}^n\) and the corresponding labels \(\mathbf{y}= \left \{ y_{i}\right \}_{i=1}^n\), where \(\mathbf{x}_{i}\) and \(y_{i}\) denote the \(i^{th}\) sample and its correct label, respectively. Let \(f(\cdot)\) represent a classifier and \(\mathbf{w}\) be its model parameters. SPL simultaneously selects easy samples and learns model parameters in each iteration by solving the following problem:
\begin{equation}
\begin{array}{cc}
\underset{\mathbf{w},\mathbf{v}}{min}\ E(\mathbf{w},\mathbf{v};\lambda)= \sum_{i=1}^n v_{i}L(y_i,f(\mathbf{x}_i, \mathbf{w})) \\ -\lambda \sum_{i=1}^n v_i, \ \  s.t.\ \mathbf{v}\in \left \{ 0,1 \right \}^n, 
\label{eqn:spl}
\end{array}
\end{equation}
where \(L(y_i,f(\mathbf{x}_i, \mathbf{w}))\) denotes the loss function that calculates the cost between the ground truth label \(y_{i}\) and the estimated label \(f(\mathbf{x}_i, \mathbf{w})\), \(\mathbf{v}\) is a binary vector to indicate which ones are easy samples, and \(\lambda\) is a parameter to control the learning pace. Eq. (\ref{eqn:spl}) is usually solved by an alternative minimization strategy: with fixing \(\mathbf{w}\), calculating \(\mathbf{v}\) by \(\mathbf{v}=\left\{\begin{matrix}
1 & L(\mathbf{x}_i, f(\mathbf{x}_i, \mathbf{w})))<\lambda, \\ 
0 & otherwise.
\end{matrix}\right. \), and then with fixing \(\mathbf{v}\), updating \(\mathbf{w}\) by using selected easy samples to train the classifier \(f(\cdot)\).

\textbf{Self-paced curriculum learning (SPCL) \cite{jiang2015self} :} Although SPL can dynamically learn the curriculum, it does not take into account the prior knowledge. Let \(\Psi\) be a feasible region encoding the information of a predetermined curriculum. To connect CL with SPL, SPCL \cite{jiang2015self} employs both the predetermined curriculum obtained by the prior knowledge before training and the learned curriculum during training with the following model:
\begin{equation}
\begin{array}{cc}
\underset{\mathbf{w},\mathbf{v}}{min}\ E(\mathbf{w},\mathbf{v};\lambda)= \sum_{i=1}^n v_{i}L(y_i,f(\mathbf{x}_i, \mathbf{w})) \\ +G(\mathbf{v}, \lambda), \ \  s.t.\ \mathbf{v}\in \left [ 0,1 \right ]^n, \mathbf{v}\in \Psi,
\end{array}
\label{eqn:spcl}
\end{equation}
where \(\mathbf{v}\) is a weight vector to reflect the significance of samples, and \(G(\cdot)\) is a self-paced function to control the learning scheme. For example, in SPL, \(G(\mathbf{v}, \lambda) =  -\lambda \sum_{i=1}^n v_i\). Similar to Eq. (\ref{eqn:spl}), Eq. (\ref{eqn:spcl}) can also be solved by using an alternative minimization method.

\textbf{Self-paced MentorNet \cite{jiang2017mentornet}:} Because the learning procedure of deep neural networks is very complicated, it is difficult to be accurately modeled by the predefined curriculum. To tackle this issue, \cite{jiang2017mentornet} employs two neural networks, one network called MentorNet \(f_m(\cdot)\) and the other called StudentNet \(f_{s}(\cdot)\).  MentorNet is to approximate a predefined curriculum in order to compute time-varying weights \(f_m(\mathbf{z}_i; \Theta^{\ast}) \in \left [ 0,1 \right ]\) for each training sample, where \(\Theta^{\ast}\) denotes the optimal parameters in \(f_m(\cdot)\), \(\mathbf{z}_{i} = \phi(\mathbf{x}_{i}, \tilde{y}_{i}, \mathbf{w})\) represents the input feature to MentorNet of the \(i^{th}\) sample \(\mathbf{x}_{i}\), \(\tilde{y}_{i}\) is the noisy label of \(\mathbf{x}_{i}\) and \(\mathbf{w}\) is the parameter of StudentNet \(f_{s}(\cdot)\), which will utilize the learned weights \(f_m(\mathbf{z}_i; \Theta^{\ast})\) to update \(\mathbf{w}\). To learn a \(\Theta^{\ast}\), MentorNet minimizes the following function:

\begin{equation}
arg\ \underset{\Theta}{min} \sum_{i=1}^n f_m(\mathbf{z}_i; \Theta) \l_i + G(f_m(\mathbf{z}_i;
\Theta); \lambda),
\end{equation}
where \(\l_i\) is the loss between one hot vector \(\mathbf{\tilde{y}}_{i}\) of the noisy label \(\tilde{y}_{i}\) and a predicting class probability vector \(f_s(\mathbf{x}_i, \mathbf{w})\), which is a discriminative function of StudentNet. Similar to SPCL, \(G(f_m(\mathbf{z}_i;
\Theta); \lambda)\) is a self-paced function.

\vspace{-0.5em}
\section{Self-Paced Resistance Learning (SPRL)}
Although MentorNet can boost model robustness when no clean validation data is used, it easily selects corrupt-label samples as confident ones and then overfits a model on them. To address this problem, we propose a novel training strategy, SPRL. It employs the memorization effect of deep neural networks to approximate a predefined curriculum in order to provide meaningful supervision for other training samples, and adopts a resistance loss to resist the effect of corrupted labels on the network. For clarity, we present the proposed SPRL framework in Fig. \ref{fig:framework}.

\vspace{-0.5em}
\subsection{Curriculum Learning using the Memorization Effect}
Given \(n\) training samples \(\mathbf{X}= \left \{ \mathbf{x}_i \right \}_{i=1}^n \), \(\mathbf{\tilde{y}}= \left \{ \tilde{y}_i \right \}_{i=1}^n \) denotes their corresponding noisy labels, where \(\mathbf{x}_{i}\) is the \(i^{th}\) training sample, \(\tilde{y}_{i}\in \left \{ 1,\cdots, c \right \}\) is its label and \(c\) is the number of classes. To avoid the abuse of symbols, we utilize \(f(\cdot)\) to represent an \(L\)-layer convolutional neural network and \(\mathbf{w}\) to denote model parameters. Let \(\mathbf{P}= \left \{ \mathbf{p}_i \right \}_{i=1}^n \) be label predictions of training samples and \(\emph{B}\) represent the index set of selected training data in each mini-batch, where  \(\mathbf{p}_{i} =f(\mathbf{x}_{i} , \mathbf{w})\in \mathbb{R}^{c}\) is the label prediction of the sample \(\mathbf{x}_{i}\). To update model parameters, we adopt the cross-entropy loss function as follows:
\begin{equation}
\underset{\mathbf{w}}{min}\  \frac{1}{\left | \emph{B}\right |} \sum_{i\in \emph{B}} -log(\mathbf{p}_{i}[\tilde{y}_i]),
\label{eqn:ce_loss}
\end{equation} 
where \(\left | \emph{B} \right |\) denotes the length of the index set \(\emph{B}\).

Suppose that we train the network for \(T\) epochs in total. When we only utilize Eq. (\ref{eqn:ce_loss}) to update model parameters during training, the model performance usually deteriorates after a few epochs, because the network might first memorize the correct and easy samples at initial epochs and then it will eventually overfit on the corrupted labels or outliers \cite{han2018co}. Based on this memorization of deep networks, we first run the model \(T_{1}\) epochs and then select \(m\) samples based on the small-loss distances to construct a predefined curriculum, which contains the probably correct data. Afterwards, we gradually add a number of samples into the curriculum every a few epochs for training. For clarity, we formulate this procedure as the following model:

\begin{equation}
\begin{array}{cc}
\underset{\mathbf{w}, \mathbf{v}}{min}\  \frac{1}{\sum_{i\in \emph{B}} v_i} \sum_{i\in \emph{B}} -v_{i}log(\mathbf{p}_{i}[\tilde{y}_i]) - \lambda v_{i}, \\ s.t.\ \mathbf{v}\in \left \{ 0,1 \right \}^n, \sum_{i=1}^n v_{i}=\delta (t),
\end{array}
\label{eqn:ce_curriculum}
\end{equation} 
where \(\delta (t)\) is a piecewise linear function to determine how many training samples are added into the curriculum and \(t\) is the current number of training epochs. Note that there are many possibilities to set \(\delta (t)\). To make confident samples play a better role in model training, we define it as:
\begin{equation}
\delta(t)=\left\{\begin{matrix}
n & t\leq T_1 \\ 
min(m+\floor{\frac{t-T_1}{\floor{\frac{T-T_1}{K-\frac{mK}{n}+1}}}}\floor{\frac{n}{K}}, n) & T_1 < t \leq T
\end{matrix}\right.
\label{eqn:vt}
\end{equation}
where \(K\in \mathbb{Z}\) is the number of subsets, each of which contains some training samples. Because the number of samples is usually much larger than the number of epochs, we add the subset into the curriculum in each epoch so that all training samples can be added in the curriculum during the training process. Eq. (\ref{eqn:vt}) suggests that Eq. (\ref{eqn:ce_curriculum}) is equivalent to Eq. (\ref{eqn:ce_loss}) when \(t\leq T_1\).

\begin{figure*}[tbp]
\center
\includegraphics[trim={0em 0em 0em 0em}, clip, width=0.8\textwidth]{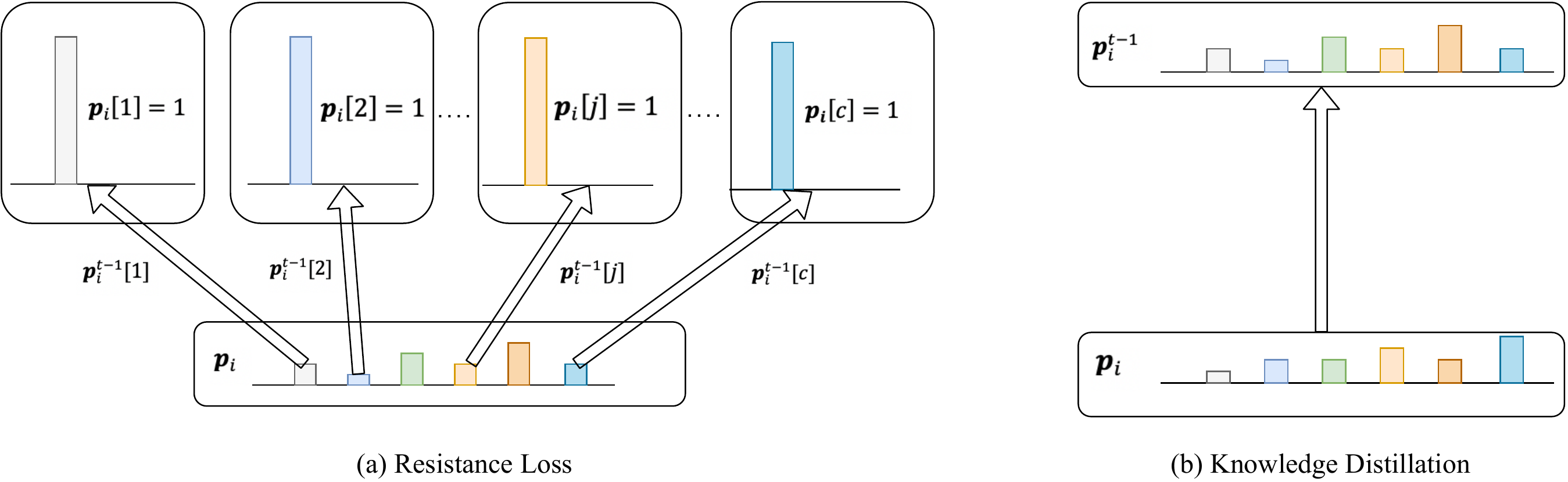}
\vspace{-0.5em}
\caption{The core idea of the proposed resistance loss and its difference from knowledge distillation. (a) Resistance loss, it contains \(c\) weighted cross-entropy losses (see Eq. (\ref{eqn:ce_ad})), \(\mathbf{p}_i^{t-1}\) is the weight, \(\mathbf{p}_{i}\) in the bottom row is the model prediction, and \(\mathbf{p}_{i}[j]=1\) (\(1\leq j\leq c\)) in the top row is the target; (b) Knowledge distillation, using the prediction \(\mathbf{p}_{i}^{t-1}\) of previous training epoch or a peer model as a teacher (see Eq. (\ref{eqn:kl}) in Section V.C). } 
\label{fig:RLKD}
\vspace{-1em}
\end{figure*}

\vspace{-0.5em}
\subsection{Resistance Loss}
Directly using Eq. (\ref{eqn:ce_curriculum}) to train CNNs is similar to self-paced learning, and the model will gradually overfit on corrupted labels, with the increasing number of training epochs. To address this problem, we propose a resistance loss using the cross entropy between model predictions of previous and current training epochs in each mini-batch. Because knowledge distillation methods \cite{hinton2015distilling} \cite{ashok2017n2n} \cite{polino2018model} \cite{lee2019overcoming} \cite{dong2019distillation} \cite{kim2020self} using model predictions as the teacher can also alleviate model overfitting, we present Fig. \ref{fig:RLKD} to illustrate the core idea of the proposed resistance loss and their differences.

Suppose that \(\mathbf{p}_{i}[j]\) is the label prediction of the sample \(\mathbf{x}_{i}\) belonging to the \(j^{th}\) class in the \(t^{th}\) training epoch, and \(\mathbf{p}_{i}^{t-1}[j] \in \mathbf{p}_{i}^{t-1} \) is the label prediction before using \(\mathbf{x}_{i}\) to update model parameters in the \(t-1^{th}\)  training epoch. We propose the resistance loss as follows:
\begin{equation}
\underset{\mathbf{w}}{min} \frac{1}{\left | \emph{B} \right |} \sum_{i\in \emph{B}} \sum_{j=1}^c  -\mathbf{p}_{i}^{t-1}[j] log(\mathbf{p}_{i}[j]).
\label{eqn:ce_ad}
\end{equation} 
Eq. (\ref{eqn:ce_ad}) is used to resist model overfitting of CNNs on corrupted labels for boosting model robustness. It is mainly inspired by: (i) Eq. (\ref{eqn:ce_ad}) might smooth the update of model parameters; (ii) Eq. (\ref{eqn:ce_ad}) tends to make \(\mathbf{p}_{i}[j]\rightarrow \frac{1}{c}\) (\(1\leq j\leq c\)). Let  \(\mathbf{p}_{i}^{t}\) be the label prediction of \(\mathbf{x}_{i}\) before using it to update model parameters in the \(t^{th}\) training epoch, to better illustrate these two motivations, we present Proposition 1 and show its proof in the following.

\newtheorem{ptheorem}{Proposition}
\begin{ptheorem}
Suppose that solving the problem in Eq. (\ref{eqn:ce_ad}) with gradient descent,  for any two entries \(\mathbf{p}_{i} [j], \mathbf{p}_{i}[k] \in \mathbf{p}_{i}\) and \(\mathbf{p}_{i}^t[j]> \mathbf{p}_{i}^t[k]\). There are only three cases between \(\frac{\mathbf{p}_{i}^t[j]}{\mathbf{p}_{i}^t[k]}\) and  \(\frac{\mathbf{p}_{i}^{t-1}[j]}{\mathbf{p}_{i}^{t-1}[k]}\): (i)  \(\frac{\mathbf{p}_{i}^{t-1}[j]}{\mathbf{p}_i^{t-1}[k]}<\frac{\mathbf{p}_{i}^{t}[j]}{\mathbf{p}_i^{t}[k]} \);  (ii)  \(\frac{\mathbf{p}_{i}^{t-1}[j] }{\mathbf{p}_{i}^{t-1}[k] } > (\frac{\mathbf{p}_{i}^{t}[j] }{\mathbf{p}_{i}^{t}[k] })^2\); (iii) \(  \frac{\mathbf{p}_{i}^{t}[j]}{\mathbf{p}_i^{t}[k]} \leq \frac{\mathbf{p}_{i}^{t-1}[j]}{\mathbf{p}_i^{t-1}[k]} \leq (\frac{\mathbf{p}_{i}^t[j]}{\mathbf{p}_{i}^t[k]})^2\). For case (i) and (ii), there exists \(\frac{\mathbf{p}_{i}[j]}{\mathbf{p}_i[k]} < \frac{\mathbf{p}_{i}^t[j]}{\mathbf{p}_i^t[k]}\) and \(\frac{\mathbf{p}_{i}^{t-1}[j]}{\mathbf{p}_{i}^{t-1}[k]}>  \frac{\mathbf{p}_{i}[j]}{\mathbf{p}_{i}[k]}> \frac{\mathbf{p}_{i}^t[j]}{\mathbf{p}_{i}^t[k]}\), respectively, thereby smoothing the update of model parameters; for case (iii), there exists \(\frac{\mathbf{p}_{i}[j]}{\mathbf{p}_{i}[k]} \leq \frac{\mathbf{p}_{i}^t[j]}{\mathbf{p}_{i}^t[k]}\leq  \frac{\mathbf{p}_{i}^{t-1}[j]}{\mathbf{p}_{i}^{t-1}[k]}\), upon which each entry in \(\mathbf{p}_{i}\) tends to be gradually equivalent, i.e. \(\mathbf{p}_i[j] = \mathbf{p}_i[k] = \frac{1}{c}\), \(\forall\)  \(1\leq j,k\leq c\).
\label{theorem:pro1}
\end{ptheorem}

\begin{proof}
Let \(E(\mathbf{p}_i) = \sum_{j=1}^c- \mathbf{p}_i^{t-1}[j]log(\mathbf{p}_i[j])\), taking its derivative with respect to (w.r.t) \(\mathbf{p}_i[j]\), we have:
\begin{equation} 
\frac{\partial E(\mathbf{p}_i)}{\partial \mathbf{p}_i[j]} = -\frac{\mathbf{p}_i^{t-1}[j]}{\mathbf{p}_i [j]},
\label{eqn:epj}
\end{equation}
which means that \(\bigtriangledown E(\mathbf{p}_i^{t}[j]) = -\frac{\mathbf{p}_i^{t-1}[j]}{\mathbf{p}_i^t [j]}\). Here, the entries in \(\mathbf{p}_{i}\) are independent, because Eq. (\ref{eqn:ce_ad}) is used as \(c\) weighted cross-entropy losses (please refer to Fig. \ref{fig:RLKD}a).  Note that in this paper, \(\log\) utilizes \(e\) as its base.  

If \(c=1\), then \(\mathbf{p}_{i}[j] = \mathbf{p}_{i}^{t}[j]+\eta \frac{\mathbf{p}_i^{t-1}[j]}{\mathbf{p}_i^t [j]}\), where \(\eta\) denotes the learning rate. Because \(\mathbf{p}_i^{t-1}[j]>0\), \(\mathbf{p}_i^t [j]>0\) and \(\eta>0\), \(\mathbf{p}_{i}[j]\) will gradually approximate to 1, i.e. \(-\mathbf{p}_i^{t-1}[j]log(\mathbf{p}_i[j])\rightarrow  0\).

If \(c>1\), for any two entries \(\mathbf{p}_i^{t}[j]> \mathbf{p}_i^{t}[k]\), \(1\leq j, k \leq c\), then when \(\eta>0\), there exists:
\begin{equation}
\frac{\mathbf{p}_{i}[j]}{\mathbf{p}_{i}[k]} = \frac{\mathbf{p}_i^{t}[j]+\eta \frac{\mathbf{p}_{i}^{t-1}[j]}{\mathbf{p}_i^t[j]}}{\mathbf{p}_i^{t}[k]+\eta \frac{\mathbf{p}_i^{t-1}[k]}{\mathbf{p}_i^t[k]}}.
\label{eqn:pc2}
\end{equation}

Based on Eq. (\ref{eqn:pc2}), when \(\frac{\mathbf{p}_{i}^{t-1}[j]}{\mathbf{p}_i^{t-1}[k]}<\frac{\mathbf{p}_{i}^{t}[j]}{\mathbf{p}_i^{t}[k]}\), there exists \(\frac{\frac{\mathbf{p}_{i}^{t-1}[j]}{\mathbf{p}_i^t[j]}}{\frac{\mathbf{p}_i^{t-1}[k]}{\mathbf{p}_i^t[k]}} <1\), leading to \(\frac{\mathbf{p}_{i}[j]}{\mathbf{p}_i[k]} < \frac{\mathbf{p}_{i}^t[j]}{\mathbf{p}_i^t[k]}\), thereby smoothing the update of model parameters. In addition, if \(\frac{\mathbf{p}_{i}^{t-1}[j]}{\mathbf{p}_i^{t-1}[k]}<1\), there exists \(\frac{\mathbf{p}_{i}^{t-1}[j]}{\mathbf{p}_i^{t-1}[k]}<\frac{\mathbf{p}_{i}[j]}{\mathbf{p}_i[k]} < \frac{\mathbf{p}_{i}^t[j]}{\mathbf{p}_i^t[k]}\).

When  \(\frac{\mathbf{p}_{i}^{t-1}[j] }{\mathbf{p}_{i}^{t-1}[k] } > (\frac{\mathbf{p}_{i}^{t}[j] }{\mathbf{p}_{i}^{t}[k] })^2\), i.e., \(\frac{\frac{\mathbf{p}_{i}^{t-1}[j]}{\mathbf{p}_i^t[j]}}{\frac{\mathbf{p}_i^{t-1}[k]}{\mathbf{p}_i^t[k]}} > \frac{\mathbf{p}_i^{t}[j]}{\mathbf{p}_i^{t}[k]}\), it has \(\frac{\mathbf{p}_i[j]}{\mathbf{p}_{i}[k]}>\frac{\mathbf{p}_{i}^t[j]}{\mathbf{p}_i^t[k]}\). Eq. (\ref{eqn:pc2}) equals
\(\frac{\mathbf{p}_{i}[j]}{\mathbf{p}_{i}[k]} = \frac{(\mathbf{p}_i^{t}[j])^2+\eta \mathbf{p}_{i}^{t-1}[j] } {(\mathbf{p}_i^{t}[k])^2+\eta \mathbf{p}_i^{t-1}[k]} \cdot \frac{\mathbf{p}_i^{t}[k]}{\mathbf{p}_{i}^t[j]}\) and  \(\frac{\mathbf{p}_i^{t}[k]}{\mathbf{p}_{i}^t[j]}<1\), so they suggest \(\frac{\mathbf{p}_{i}[j]}{\mathbf{p}_{i}[k]} < \frac{(\mathbf{p}_i^{t}[j])^2+\eta \mathbf{p}_{i}^{t-1}[j] } {(\mathbf{p}_i^{t}[k])^2+\eta \mathbf{p}_i^{t-1}[k]} <\frac{\mathbf{p}_{i}^{t-1}[j]}{\mathbf{p}_{i}^{t-1}[k]}\). Thus, \(\frac{\mathbf{p}_{i}^{t-1}[j]}{\mathbf{p}_{i}^{t-1}[k]}>  \frac{\mathbf{p}_{i}[j]}{\mathbf{p}_{i}[k]}> \frac{\mathbf{p}_{i}^t[j]}{\mathbf{p}_{i}^t[k]}\), which means model parameters' update would be smoothed.

When  \(\frac{\mathbf{p}_{i}^{t-1}[j] }{\mathbf{p}_{i}^{t-1}[k] } \leq (\frac{\mathbf{p}_{i}^{t}[j] }{\mathbf{p}_{i}^{t}[k] })^2\), i.e., \(\frac{\frac{\mathbf{p}_{i}^{t-1}[j]}{\mathbf{p}_i^t[j]}}{\frac{\mathbf{p}_i^{t-1}[k]}{\mathbf{p}_i^t[k]}} \leq \frac{\mathbf{p}_i^{t}[j]}{\mathbf{p}_i^{t}[k]}\), it has \(\frac{\mathbf{p}_{i}[j]}{\mathbf{p}_i[k]} \leq \frac{\mathbf{p}_{i}^t[j]}{\mathbf{p}_i^t[k]}\). With an additional constraint \(\frac{\mathbf{p}_{i}^{t-1}[j] }{\mathbf{p}_{i}^{t-1}[k] } \geq \frac{\mathbf{p}_{i}^{t}[j] }{\mathbf{p}_{i}^{t}[k] }\),  it means  \(\frac{\mathbf{p}_{i}[j]}{\mathbf{p}_i[k]} \leq \frac{\mathbf{p}_{i}^t[j]}{\mathbf{p}_i^t[k]}\leq \frac{\mathbf{p}_{i}^{t-1}[j]}{\mathbf{p}_i^{t-1}[k]} \). In this case, each entry in \(\mathbf{p}_{i}\) will gradually becomes equivalent, i.e. \(\mathbf{p}_{i}[j]=\mathbf{p}_{i}[k]\). With a constraint \(\sum_{j=1}^c \mathbf{p}_{i}^t[j]=1\), there will be \(\mathbf{p}_{i}[j]=\mathbf{p}_{i}[k]\rightarrow \frac{1}{c}\).

\noindent Therefore, Proposition \ref{theorem:pro1} is proved.
\end{proof}

\begin{figure}[tbp]
	\center
	\subfigure[MNIST]{\includegraphics[trim={3em 1em 1em 3em}, clip, width=0.24\textwidth]{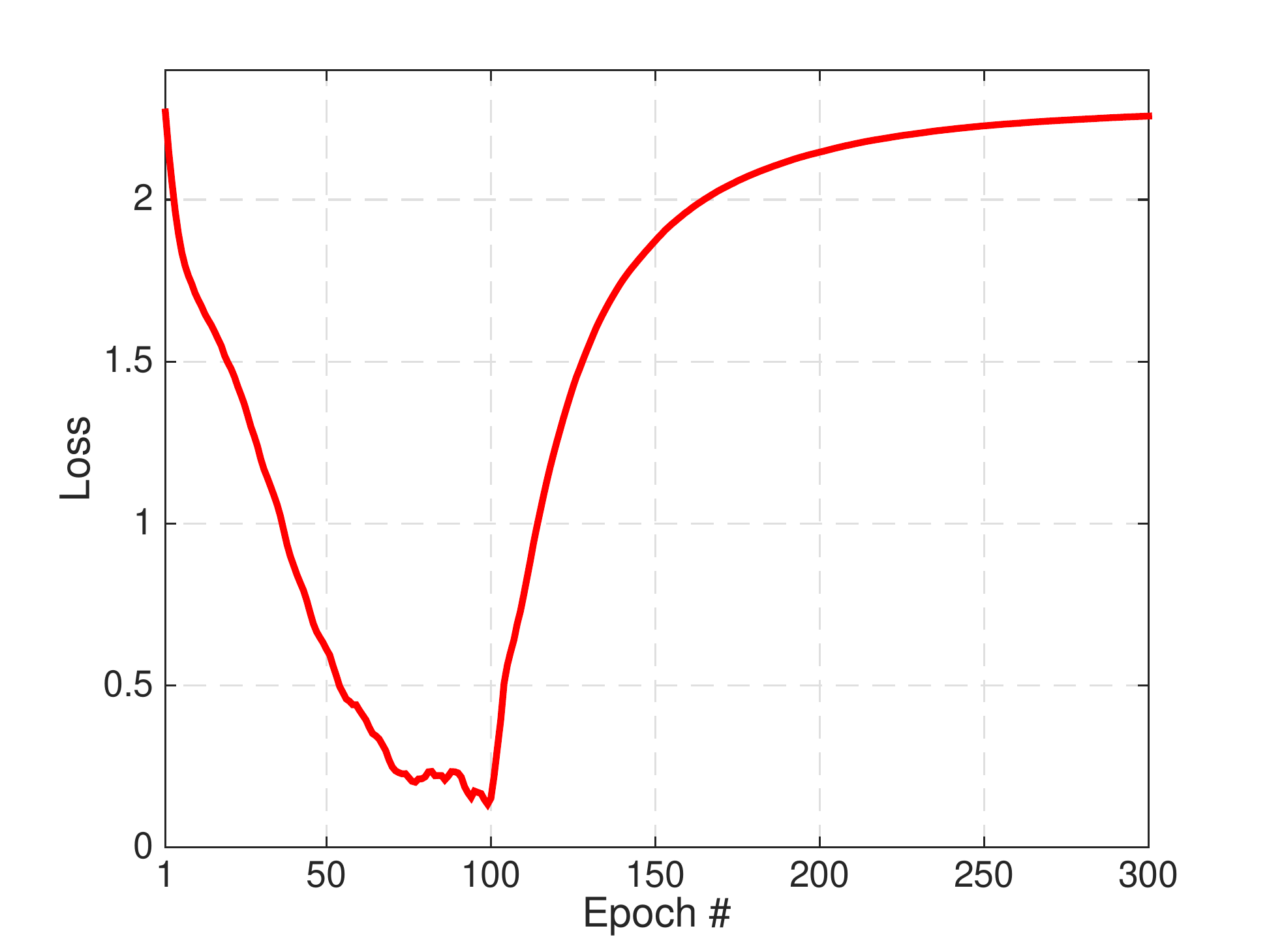}}
	\subfigure[CIFAR10]{\includegraphics[trim={3em 1em 1em 3em}, clip, width=0.24\textwidth]{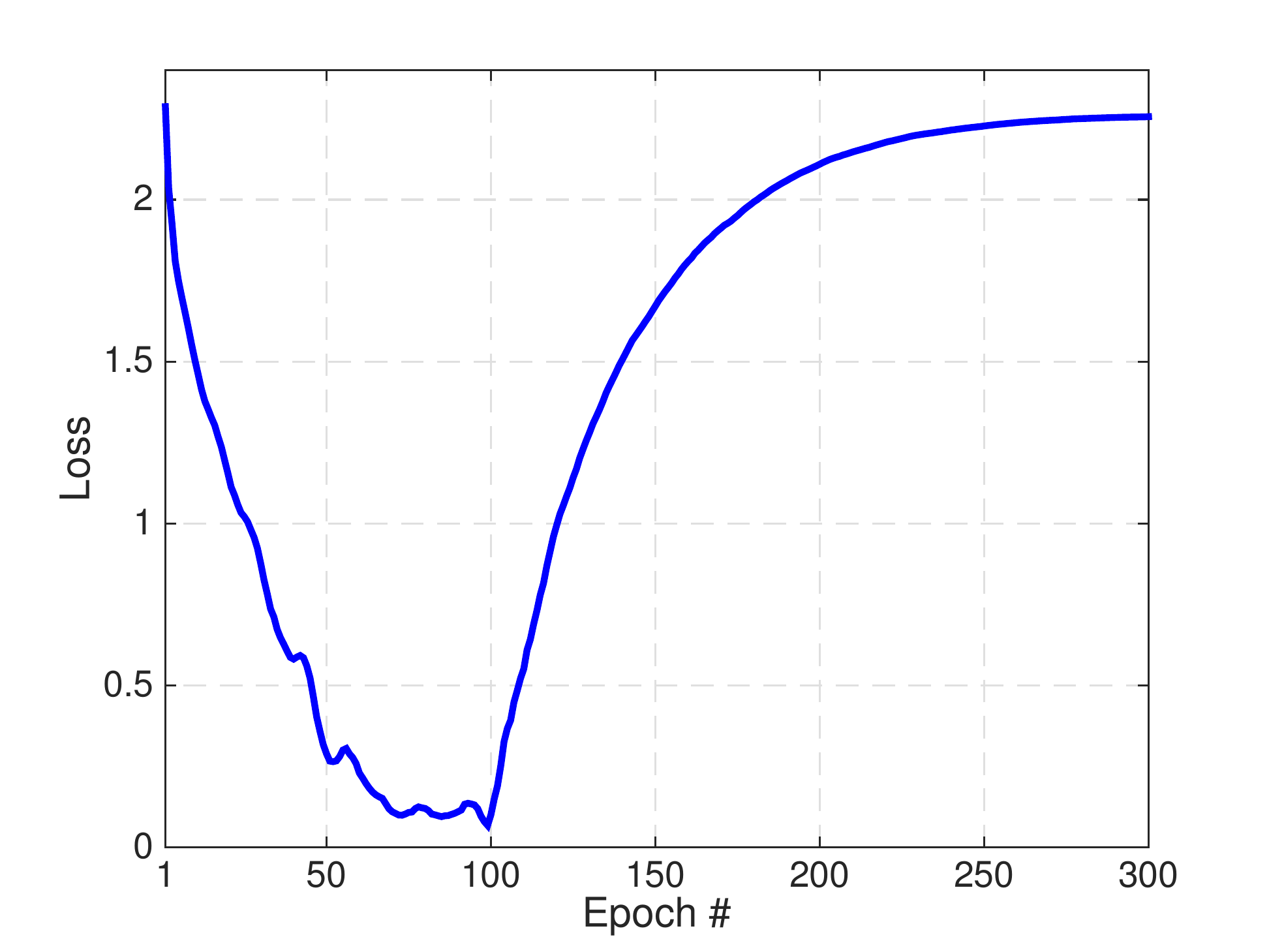}}
	\vspace{-0.5em}
	\caption{The loss of Eq. (\ref{eqn:ce_loss}) and Eq. (\ref{eqn:ce_ad}) change with the number of training epochs by using random 1000 digits from `0' to `9' in MNIST \cite{lecun1998gradient} and random 1000 images belonging to 10 categories from CIFAR10 \cite{krizhevsky2009learning} . We first utilize Eq. (\ref{eqn:ce_loss}) to train ResNet18 \cite{he2016deep} with 100 epochs, and then adopt Eq. (\ref{eqn:ce_ad}) to train the network for the subsequent 200 epochs. The loss gradually approximates to \(\sum_{j=1}^{10}-0.1*log(0.1)=2.303\). } 
	\vspace{-1em}
	\label{fig:ce}
\end{figure}

In practice, because cases (i) and (ii) in Proposition 1 smoothly update model parameters, the relationship between \(\frac{\mathbf{p}_{i}^t[j]}{\mathbf{p}_{i}^t[k]}\) and  \(\frac{\mathbf{p}_{i}^{t-1}[j]}{\mathbf{p}_{i}^{t-1}[k]}\) might gradually satisfy the case (iii), thereby causing  each entry in \(\mathbf{p}_{i}\) gradually to be equivalent. For clarity, Fig. \ref{fig:ce} presents two examples to show the change of the objective in Eq. (\ref{eqn:ce_ad}) from 101 to 300 epochs during training, where \(\mathbf{p}_{i}[j]\) and \(\mathbf{p}_{i}[k]\) (\(\forall\)  \(1\leq j,k\leq c\)) are gradually equivalent, i.e. \(\mathbf{p}_i[j] = \mathbf{p}_i[k] \rightarrow \frac{1}{c}\), so that the objective of Eq. (\ref{eqn:ce_ad}) becomes larger. This infers the case (iii) in Proposition \ref{theorem:pro1}, i.e, Eq. (\ref{eqn:ce_ad}) can gradually make each entry of probability ratios be equivalent.  Moreover, Fig. \ref{fig:diff} shows an example to display the change of prediction probability of training data on their true classes when using Eq. (\ref{eqn:ce_loss}), Eq. (\ref{eqn:ce_curriculum}) and Eq. (\ref{eqn:loss}) (Eq. (\ref{eqn:ce_curriculum})+Eq. (\ref{eqn:ce_ad})).  Fig. \ref{fig:diff} presents that the curve in Fig. \ref{fig:diff}c is smoother than that in Fig. \ref{fig:diff}a-b. This suggests that Eq. (\ref{eqn:ce_ad}) can smooth the update of model parameters. Fig. \ref{fig:diff}c also illustrates that SPRL using Eq. (\ref{eqn:ce_ad}) can resist model overfitting on corrupted labels. Note that SPRL does not distinguish correct and corrupted labels during training, thereby causing the decrease of prediction probability on correct labels in Fig. \ref{fig:diff}c.

\begin{figure*}[htbp]
	\center
	\subfigure[Standard]{\includegraphics[trim={3em 0em 4em 0em}, clip, width=0.32\textwidth]{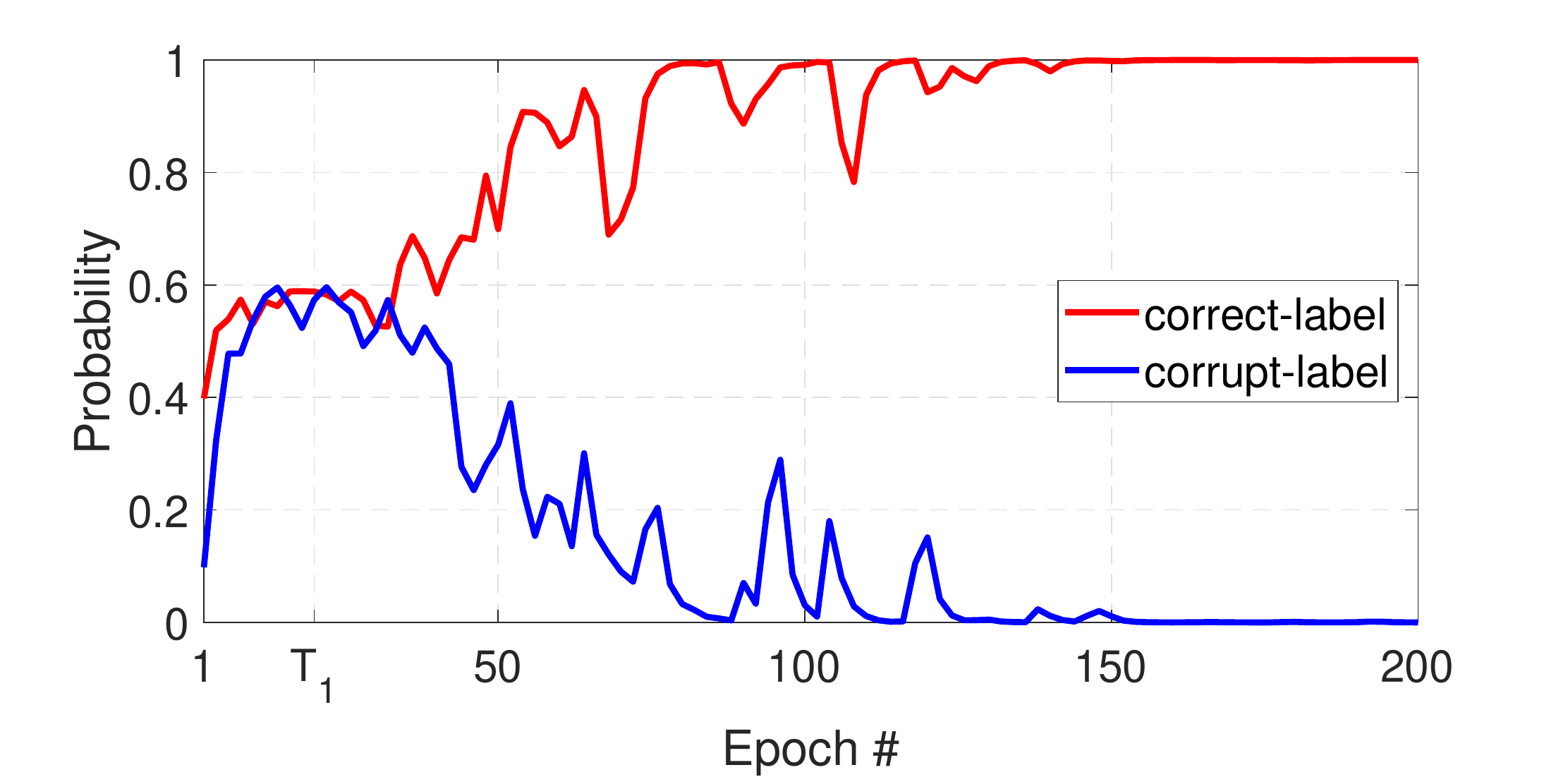}}
	\subfigure[CL]{\includegraphics[trim={3em 0em 4em 0em}, clip, width=0.32\textwidth]{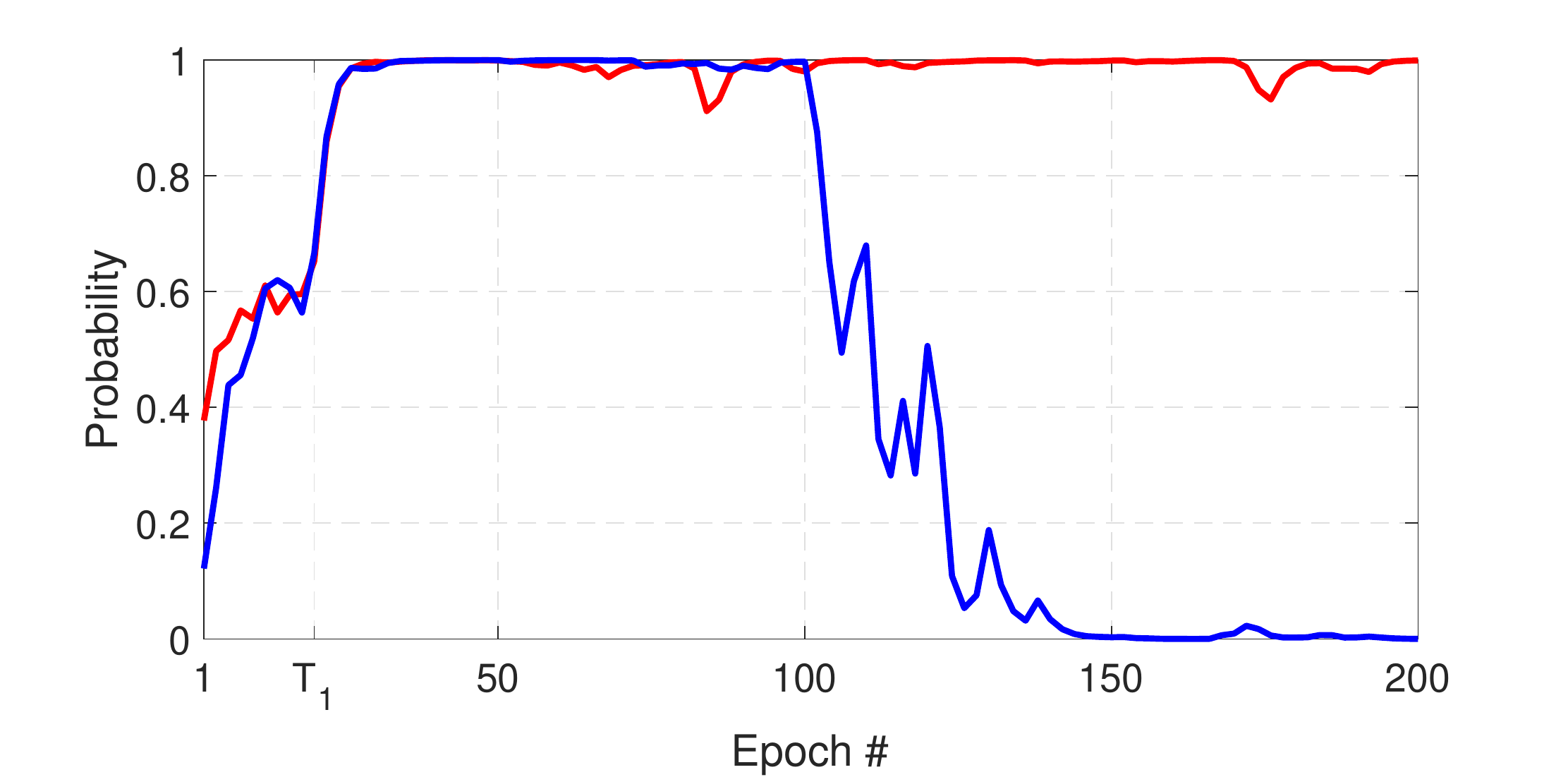}}
	\subfigure[SPRL]{\includegraphics[trim={3em 0em 4em 0em}, clip, width=0.32\textwidth]{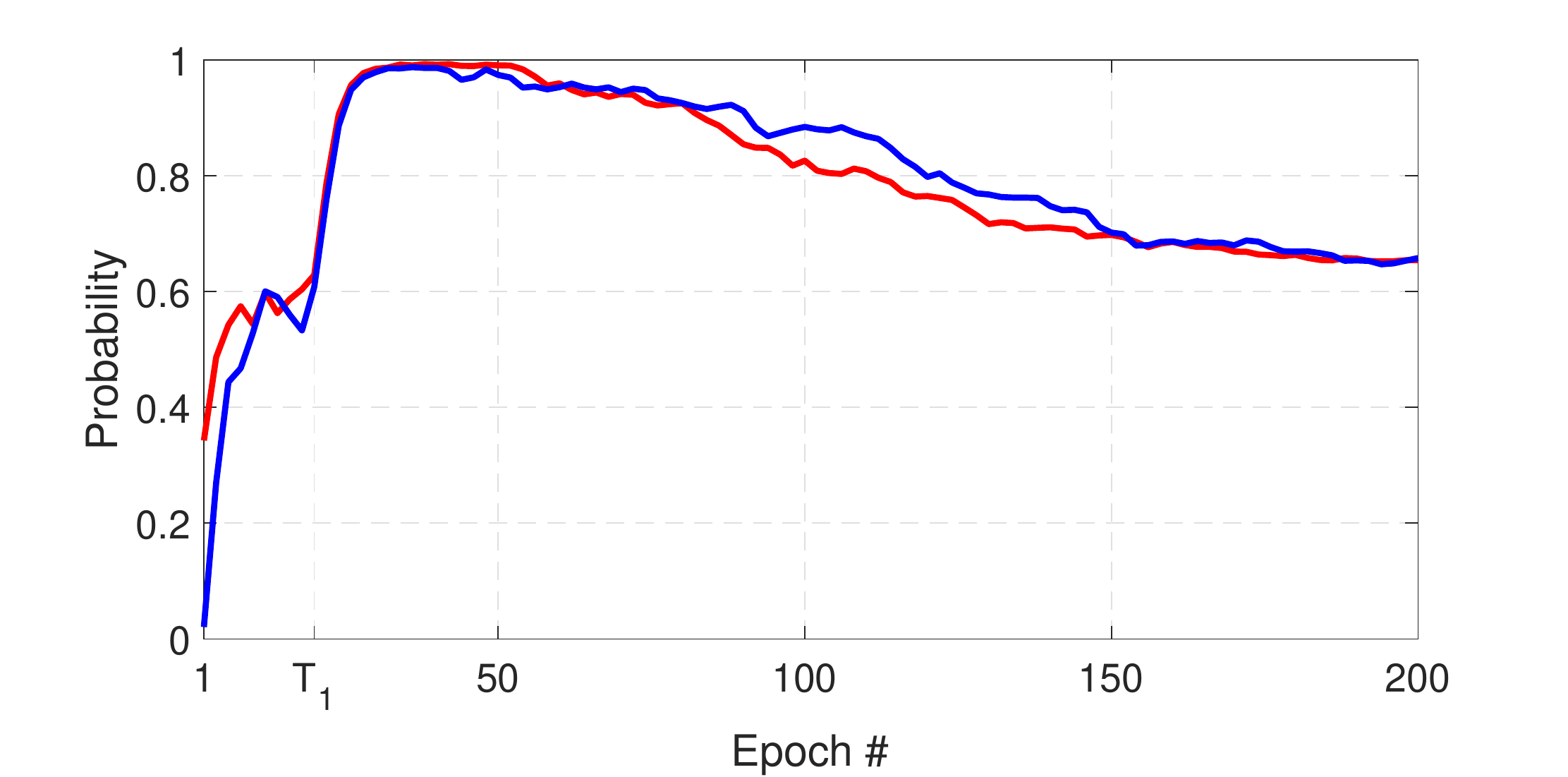}}
	\vspace{-0.5em}
	\caption{The prediction probability of training data on their true classes with respect to training epochs when using three methods: (a) Standard, (b) CL, (c) SPRL. They adopt Eq. (\ref{eqn:ce_loss}), Eq. (\ref{eqn:ce_curriculum}) and Eq. (\ref{eqn:loss}) to train ResNet18 \cite{he2016deep} on noisy-label data, respectively.  We randomly select training data from CIFAR-10 \cite{krizhevsky2009learning} and flip their labels using Eq. (\ref{subeqn:pair}).} 
	\vspace{-1em}
	\label{fig:diff}
\end{figure*}

\vspace{-0.5em}
\subsection{Self-paced Resistance Loss}
Based on the learned curriculum and the proposed resistance loss, we can obtain the loss function of the proposed framework. Specifically, combining Eq. (\ref{eqn:ce_curriculum}) with Eq. (\ref{eqn:ce_ad}), we have:
\begin{equation}
\begin{array}{ccc}
\underset{\mathbf{w}, \mathbf{v}}{min}\ E(\mathbf{w}, \mathbf{v};\lambda)= \frac{1}{\sum_{i\in \emph{B}} v_{i}}\sum_{i\in \emph{B}} -v_{i}(log(\mathbf{p}_{i}[\tilde{y}_i]) +\lambda)\\+ \frac{\gamma(t)}{\left |\emph{B} \right |} \sum_{i\in \emph{B}}\sum_{j=1}^c  -\mathbf{p}_{i}^{t-1}[j] log(\mathbf{p}_{i}[j]), \\ s.t.\ \mathbf{v}\in \left \{ 0,1 \right \}^n, \sum_{i=1}^n v_{i}=\delta (t),
\end{array}
\label{eqn:loss}
\end{equation}
where \(\gamma(t)\) is a time-dependent weighting function to gradually enhance the weight of model predictions with the increasing number of epochs, so that Eq. (\ref{eqn:ce_ad}) is mainly used to prevent model overfitting on corrupted labels. Because deep neural networks might first memorize the correct-label data and then corrupt-label samples, and the noise rate of selected samples usually increases eventually. 

There are many choices for \(\gamma(t)\). Similar to the popular ramp-up function in \cite{laine2016temporal}, we utilize the following function:
\begin{equation}
\gamma(t)=\left\{\begin{matrix}
0 & t\leq T_1 \\ 
\gamma_{max} e^{-5\left \| 1-\mu \right \|_F^2} & T_1 < t \leq T,
\end{matrix}\right.
\label{eqn:gammat}
\end{equation}
where \(\mu\) linearly ramps up from 0 to 1 during \(T-T_1\) epochs, \(\gamma_{max}\) is the maximum of  \(\gamma(t)\) depending on \(m\), e.g. \(\gamma_{max} = \gamma_d(10-\ceil{\frac{m}{0.1n}})\). This is because a larger \(\gamma_{max}\) is required for a larger noise rate. \

The optimization of Eq. (\ref{eqn:loss}) is similar to that of Eq. (\ref{eqn:spl}) and Eq. (\ref{eqn:spcl}), and thus we solve it by utilizing an alternative minimization strategy  \cite{kumar2010self}  \cite{jiang2015self} . Specifically, it can be divided into two sub-problems:
\begin{subequations}\label{eqn:subps}
 \begin{align}
 \begin{array}{cc}
     \underset{\mathbf{v}}{min} \sum_{i\in \emph{B}} -v_{i}log(\mathbf{p}_{i}^{t-1}[\tilde{y}_i]) - \lambda v_{i}, \\
      s.t. \ \mathbf{v}\in \left \{ 0,1 \right \}^n,  \sum_{i=1}^n v_{i}=\delta (t-1). \label{eqn:sub1} 
 \end{array} \\
 \begin{array}{cc}
  \underset{\mathbf{w}}{min}\frac{1}{\sum_{i \in \emph{B}}v_i} \sum_{i\in \emph{B}} -v_{i}log(\mathbf{p}_{i}[\tilde{y}_i])\\ + \frac{\gamma(t)}{\left | B \right |}\sum_{i\in \emph{B}}\sum_{j=1}^c  -\mathbf{p}_{i}^{t-1}[j] log(\mathbf{p}_{i}[j]).
  \label{eqn:sub2}
 \end{array}
 \end{align}
\end{subequations}
Eq. (\ref{eqn:sub1}) is a \(\mathbf{v}\)-subproblem, in which the model parameter \(\mathbf{w}\) is known, and it aims to learn a curriculum consisting of confident samples; Eq. (\ref{eqn:sub2}) is a \(\mathbf{w}\)-subproblem, which consists of a cross-entropy loss to utilize selected confident samples to update model parameters, and a resistance loss to resist model overfitting of CNNs on corrupted labels. We alternatively solve Eq. (\ref{eqn:sub1}) and Eq. (\ref{eqn:sub2}), i.e. fixing \(\mathbf{w}\), based on Eq. (\ref{eqn:sub1}), we can calculate \(\mathbf{v}\) as follows:
\begin{equation}
v_{i}^{\ast}=\left\{\begin{matrix}
1 & if -log(\mathbf{p}_i^{t-1}[\tilde{y}_i])< \lambda\\ 
0 & otherwise. 
\end{matrix}\right.
\label{eqn:v}
\end{equation}
Note that in each epoch we might need to adjust \(\lambda\) so that \(\sum_{i=1}^n v_i = \delta (t-1)\). Then with a fixed \(\mathbf{v}\), we can update the model parameter \(\mathbf{w}\) by solving Eq. (\ref{eqn:sub2}) via any optimizer, e.g. Adam \cite{kingma2014adam}. In summary, we present the detailed procedure to solve Eq. (\ref{eqn:loss}) in Algorithm 1.

\begin{table}[tb]
\small
\begin{tabular}{l}
\hline
\textbf{Algorithm 1: SPRL}         \\
\hline
\textbf{Input:} Training data \(\mathbf{X}=\left \{ \mathbf{x}_{i} \right \}_{i=1}^{n}\), noisy labels \(\mathbf{\tilde{y}}=\left \{ \tilde{y}_{i} \right \}_{i=1}^{n}\),  \\
number of training epochs: \(T_1\), \(T\), parameters \(\lambda\), \(K\), \(\gamma_d\) \\
piecewise linear function \(\delta (t)\), \\
stochastic neural network with parameters \(\mathbf{w}\): \(f(\cdot)\),   \\
stochastic input augmentation function: \(h(\cdot)\)   \\
\textbf{Output:} Parameters \(\mathbf{w}\) \\
\hline
1. \ \ \textbf{for} \(t\) \textbf{in} \([1, T_1]\) \textbf{do}   \\
2. \ \ \quad \textbf{for} each mini-batch \(\emph{B}\) \textbf{do}  \\
3. \ \ \qquad   \(\mathbf{p}_{i\in \emph{B}} \leftarrow f(h(\mathbf{x}_{i\in \emph{B}}))\)   \\
4. \ \ \qquad 	loss \(\leftarrow\) Eq. (\ref{eqn:ce_loss}) \\
5. \ \ \qquad updating \(\mathbf{w}\) using optimizers, e.g. Adam   \\
6. \ \ \quad \textbf{end for} \\
7. \ \ \textbf{end for} \\
8. \ \ \textbf{for} \(t\) \textbf{in} \([T_1+1, T]\) \textbf{do}   \\  
9.  \  \ \quad  \(\mathbf{v}\) \(\leftarrow\) Eq. (\ref{eqn:v}) \ \ \ \ \  \(\triangleright\)  Adjust \(\lambda\) to make \(\sum_{i=1}^n v_{i} =\delta (t-1)\)  \\
10.  \ \quad \textbf{for} each mini-batch \(\emph{B}\) \textbf{do}  \\
11.  \ \qquad   \(\mathbf{p}_{i\in \emph{B}} \leftarrow f(h(\mathbf{x}_{i\in \emph{B}}))\)   \\
12.  \ \qquad 	loss \(\leftarrow\) Eq. (\ref{eqn:sub2}) \\
13.  \ \qquad updating \(\mathbf{w}\) using optimizers, e.g. Adam   \\
14.  \ \qquad \(\mathbf{p}_{i\in \emph{B}}^{t-1} \leftarrow \mathbf{p}_{i\in \emph{B}}\)  \\
15.  \ \quad \textbf{end for} \\
16. \ \textbf{end for} \\
\hline
\end{tabular}
\vspace{-1em}
\end{table}

\vspace{-0.5em}
\section{Experiments}

To evaluate the proposed SPRL, we conduct experiments on four large-scale benchmark datasets: MNIST, CIFAR-10, CIFAR-100  and Mini-ImageNet. We briefly introduce them in the following. 

\textbf{MNIST} \cite{lecun1998gradient} consists of 70K images with handwritten digits from `0' to `9'. There are 60K training and 10K testing images, each of which has a size of \(28\times 28\). 

\textbf{CIFAR-10} \cite{krizhevsky2009learning} contains 60K color images belonging to 10 classes, each of which consists of 6K images. There are 50K training and 10K testing images. Each one is aligned and cropped to \(32\times 32\) pixels. 

\textbf{CIFAR-100} \cite{krizhevsky2009learning} has 60K color images in 100 classes, with 600 images per class. There are also 50K training and 10K testing images. Each image has a size of \(32\times 32\). 

\textbf{Mini-ImageNet} \cite{vinyals2016matching} is more complex than CIFAR-100. It is composed of 60K color images selected from the ImageNet dataset \cite{russakovsky2015imagenet}. These images belong to 100 classes, with 600 images per class. We divide them into a training set with 50K images and a testing set containing 10K images, and resize each image to \(32\times 32\).

The images in CIFAR-10, CIFAR-100 and Mini-ImageNet datasets are with the popular augmentation: random translations (\(\left \{ \triangle x,  \triangle y \right \} \sim \left [ -4, 4 \right ]\)) and horizontal flip (\(p=0.5\)), and each image in MNIST is only augmented by the random translation (\(\left \{ \triangle x,  \triangle y \right \} \sim \left [ -2, 2 \right ]\)) .

\vspace{-0.5em}
\subsection{Implementation Details}

We implement SPRL with the PyTorch framework and employ a 13-layer convoluational neural network (ConvNet) \cite{rasmus2015semi} \cite{laine2016temporal} or ResNet18 \cite{he2016deep} as the backbone network. We adopt the optimizer, Adam \cite{kingma2014adam}, to update the network parameters, with initializing the momentum parameters \(\beta_{1} =0.9\) and \(\beta_{2}=0.999\). By default, we follow \cite{han2018co} to set the maximum learning rate \(\eta\) to be 0.001, run the network for \(T=200\) epochs and set the batch size to be 128. When using ResNet18 on MNIST, we choose \(\eta =0.0001\) to avoid exploding gradient. After the first 80 epochs, \(\beta_{1}\) becomes 0.1 and the learning rate linearly decreases to 0 over the following 120 epochs. \(T_{1}\) can be obtained through a validation set. Specifically, we randomly select 10\% noisy training data to construct a validation set. \(T_{1}\) is the epoch number, at which the network attains the best validation accuracy, in order to obtain the best model predictions. When the noise rate \(\epsilon\) is not known, \(m\) is the maximum number of training data whose prediction \(\mathbf{p}_{i}[y_{i}]\geq 0.5 \ (1\leq i\leq n)\) during the first \(T_1\) epochs; when \(\epsilon\) is known, we can empirically choose \(m\) within the range of \(\left [ 0.5 (1-\epsilon)n , 0.8( 1-\epsilon)n \right ]\). Additionally, \(m\) should satisfy \(m\in \left [ 0.1n, 0.5n \right ]\), because the noise rate \(\epsilon\) is usually smaller than 0.9 and a large \(m\) might reduce the effect of curriculum learning. There are many choices for \(K\), we set \(K =10\). \(\gamma_d\) can be estimated with cross-validation on noisy validation sets.  For clarity, we present the detailed parameter settings (\(T_{1}\) and \(\gamma_d\)) of each experiment in the supplemental materials (Please refer to Tables A3-A4).

\subsection{Experimental Settings}
We compare the proposed SPRL with seven state-of-the-art algorithms. We briefly introduce them as follows: \\
\textbf{Standard}: the standard deep neural networks trained on noisy datasets. \\
\textbf{Bootstrap} \cite{reed2014training}: which corrects the label by using the weighted combination of predicted and original labels. We adopt hard labels in our experiments because they usually perform better than soft ones.\\ 
\textbf{F-correction} \cite{patrini2017making}: which utilizes a label transition matrix to correct model predictions. We employ the forward strategy, which usually yields better performance, and utilize a validation set to estimate the label transition matrix.\\
\textbf{Decoupling} \cite{malach2017decoupling}: which updates model parameters using the samples with different predictions of two classifiers. \\
\textbf{MentorNet} \cite{jiang2017mentornet}: which adopts an additional network to learn an approximate predefined curriculum and employs another network, StudentNet, for classification. We utilize self-paced MentorNet, which is used for the case that no clean validation data is known. \\
\textbf{Co-teaching} \cite{han2018co}: which trains two networks in a symmetric way and each network selects the samples with the small-loss distance as the confident data for the other one. \\
\textbf{Co-teaching+} \cite{yu2019does}: which is based on Co-teaching but using the strategy of``Update by Disagreement'' \cite{malach2017decoupling}. 

Here, we suppose that the noise rate is known in Co-teaching and Co-teaching+, but the noise rate is unknown in the proposed SPRL. 
For fairness, we re-implement all the seven state-of-the-art algorithms with the PyTorch framework based on their provided public codes and utilize their default parameter settings. Additionally, they adopt the same backbone networks and training procedure as SPRL.

\begin{table}[tbp]
\scriptsize 
\centering
\caption{The best testing accuracy (\%) of eight different methods on MNIST, CIFAR-10, CIFAR-100 and Mini-ImageNet with clean training data (\(\epsilon=0\)). We bold the best accuracy and its similar results (within 0.5\%).}
\vspace{-1em}
\begin{tabular}{|c|c|c|c|c|}
\hline
\multirow{3}*{Method}  & \multicolumn{4}{c|}{ResNet18}  \\
 \cline{2-5}
 & MNIST& CIFAR-10 & CIFAR-100& Mini-ImageNet   \\
\hline
Standard &\(\mathbf{99.63}\)     &93.06&72.35   & 55.26 \\
\hline
Boostrap  &\(\mathbf{99.65}\)  &\(\mathbf{94.25}\)  & 73.03 &58.27  \\
\hline
F-correction &\(\mathbf{99.65}\)      &\(\mathbf{94.08}\)  &72.92   &58.03  \\
\hline
Decoupling  &\(\mathbf{99.68}\)   &92.10  &69.87  & 47.74 \\
\hline
MentorNet  &\(\mathbf{99.65}\)   & \(\mathbf{93.87}\)  &70.46   &53.48  \\
\hline
Co-teaching  &\(\mathbf{99.58}\)   &92.88    &72.68 & 57.04  \\
\hline
Co-teaching+  & \(\mathbf{99.58}\)   &  93.17  & 70.48 & 55.88  \\
\hline
\textbf{SPRL}  &\(\mathbf{99.67}\)   &\(\mathbf{94.20}\)  &\(\mathbf{73.88}\)    &\(\mathbf{63.04}\)  \\
\hline
  & \multicolumn{4}{c|}{ConvNet} \\
\cline{2-5}
 & MNIST& CIFAR-10 & CIFAR-100& Mini-ImageNet   \\
\hline
Standard      &\(\mathbf{99.67}\)   &92.66  &71.04  & 53.82  \\
\hline
Boostrap     & \(\mathbf{99.69}\)  &\(\mathbf{93.57}\)  &72.34  &56.36  \\
\hline
F-correction     &\(\mathbf{99.69}\)  &\(\mathbf{93.76}\)  &\(\mathbf{73.01}\)  & 58.66 \\
\hline
Decoupling     & \(\mathbf{99.48}\)  &92.35  &70.03  & 47.68  \\
\hline
MentorNet    &\(\mathbf{99.71}\)  & 92.15 & 69.23  & 55.40  \\
\hline
Co-teaching    &\(\mathbf{99.74}\)   & \(\mathbf{93.60}\) &72.04 &58.57 \\
\hline
Co-teaching+  & \(\mathbf{99.69}\)   &  92.74  & 70.96 &59.10  \\
\hline
\textbf{SPRL}    &\(\mathbf{99.69}\)  & 92.76 &72.23  & \(\mathbf{60.27}\) \\
\hline
\end{tabular}
\label{table:clean}
\vspace{-1.5em}
\end{table}

\begin{table*}[tbp]
\scriptsize 
\centering
\caption{Average of testing accuracy (\%) on MNIST, CIFAR-10, CIFAR-100 and Mini-ImageNet over the last ten epochs. We bold the best results and highlight the second best ones via underlines.}
\vspace{-1em}
\begin{tabular}{|c|c|c|c|c|c|c|c|c|}
\hline
\multirow{3}*{Method}  & \multicolumn{4}{c|}{ResNet18} & \multicolumn{4}{c|}{ConvNet}  \\      
 \cline{2-9}
 & \multicolumn{3}{c|}{Symmetry}   & Pair & \multicolumn{3}{c|}{Symmetry}   & Pair    \\
  \cline{2-9}
   & \(\epsilon=0.2\) &\(\epsilon=0.5\)  & \(\epsilon=0.8\)   & \(\epsilon=0.45\)   & \(\epsilon=0.2\) &\(\epsilon=0.5\)  & \(\epsilon=0.8\)   & \(\epsilon=0.45\)  \\
 \hline
   \multicolumn{9}{|c|}{MNIST}  \\
\hline
Standard    & \(92.69\pm 0.18\) & \(65.49\pm 0.33\) & \(24.59\pm 0.18\) & \(58.50\pm 0.34\) & \(86.84\pm 0.27\)   & \(60.80\pm 0.59\)  & \(24.80\pm 0.51\)  & \(57.14\pm 0.56\)  \\
\hline
Boostrap  & \(93.89\pm 0.08\)  & \(66.48\pm 0.63\)   & \(24.38\pm 0.35\) & \(59.92\pm 0.52\) & \(91.48\pm 0.16\)  & \(61.05\pm 0.69\)  & \(21.11\pm 0.32\) & \(55.51\pm 0.72\)  \\
\hline
F-correction  & \(97.08\pm 0.11\) &\(92.86\pm 0.14\)   &\(40.93\pm 0.30\) & \(10.32\pm 0.01\) & \(87.12\pm 0.19\)  & \(66.36\pm 0.45\)  & \(58.17\pm 0.54\)  & \(57.70\pm 0.64\)  \\
\hline
Decoupling  & \(95.70\pm 0.64\)  & \(72.60\pm 4.17\)  & \(27.00\pm 0.52\) & \(71.58\pm 2.35\) & \(96.26\pm 0.24\)  & \(88.93\pm 0.40\)  & \(71.02\pm 0.39\) &\(61.22\pm 2.34\) \\
\hline
MentorNet & \(93.51\pm 0.01\) &\(83.10\pm 0.01\)   & \(24.96\pm 0.01\) & \(82.51\pm 0.01\) & \(95.78\pm 0.01\)  & \(92.44\pm 0.01\)  & \(47.55\pm 0.01\) & \(73.67\pm 0.01\) \\
\hline
Co-teaching  & \(96.89\pm 0.11\)  & \(91.01\pm 0.13\)  & \(\underline{75.92\pm 0.47}\) & \(\underline{87.44\pm 0.33}\) & \(98.91\pm 0.04\)   & \(96.55\pm 0.07\)  & \(\underline{89.54\pm 0.26}\) & \(93.64\pm 0.14\)   \\
\hline
Co-teaching+  & \(\underline{99.04\pm 0.02}\)  & \(\underline{94.69\pm 0.15}\) & \(38.34\pm 1.23\) & \(87.36\pm 0.39\) & \(\mathbf{99.56\pm 0.01}\) &\(\underline{99.15\pm 0.02}\) &  \(77.77\pm 0.03\)& \(\underline{97.25\pm 0.17}\)   \\
\hline
\textbf{SPRL}  & \(\mathbf{99.58\pm 0.01}\)  & \(\mathbf{99.53\pm 0.01}\) &\(\mathbf{98.35\pm 0.16}\) & \(\mathbf{99.30\pm 0.01}\) &\(\mathbf{99.56\pm 0.01}\)   & \(\mathbf{99.43\pm 0.01}\)  & \(\mathbf{97.52\pm 0.04}\) & \(\mathbf{99.28\pm 0.01}\)   \\
\hline
   \multicolumn{9}{|c|}{CIFAR-10}  \\
\hline
Standard &\(79.47\pm 0.24\)   &\(45.37\pm 0.53\) &\(10.01\pm 0.19\)  &\( 51.48\pm 0.80\) & \(77.82\pm 0.27\)  & \(48.11\pm 0.42\)  &\(22.31\pm 0.34\)  & \(50.73\pm 0.62\)  \\
\hline
Boostrap  &\(83.39\pm 0.26\)  &\(57.24\pm 0.48\)    &\(17.96\pm 0.27\) &\(51.95\pm 0.57\)  & \(75.34\pm 0.84\)   &\(47.37\pm 0.74\) & \(18.00\pm 0.53\) & \(51.11\pm 0.74\)  \\
\hline
F-correction  &\(80.11\pm 0.16\) &\(45.92 \pm 0.65\) & \(6.51\pm 0.19\)  & \(52.35\pm 0.48\) & \(84.26\pm 0.21\)  &\(62.90\pm 0.42\)  & \(11.58\pm 0.16\)  & \(61.98\pm 0.40\)  \\
\hline
Decoupling  & \(76.60\pm 1.25\)  & \(53.63\pm 1.21\)  & \(16.18\pm 0.17\) &\(50.22\pm 2.58\) &\(83.49\pm 0.19\) & \(68.73\pm 0.27\)  &\(\underline{40.16\pm 0.30}\)  & \(50.61\pm 3.12\)\\
\hline
MentorNet &\(80.04\pm 0.21\)&\(53.20\pm 0.16\)&\(\underline{42.02\pm 0.19}\) &\(49.93\pm 0.11\) & \(81.80\pm 0.01\)   & \(73.62\pm 0.14\) & \(27.90\pm 0.05\) & \(52.96\pm 0.02\)\\
\hline
Co-teaching  & \(88.34\pm 0.23\)  & \(81.64\pm 0.19\)  & \(32.39\pm 0.26\) & \(\underline{79.09\pm 0.43}\) & \(86.40\pm 2.58\) &\(83.06\pm 0.15\)  & \(28.39\pm 0.31\) & \(\underline{80.21\pm 0.58}\)   \\
\hline
Co-teaching+  & \(\underline{90.87\pm 0.10}\) &\(\underline{83.52\pm 0.10}\)   &\(23.18\pm 0.10\) & \(60.07\pm 0.56\) & \(\underline{90.43\pm 0.10}\)& \(\underline{85.87\pm 0.08}\)& \(20.41\pm 0.04\)  & \(77.51\pm 0.22\)  \\
\hline
\textbf{SPRL}  &\(\mathbf{92.68\pm 0.03}\) & \(\mathbf{88.25\pm0.06}\)   &\(\mathbf{57.50\pm0.11}\)   &\(\mathbf{91.89\pm 0.06}\) & \(\mathbf{90.47\pm 0.06}\)   & \(\mathbf{85.99\pm 0.05}\)  & \(\mathbf{60.42\pm 0.12}\)  & \(\mathbf{83.69\pm 0.12}\)   \\
\hline
   \multicolumn{9}{|c|}{CIFAR-100}  \\
\hline
Standard    &\(54.14\pm0.21\)  &\(28.73\pm 0.19\)  & \(7.03\pm 0.10\)  &\(35.24\pm 0.08\)& \(49.71\pm 0.44\)  & \(23.87\pm 0.15\)  & \(9.37\pm 0.11\) & \(34.83\pm 0.30\)  \\
\hline
Boostrap  &\(55.94\pm 0.29\) &\( 31.37\pm 0.22\)  &\( 7.48\pm 0.10\) &\(37.07\pm 0.21\)  & \(50.51\pm 0.25\)   & \(25.21\pm 0.18\) & \(9.66\pm 0.15\)  & \(34.43\pm 0.21\) \\
\hline
F-correction  &\(56.32\pm 0.12\)  &\(37.73\pm 0.08\)    & \(9.09\pm 0.08\)&\(37.79\pm 0.19\)  & \(54.42\pm 0.13\)   & \(33.19\pm 0.11\)  & \(5.54\pm 0.08\) & \(38.70\pm 0.34\)  \\
\hline
Decoupling  &\(54.56\pm 0.61\)  & \(30.51\pm 0.36\)  & \(7.37\pm 0.09\) & \(36.74\pm 0.28\)  & \(53.99\pm 0.17\)  & \(32.84\pm 0.10\) & \(14.81\pm 0.09\) & \(37.31\pm 0.23\)\\
\hline
MentorNet & \(52.11\pm 0.16\) & \(26.71\pm 0.16\)& \(12.76\pm 0.08\) &\(33.92\pm 0.14\) &  \(52.70\pm 0.01\) &\(38.75\pm 0.02\) & \(11.02\pm 0.01\) & \(31.87\pm 0.01\) \\
\hline
Co-teaching  & \(62.71\pm 0.13\) & \(48.14\pm 0.15\)   & \(15.94\pm 0.10\) & \(\underline{39.49\pm 0.23}\) & \(66.30\pm 0.43\)   &  \(57.29\pm 0.13\) &\(\underline{19.96\pm 0.17}\)  & \(37.50\pm 0.19\)   \\
\hline
Co-teaching+  & \(\underline{66.41\pm 0.12}\)  &\(\underline{51.65\pm 0.13}\)   &\(\underline{19.17\pm 0.09}\) &\(34.96\pm 0.28\)  &\(\mathbf{69.00\pm 0.12}\) & \(\underline{59.79\pm 0.13}\) &\(11.57\pm 0.07\)   &\(\underline{43.21\pm 0.21}\)   \\
\hline
\textbf{SPRL}  &\(\mathbf{70.93\pm 0.06}\) & \(\mathbf{59.31\pm 0.07}\)   &\(\mathbf{28.53\pm 0.10}\) &\(\mathbf{53.59\pm 0.06}\) & \(\underline{67.65\pm 0.10}\)   & \(\mathbf{59.81\pm 0.12}\) & \(\mathbf{35.82\pm 0.14}\) & \(\mathbf{47.26\pm 0.11}\)   \\
\hline
   \multicolumn{9}{|c|}{Mini-ImageNet}  \\
\hline
Standard    & \(34.07\pm 0.18\)  & \(16.17\pm 0.21\)  & \(3.55\pm 0.10\) & \(22.78\pm 0.16\) & \(38.06\pm 0.31\)   & \(19.61\pm 0.25\) & \(\underline{7.96\pm 0.11}\) & \(26.69\pm 0.16\)  \\
\hline
Boostrap  & \(34.91\pm  0.39\)  & \(18.22\pm 0.19\)  & \(4.14\pm 0.18\)  & \(24.28\pm 0.27\) & \(38.73\pm 0.38\)  & \(19.12\pm 0.15\)  & \(5.75\pm 0.11\)  & \(27.60\pm 0.22\) \\
\hline
F-correction  & \(31.81\pm 0.14\)  & \(12.29\pm 0.10\)  & \(2.13\pm 0.04\) & \(6.13\pm 0.07\) & \(33.45\pm 0.18\)   & \(26.96\pm 0.08\)  & \(2.24\pm 0.03\)   & \(5.04\pm 0.06\)  \\
\hline
Decoupling  & \(33.38\pm 0.17\)  & \(16.63\pm 0.12\)   & \(4.26\pm 0.06\) & \(22.78\pm 0.19\) & \(30.37\pm 0.15\)   & \(15.46\pm 0.17\)  & \(6.21\pm 0.05\)  & \(24.12\pm 0.11\)\\
\hline
MentorNet & \(29.19\pm  0.01\)  & \(14.14\pm 0.01\) &\(1.06\pm 0.01\)   & \(21.89\pm 0.01\) & \(43.47\pm 0.02\)  & \(31.09\pm 0.01\)  & \(1.80\pm 0.01\)  & \(27.01\pm 0.01\) \\
\hline
Co-teaching  & \(48.84\pm 0.09\)  & \(\underline{36.98\pm 0.17}\) & \(5.86\pm 0.11\) & \(\underline{29.21\pm 0.11}\) & \(53.62\pm 0.13\)   &  \(43.54\pm 0.15\) & \(5.51\pm 0.05\)  & \(30.68\pm 0.16\)  \\
\hline
Co-teaching+  &\(\underline{51.13\pm0.14}\)   &\(36.86\pm 0.23\)    & \(\underline{7.23\pm 0.05}\) & \(27.46\pm 0.07\)  &\(\underline{54.99\pm 0.13}\) & \(\underline{45.02\pm 0.23}\)& \(6.06\pm 0.06\) & \(\underline{33.94\pm 0.15}\)  \\
\hline
\textbf{SPRL}  & \(\mathbf{57.24\pm 0.09}\) & \(\mathbf{47.66\pm 0.11}\)  & \(\mathbf{20.77\pm 0.09}\) & \(\mathbf{39.53\pm 0.07}\)  & \(\mathbf{55.32\pm 0.08}\)   & \(\mathbf{46.32\pm 0.13}\)   & \(\mathbf{24.40\pm 0.08}\)  & \(\mathbf{37.78\pm 0.12}\)   \\
\hline
\end{tabular}
\label{table:noisy}
\vspace{-1.5em}
\end{table*}

\vspace{-0.5em}
\subsection{Experiments on Labels with Symmetry and Pair Flipping}

Following \cite{patrini2017making} \cite{han2018co}, we corrupt the four datasets manually via a label transition matrix \(\mathbf{Q}\) that is calculated by \(q_{ij} = Pr(\tilde{y}=j|y=i)\), where the noisy label \(\tilde{y}\) is flipped from the correct label \(y\). Similar to \cite{han2018co}, here \(\mathbf{Q}\) has two representative structures: symmetric flipping (class-independent noise) and pair flipping (class-dependent noise). For clarity, we present the definition of \(\mathbf{Q}\) with symmetric and pair flipping structures in Eqs. (\ref{subeqn:symmetric}) and (\ref{subeqn:pair}), respectively.
It is worth noting that for symmetric flipping, the noise rate \(\epsilon\) should be smaller than \(\frac{c-1}{c}\), i.e. \(\epsilon<\frac{c-1}{c}\); for pair flipping, \(\epsilon <0.5\) so that more than half of labels are correct.  Note that  the noise rate \(\epsilon\) denotes the ratio of corrupted labels in the whole training data.

\begin{subequations}\label{eqn:noise}
 \begin{align}
Q= \begin{bmatrix}
 1-\epsilon & \frac{\epsilon}{c-1}  & \cdots  & \frac{\epsilon}{c-1}   &\frac{\epsilon}{c-1}  \\ 
\frac{\epsilon}{c-1} &  1-\epsilon  & \frac{\epsilon}{c-1} & \cdots &\frac{\epsilon}{c-1} \\ 
\vdots & &  \ddots & & \vdots \\ 
 \frac{\epsilon}{c-1}& \cdots & \frac{\epsilon}{c-1} &  1-\epsilon & \frac{\epsilon}{c-1}\\ 
 \frac{\epsilon}{c-1}& \frac{\epsilon}{c-1} & \cdots &\frac{\epsilon}{c-1}  &  1-\epsilon
\end{bmatrix} 
\label{subeqn:symmetric}  \\
Q= \begin{bmatrix}
 1-\epsilon & \epsilon  & 0  & \cdots   & 0   \\ 
0 &  1-\epsilon  & \epsilon &  & 0 \\ 
\vdots & &  \ddots & & \vdots \\ 
 0 &  &  &  1-\epsilon & \epsilon\\ 
 \epsilon& 0 & \cdots & 0  &  1-\epsilon
\end{bmatrix} 
 \label{subeqn:pair}
 \end{align}
\end{subequations}

\vspace{-0.5em}
\subsubsection{Experimental Results and Analysis}
To better illustrate the strength of the proposed SPRL, we first run all the eight methods with clean training data of the four datasets, and then present their best testing accuracy in Table \ref{table:clean}. As we can see, SPRL can achieve better or very competitive testing accuracy to the best competitors when using clean training  data, and it consistently outperforms Standard, especially for more difficult datasets CIFAR-100 and Mini-ImageNet.  A main possible reason is that the proposed method could reduce overfitting caused by outliers. This finding is very important. It shows that the proposed method has a wide range of applications.

Table \ref{table:noisy} shows the average of testing accuracy of the proposed SPRL and seven compared algorithms on MNIST, CIFAR-10, CIFAR-100 and Mini-ImageNet over the last ten epochs. It illustrates that SPRL significantly outperforms the other seven algorithms on the four datasets, especially on extremely noisy labels. For example, when using ResNet18, for symmetric flipping with \(\epsilon=0.8\), the average accuracy of SPRL is 22.43\%, 15.48\%, 9.36\% and 13.54\%  higher than the best competitors on the four datasets, respectively; for pair flipping with \(\epsilon=0.45\), its accuracy is 11.86\%, 12.80\%, 14.10\% and  10.32\% higher than the best competitors on the four datasets, respectively. The superior accuracy of SPRL over the others can also be observed when using ConvNet. Note that, the implementation results of Co-teaching with ConvNet are significantly better than the reported ones in \cite{han2018co}. Because we utilize the data augmentation, which boosts the model performance. Moreover, we present testing accuracy of the eight methods at different numbers of training epochs on the four datasets in the supplemental materials (please see Figs. A3-A6), which further illustrate that SRL can obtain the best accuracy among all methods on two different network architectures, and its accuracy is much smoother than that of the others during training.

\renewcommand{\thesubfigure}{\relax}
\begin{figure*}[htbp]
\centering
\subfigure[(a) Symmetry \(\epsilon=0.2\)]{\includegraphics[trim={0.5em 0em 2em 1em}, clip, width=0.24\textwidth]{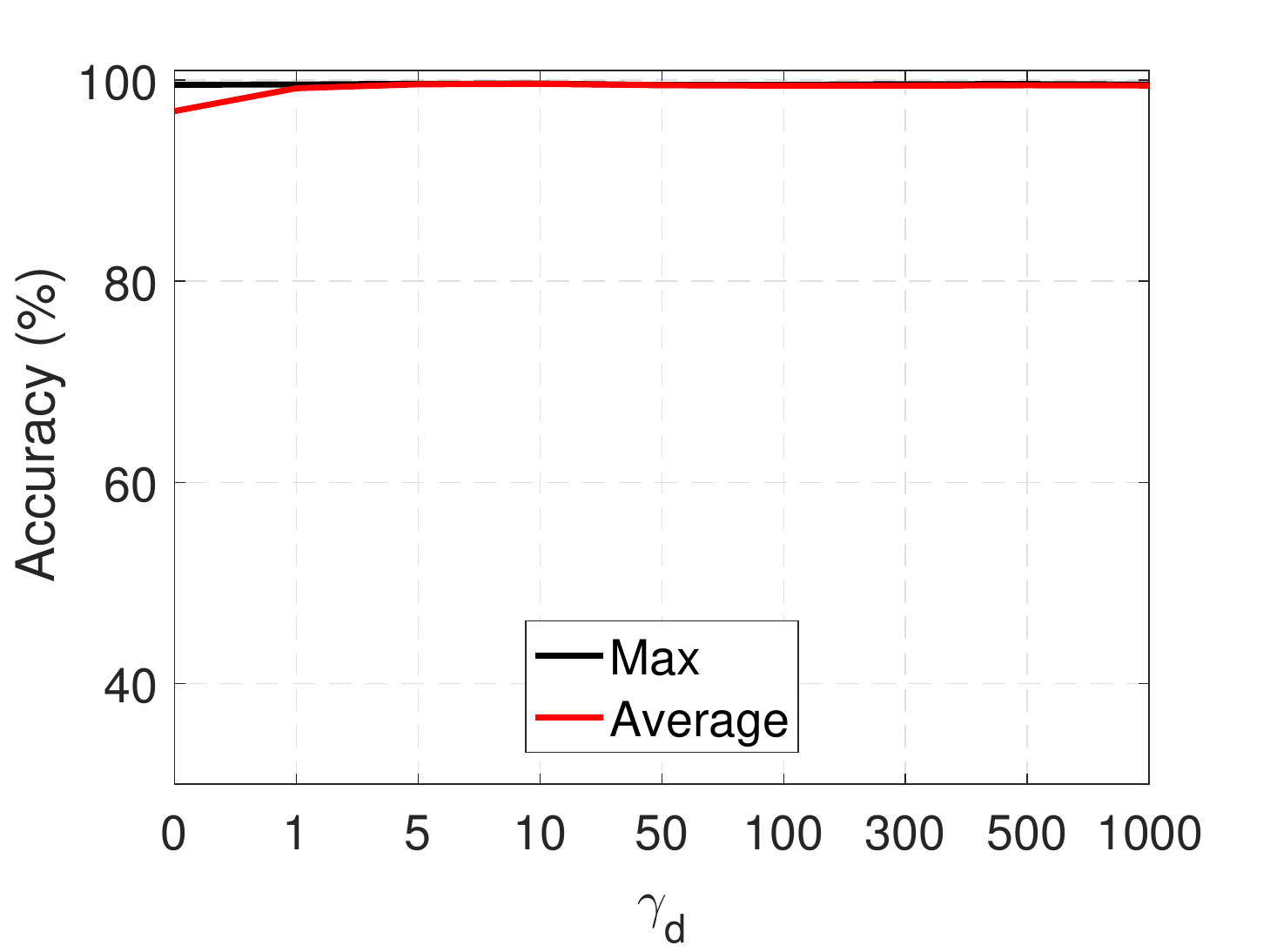}}
\subfigure[(b) Symmetry \(\epsilon=0.5\)]{\includegraphics[trim={0.5em 0em 2em 1em}, clip, width=0.24\textwidth]{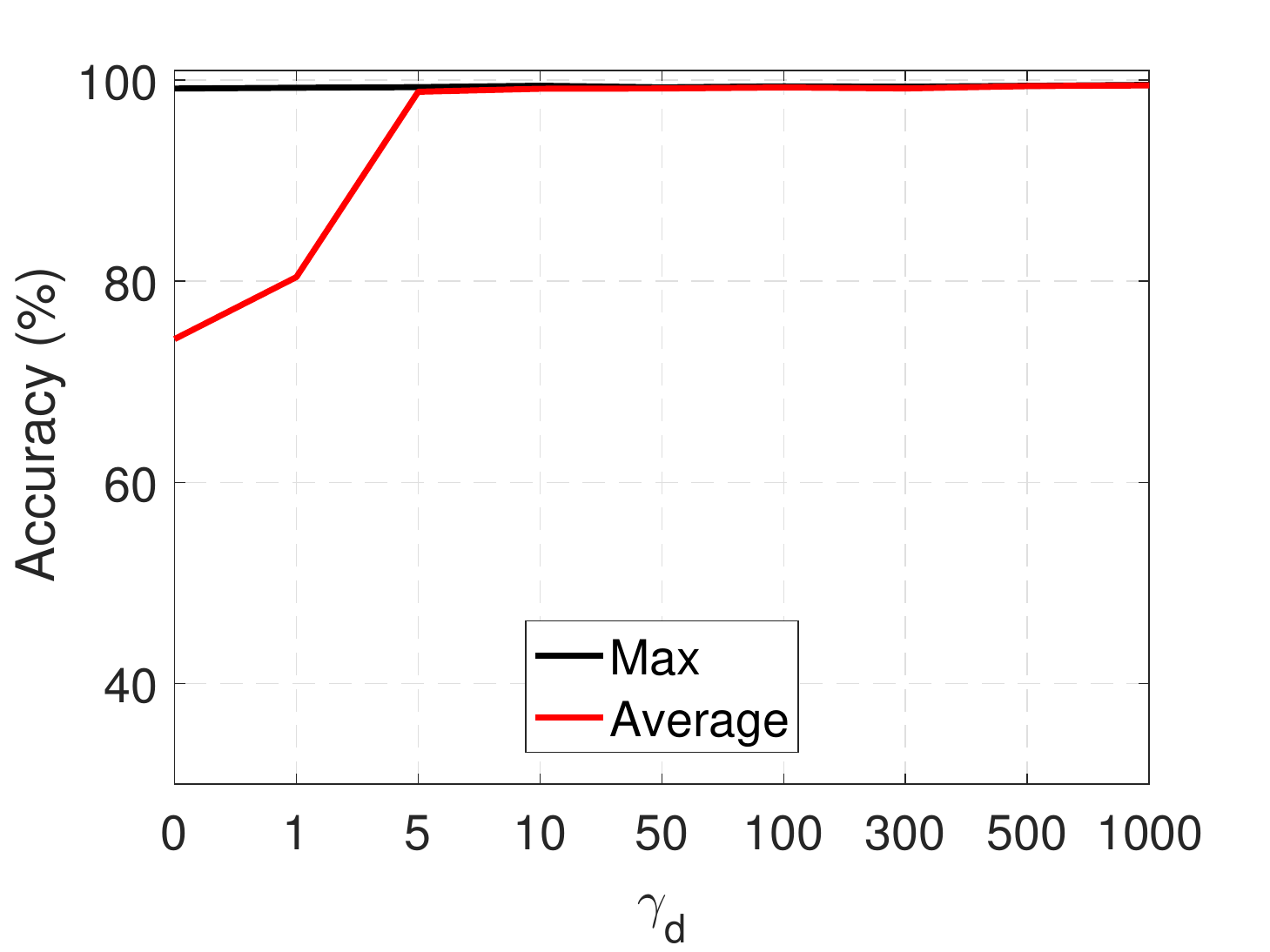}}
\subfigure[(c) Symmetry \(\epsilon=0.8\)]{\includegraphics[trim={0.5em 0em 2em 1em}, clip, width=0.24\textwidth]{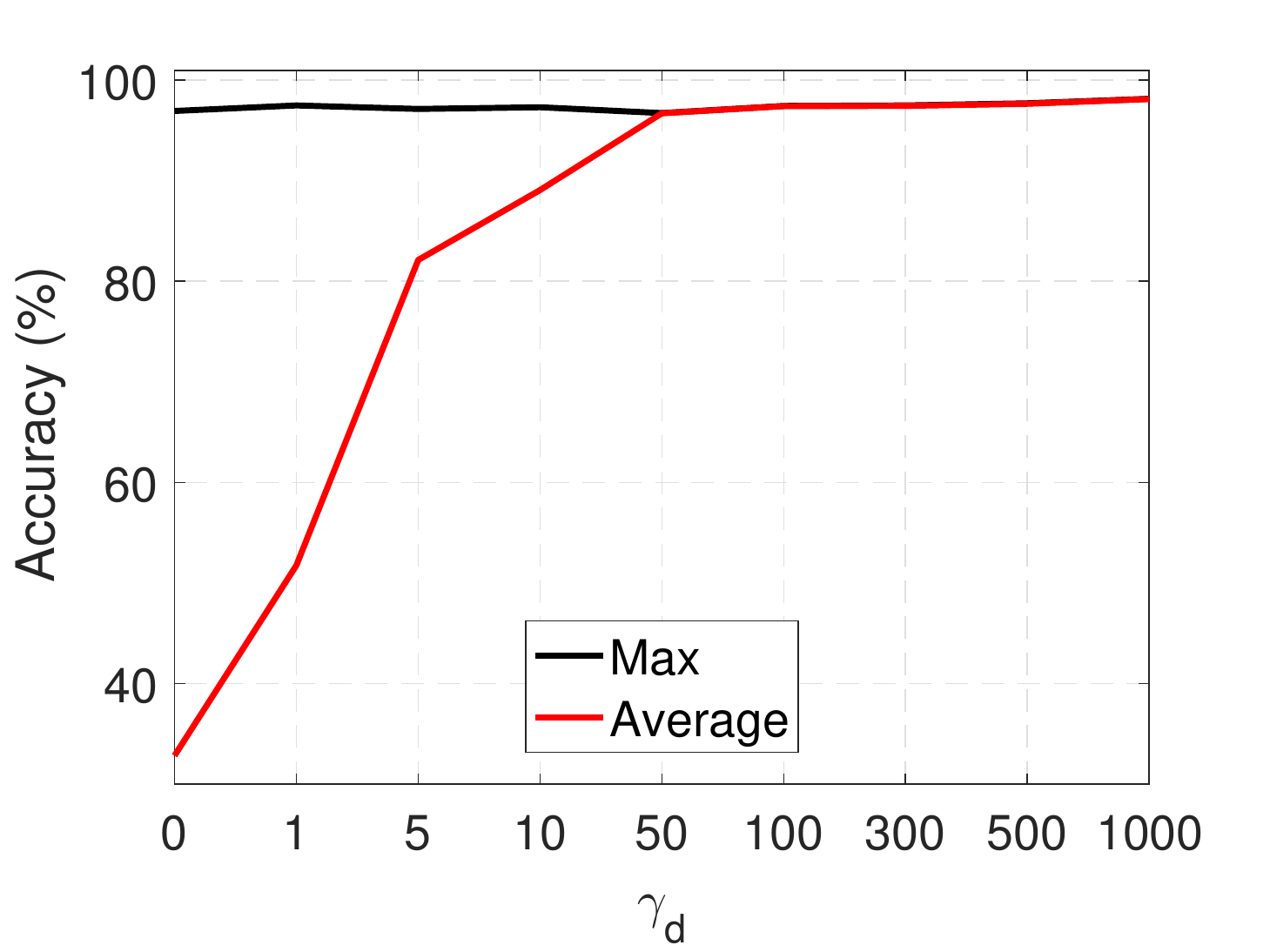}}
\subfigure[(d) Pair \(\epsilon=0.45\)]{\includegraphics[trim={0.5em 0em 2em 1em}, clip, width=0.24\textwidth]{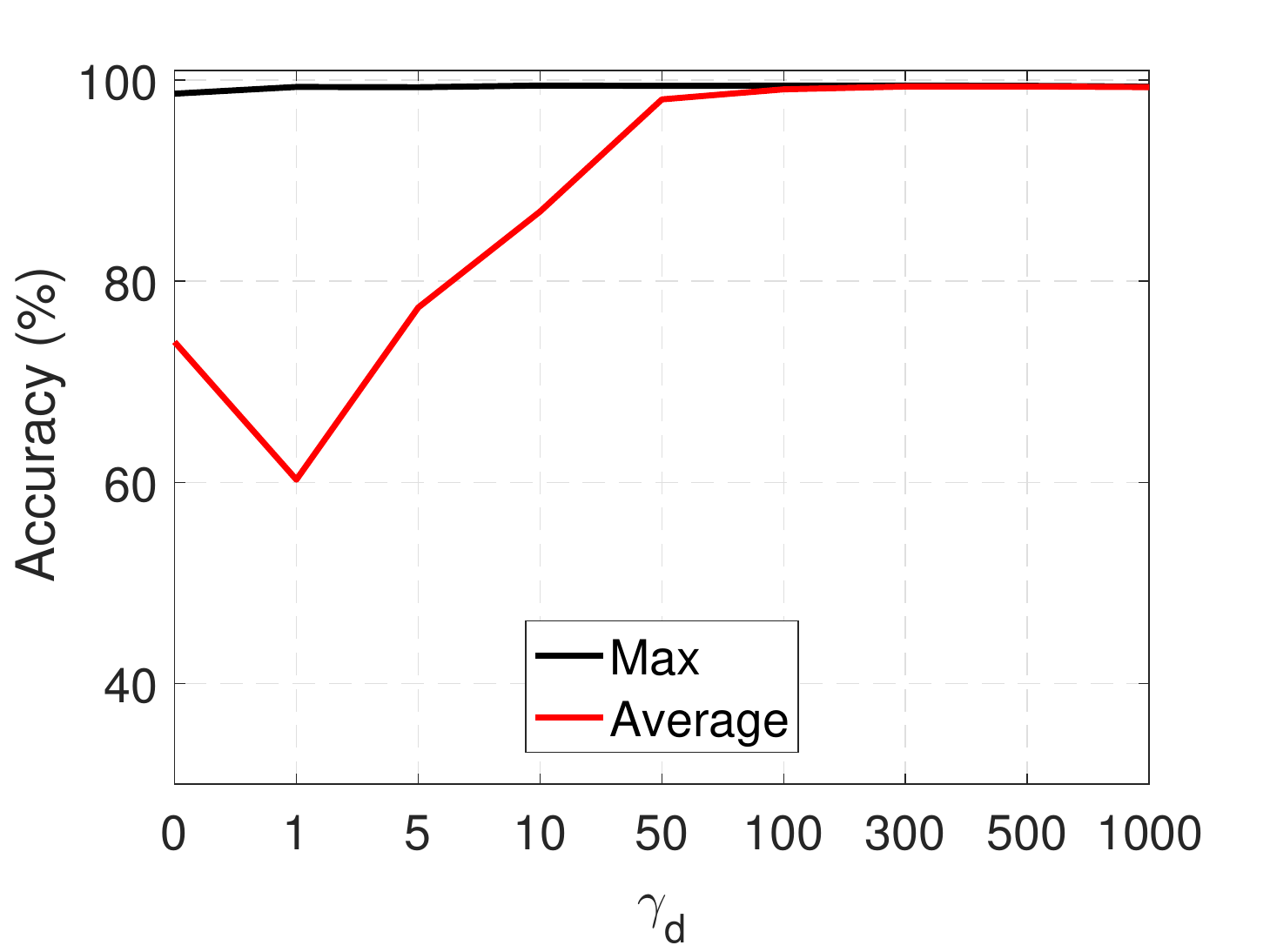}}
\vspace{-0.5em}
\caption{Testing accuracy of SPRL with different values of \(\gamma_d\) on MNIST when \(T_{1}=15\). `Average' means the average of testing accuracy over the last ten epochs, and `Max' denotes the maximum of testing accuracy among all training epochs.} 
\label{fig:mnist_gamma}
\vspace{-1.2em}
\end{figure*}

\renewcommand{\thesubfigure}{\relax}
\begin{figure*}[htbp]
\centering
\subfigure[(a) Symmetry \(\epsilon=0.2\)]{\includegraphics[trim={0.5em 0em 2em 1em}, clip, width=0.24\textwidth]{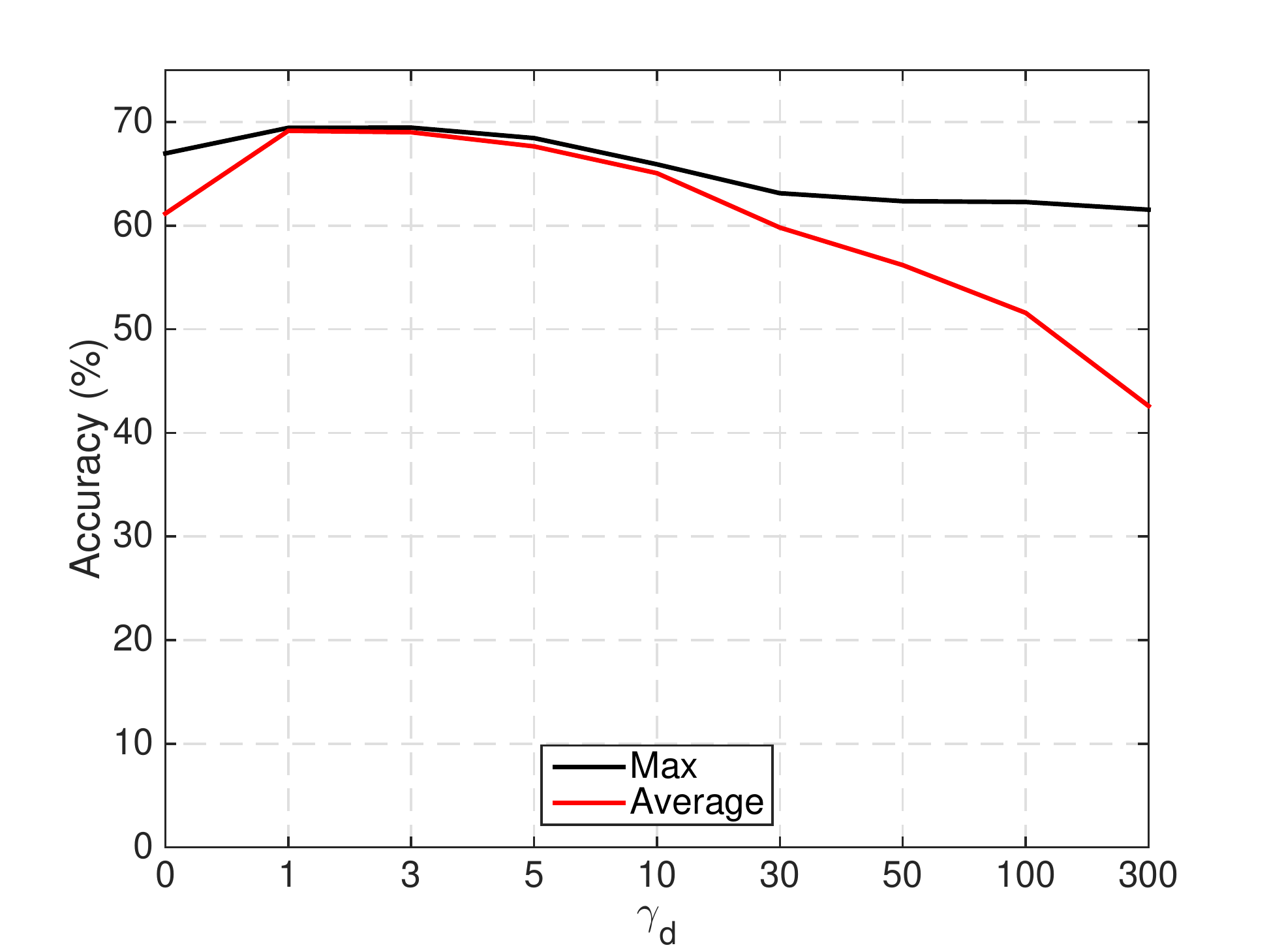}}
\subfigure[(b) Symmetry \(\epsilon=0.5\)]{\includegraphics[trim={0.5em 0em 2em 1em}, clip, width=0.24\textwidth]{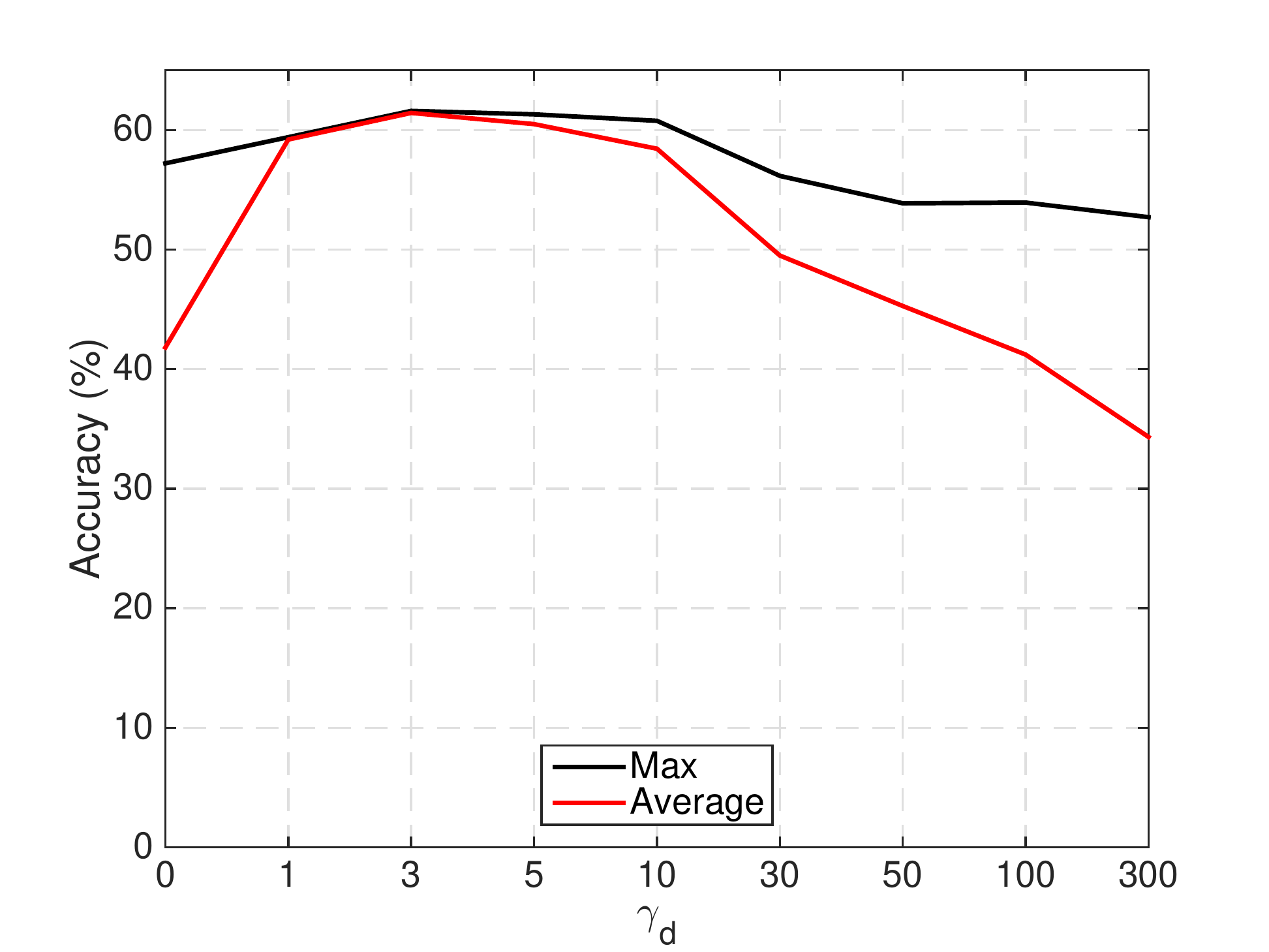}}
\subfigure[(c) Symmetry \(\epsilon=0.8\)]{\includegraphics[trim={0.5em 0em 2em 1em}, clip, width=0.24\textwidth]{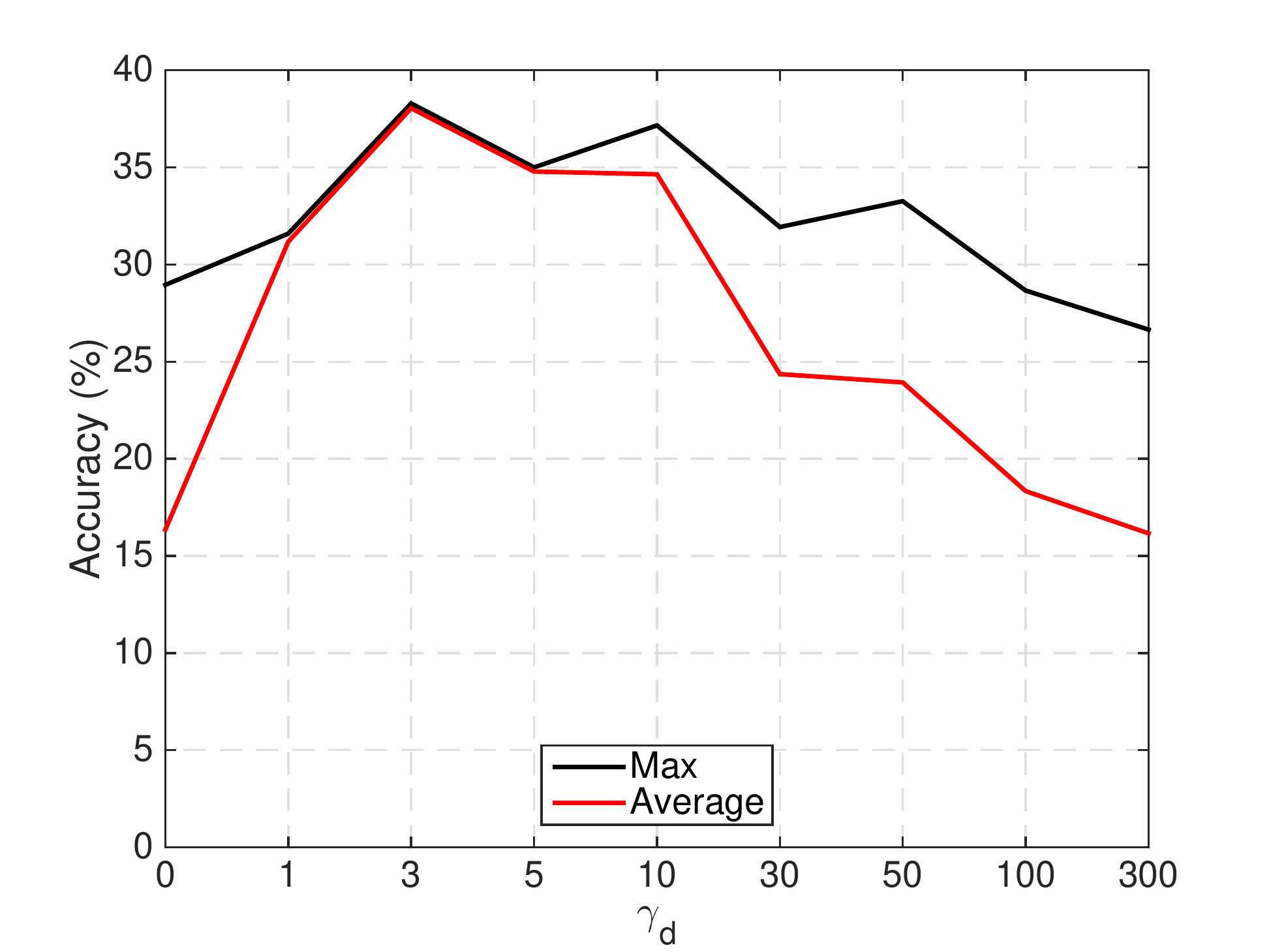}}
\subfigure[(d) Pair \(\epsilon=0.45\)]{\includegraphics[trim={0.5em 0em 2em 1em}, clip, width=0.24\textwidth]{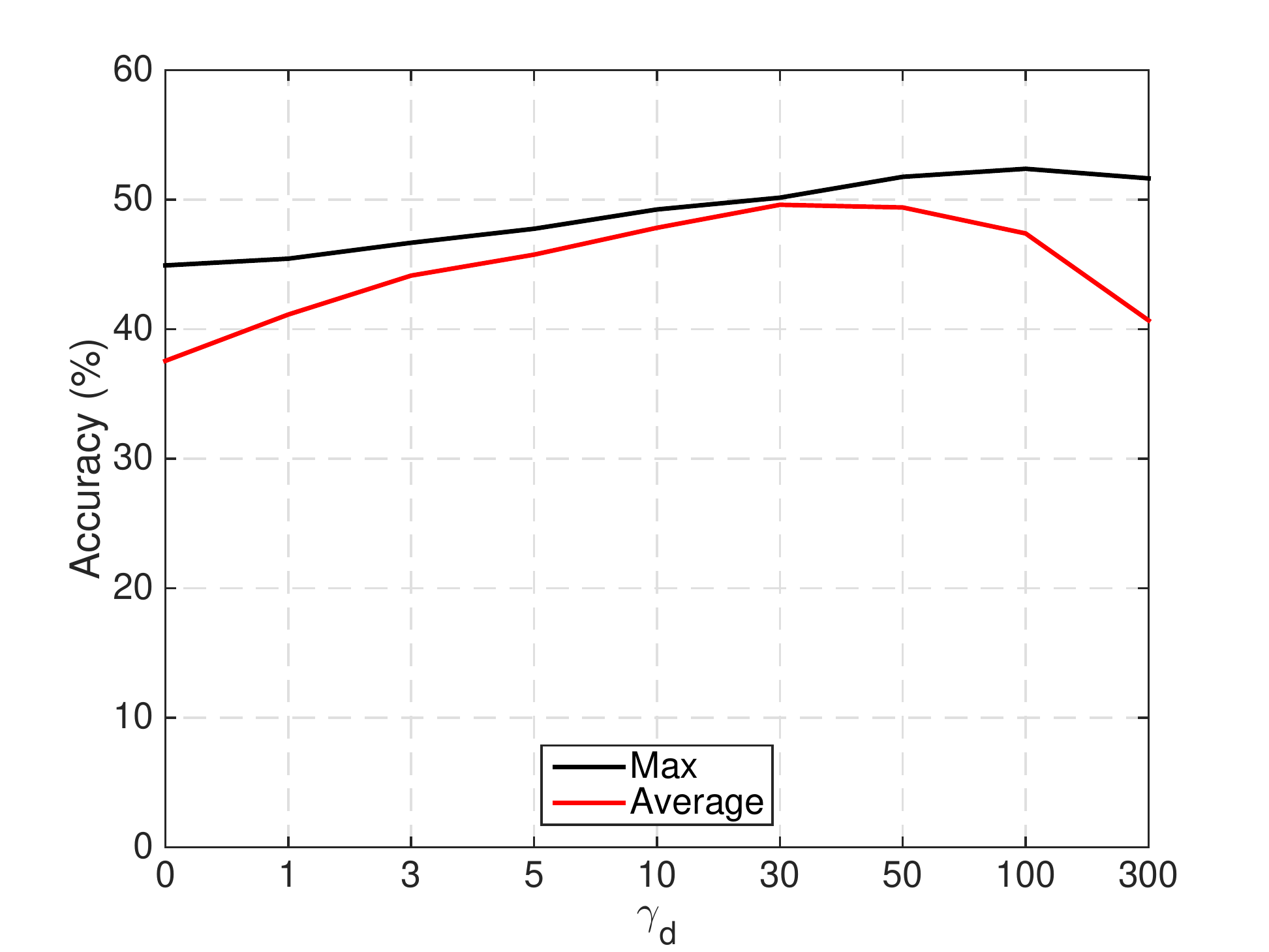}}
\vspace{-0.5em}
\caption{Testing accuracy of SPRL with different values of \(\gamma_d\) on CIFAR-100 when \(T_{1}=40\). `Average' means the average of testing accuracy over the last ten epochs, and `Max' denotes the maximum of testing accuracy among all training epochs.} 
\label{fig:cifar100_gamma}
\vspace{-1.2em}
\end{figure*}

\vspace{-0.5em}
\subsubsection{Parameter Analysis}
The proposed SPRL has three essential parameters \(\gamma_d\) , \(K\), and \(T_{1}\), where \(\gamma_d\) and \(K\) determine \(\gamma(t)\) and \(\delta(t)\), respectively, and \(T_1\) determines \(m\).  Here, we evaluate them by utilizing an easy dataset MNIST, a complex dataset CIFAR-100 and the network ConvNet. Specifically, Figs. \ref{fig:mnist_gamma}-\ref{fig:cifar100_gamma} show testing accuracy of SPRL with different values of \(\gamma_d\) on symmetric or pair flipping label noise, including \(\gamma_d\in \left \{ 0, 1, 5, 10, 50, 100, 300, 500, 1000 \right \} \) on MNIST and \(\gamma_d\in \left \{ 0, 1, 3, 5, 10, 30, 50, 100, 300 \right \} \) on CIFAR-100.  Fig. \ref{fig:K} displays testing accuracy of SPRL with different \(K\) within \(\left \{ 100, 50, 20, 10, 5, 2 \right \}\),  and Fig. \ref{fig:T} presents its testing accuracy  with different values of \(T_{1}\), like \(T_1\in \left \{ 5, 10, 15, 20, 30, 40, 60, 80, 100 \right \}\) on MNIST and \(T_1 \in \left \{ 5, 10, 15, 20, 30, 40, 60, 80, 100 \right \}\) on CIFAR-100. Fig. \ref{fig:loss} displays the effect of different noise rates on the loss function Eq. (\ref{eqn:sub2}) during training.

Figs. \ref{fig:mnist_gamma}-\ref{fig:cifar100_gamma} suggest that a large weight of the resistance loss (Eq. (\ref{eqn:ce_ad})) can prevent the performance degradation of CNNs on symmetric or pair flipping label noise.  Additionally, Fig. \ref{fig:cifar100_gamma} also suggests that Eq. (\ref{eqn:ce_ad}) with a large weight can boost the model accuracy. When \(\gamma_d\geq 50\), SPRL obtains the best or sub-optimal accuracy on MNIST; when \(\gamma_d\in \left [1,10  \right ]\), SPRL obtains the best or sub-optimal accuracy on CIFAR-100 with symmetric label noise, and when \(\gamma_d\in \left [10,50  \right ]\), it achieves the best or sub-optimal accuracy on pair flipping label noise. However, Fig. \ref{fig:cifar100_gamma} illustrates that if \(\gamma_d\) is too large, the model accuracy will decrease on CIFAR-100, probably because the resistance loss with model predictions tends to make the prediction on each class be equivalent. Furthermore, \(\gamma_d\) = 0 means removing the resistance loss (Eq. (\ref{eqn:ce_ad})) from the proposed loss function (Eq. (\ref{eqn:loss})), as shown in Figs. \ref{fig:mnist_gamma}-\ref{fig:cifar100_gamma}, the proposed resistance loss is very helpful to improve performance.

Figs. \ref{fig:K}-\ref{fig:T} infer that both \(K\) and \(T_{1}\) can affect the performance of SPRL on both MNIST and CIFAR-100, especially on the complex dataset CIFAR-100. If \(K\) is too small, SPRL might select more corrupt-label samples at each pace to update model parameters, thereby decreasing its accuracy.   When \(T_{1}\) is too small, it will result in low training and testing accuracy; when \(T_{1}\) is too large, the model will be overfitted on corrupted labels, thereby decreasing the model performance. Therefore, we select \(T_{1}\) where SPRL achieves the best or sub-optimal accuracy on a validation set constructed by noisy training data. Fig. \ref{fig:loss} demonstrates that a larger noise rate will result in a larger loss. The reason might be that a larger noise rate leads to more training samples with different model predictions from corrupted labels. 

\renewcommand{\thesubfigure}{\relax}
\begin{figure}[tbp]
	\subfigure[(a) MNIST@ \(\gamma_d=300\)]{\includegraphics[trim={0.5em 0em 2em 1em}, clip, width=0.24\textwidth]{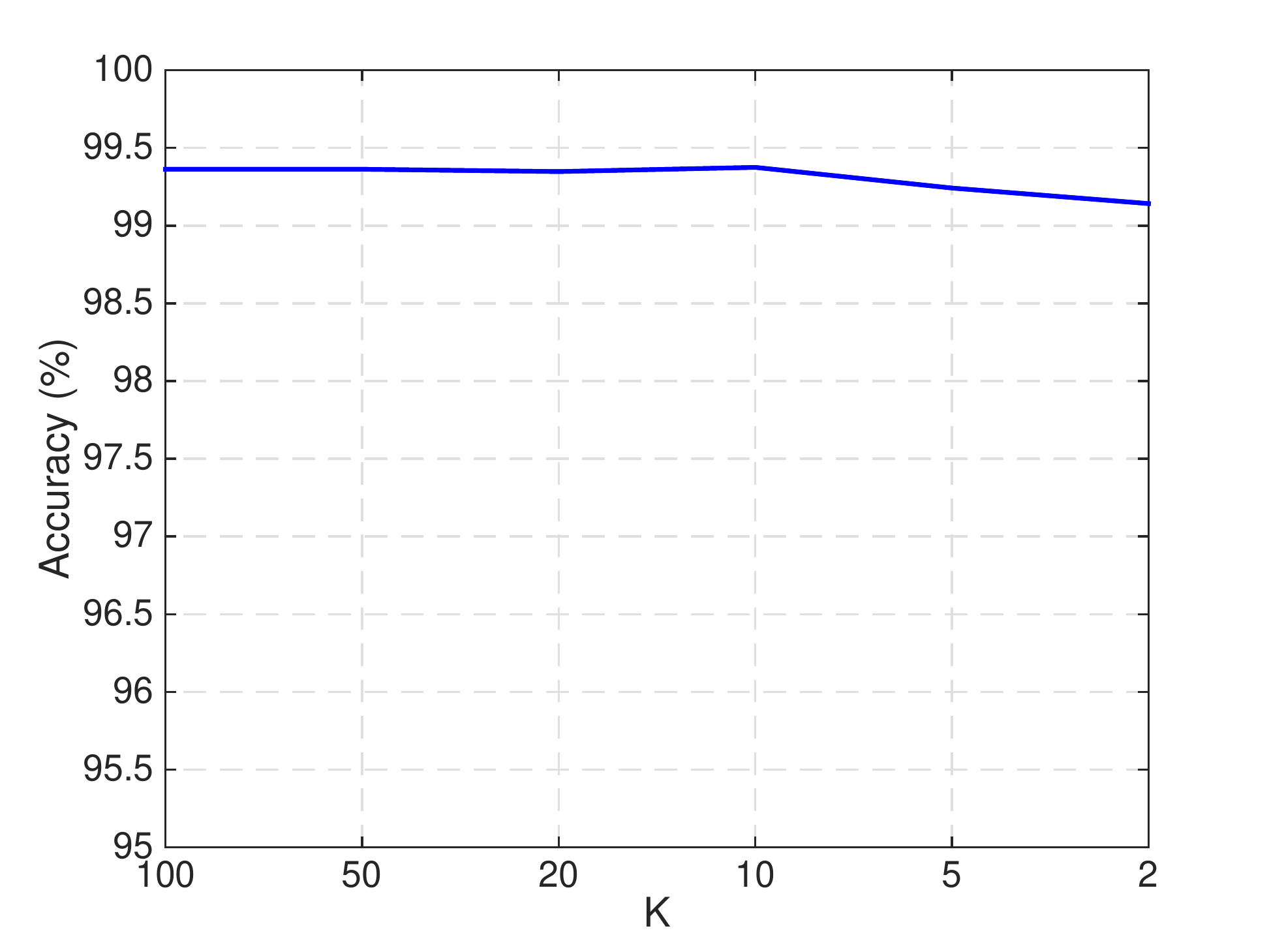}}
	\subfigure[(b) CIFAR-100@ \(\gamma_d=5\)]{\includegraphics[trim={0.5em 0em 2em 1em}, clip, width=0.24\textwidth]{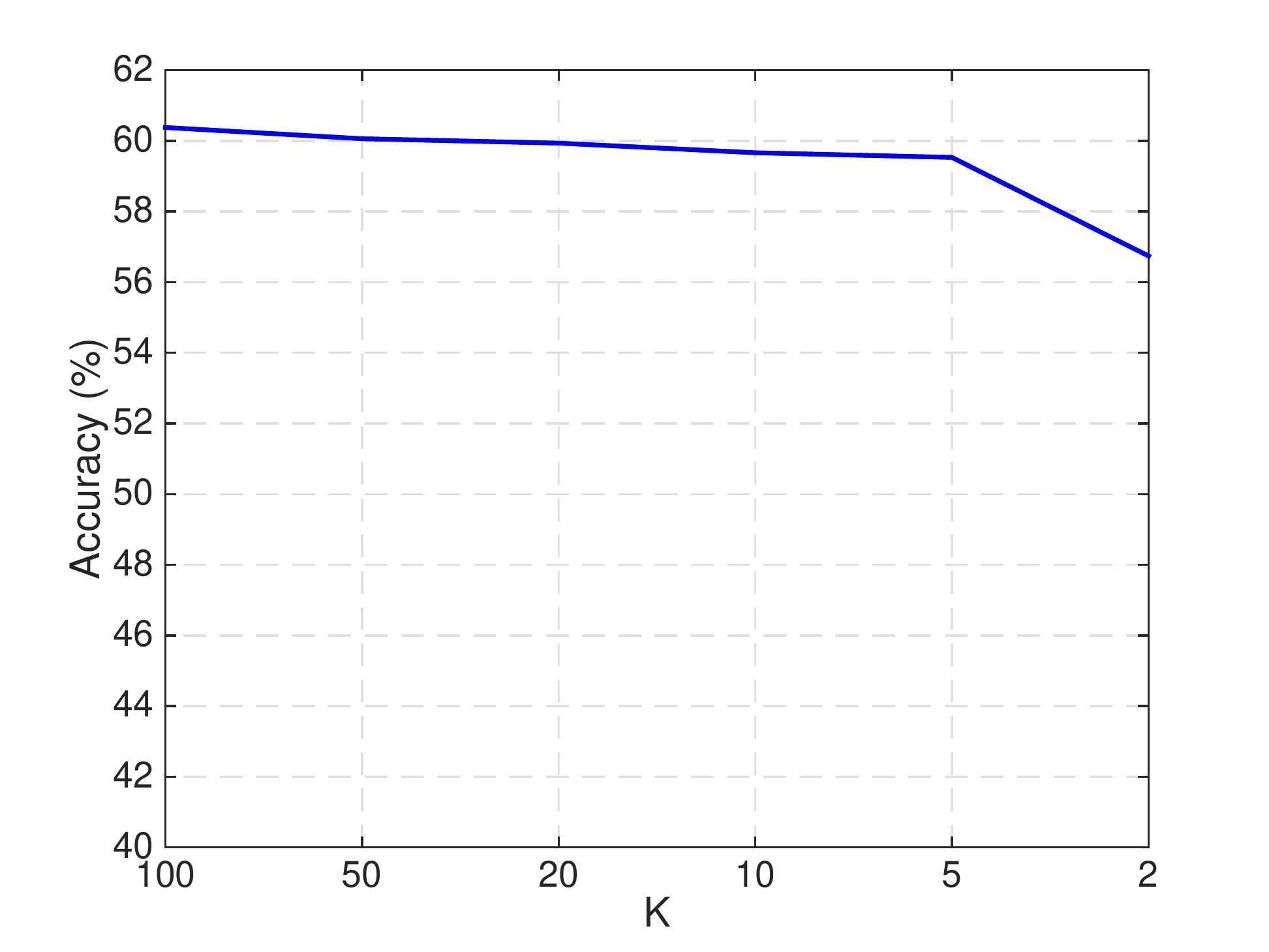}}
	\vspace{-0.5em}
	\caption{The average of testing accuracy of SPRL with different \(K\) on MNIST and CIFAR-100 over the last ten epochs when labels are corrupted by symmetric flipping with a 50\% noise rate.} 
	\label{fig:K}
	\vspace{-1.5em}
\end{figure}

\renewcommand{\thesubfigure}{\relax}
\begin{figure}[tbp]
	\subfigure[(a) MNIST@ \(\gamma_d=300\)]{\includegraphics[trim={0.5em 0em 2em 1em}, clip, width=0.24\textwidth]{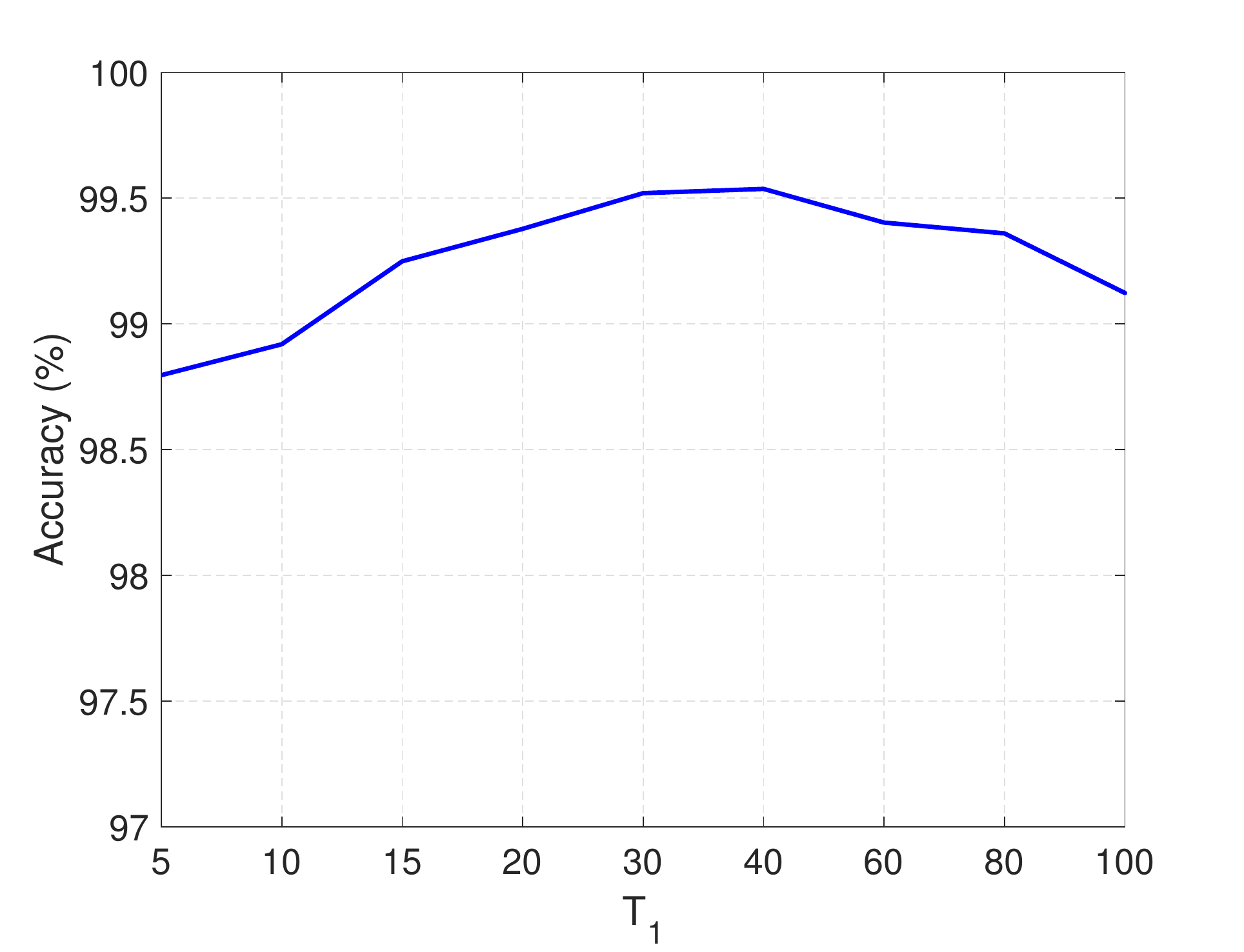}}
	\subfigure[(b) CIFAR-100@ \(\gamma_d=5\)]{\includegraphics[trim={0.5em 0em 2em 1em}, clip, width=0.24\textwidth]{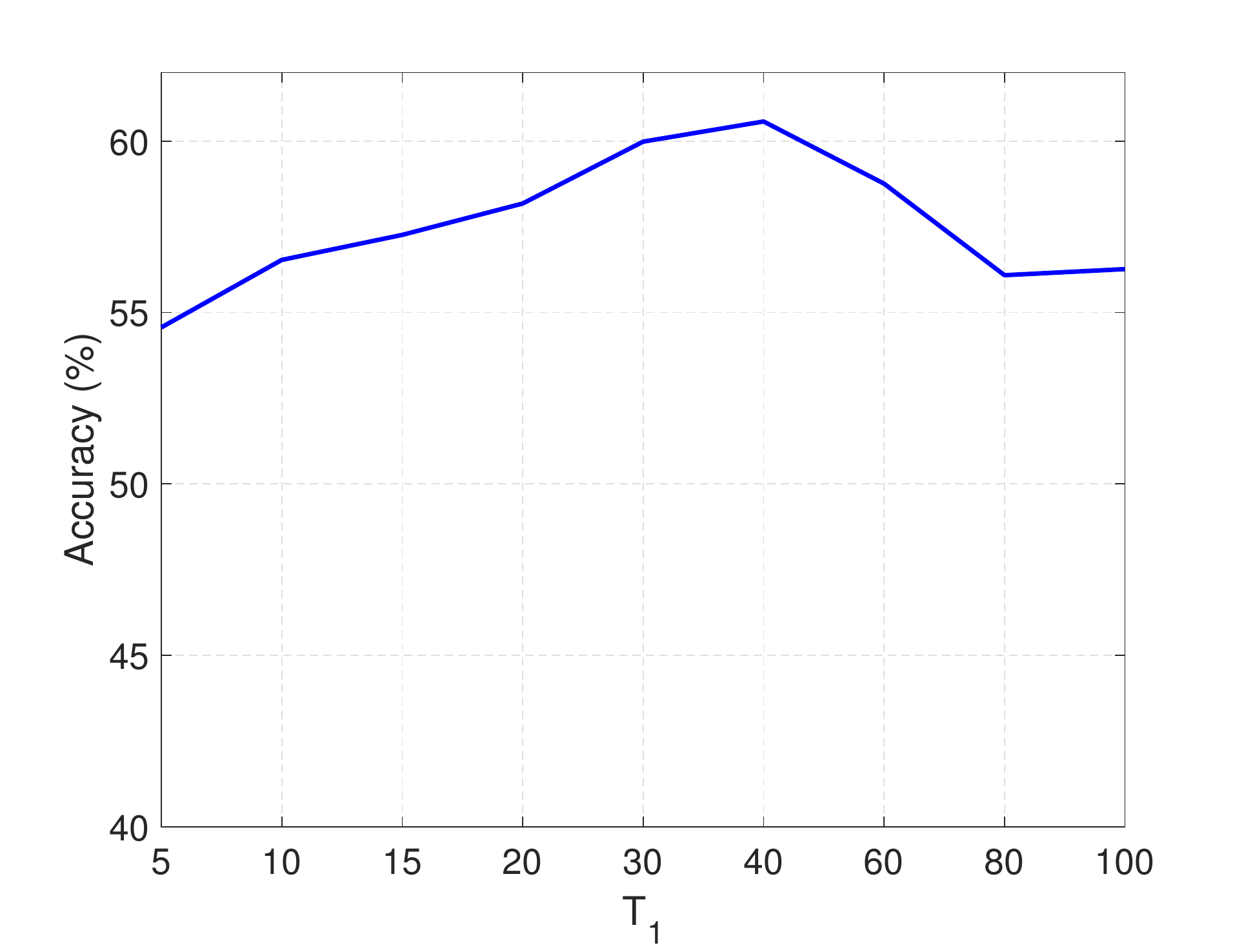}}
	\vspace{-0.5em}
	\caption{The average of testing accuracy of SPRL with different \(T_{1}\) on MNIST and CIFAR-100 over the last ten epochs when labels are corrupted by symmetric flipping with a 50\% noise rate. }
	\label{fig:T}
	\vspace{-1.5em}
\end{figure}

\subsubsection{Comparison with Knowledge Distillation and Label Smooth Regularization}
\label{comkl}
Knowledge distillation \cite{hinton2015distilling}  and label smooth \cite{szegedy2016rethinking} are two popular methods for boosting the model generalization. Here, we utilize Eq. (\ref{eqn:kl}) to distill knowledge from previous training epochs and Eq. (\ref{eqn:lsr}) to smooth labels, and replace Eq. (\ref{eqn:ce_ad}) with them in Eq. (\ref{eqn:loss}), respectively.  They are:

\begin{equation}
\underset{\mathbf{w}}{min} \frac{1}{\left | \emph{B} \right |} \sum_{i\in \emph{B}} \mathbf{p}_{i}^{t-1} log(\frac{\mathbf{p}_{i}^{t-1}}{\mathbf{p}_{i}}),
\label{eqn:kl}
\end{equation} 

\begin{equation}
\underset{\mathbf{w}}{min} \frac{1}{\left | \emph{B} \right |} \sum_{i\in \emph{B}} \mathbf{u}_{i} log(\frac{\mathbf{u}_{i}}{\mathbf{p}_{i}}),
\label{eqn:lsr}
\end{equation} 
where \(\mathbf{w}\) denotes model parameters and \(\mathbf{u}_{i}=\left \{ \frac{1}{c},\frac{1}{c},\cdots, \frac{1}{c} \right \}\in \mathbb{R}^{c}\).

Fig. \ref{fig:kl}  shows their performance using ResNet18 as the backbone network on CIFAR-10 and CIFAR-100 with symmetric label noise. It demonstrates the superior performance of Eq. (\ref{eqn:ce_ad}) over Eq. (\ref{eqn:kl}) and Eq. (\ref{eqn:lsr}).

\renewcommand{\thesubfigure}{\relax}
\begin{figure}[tbp]
	\centering
	\subfigure[(a) MNIST@\(T_{1}=15, \gamma_d=300\)]{\includegraphics[trim={0.5em 0em 2em 1em}, clip, width=0.24\textwidth]{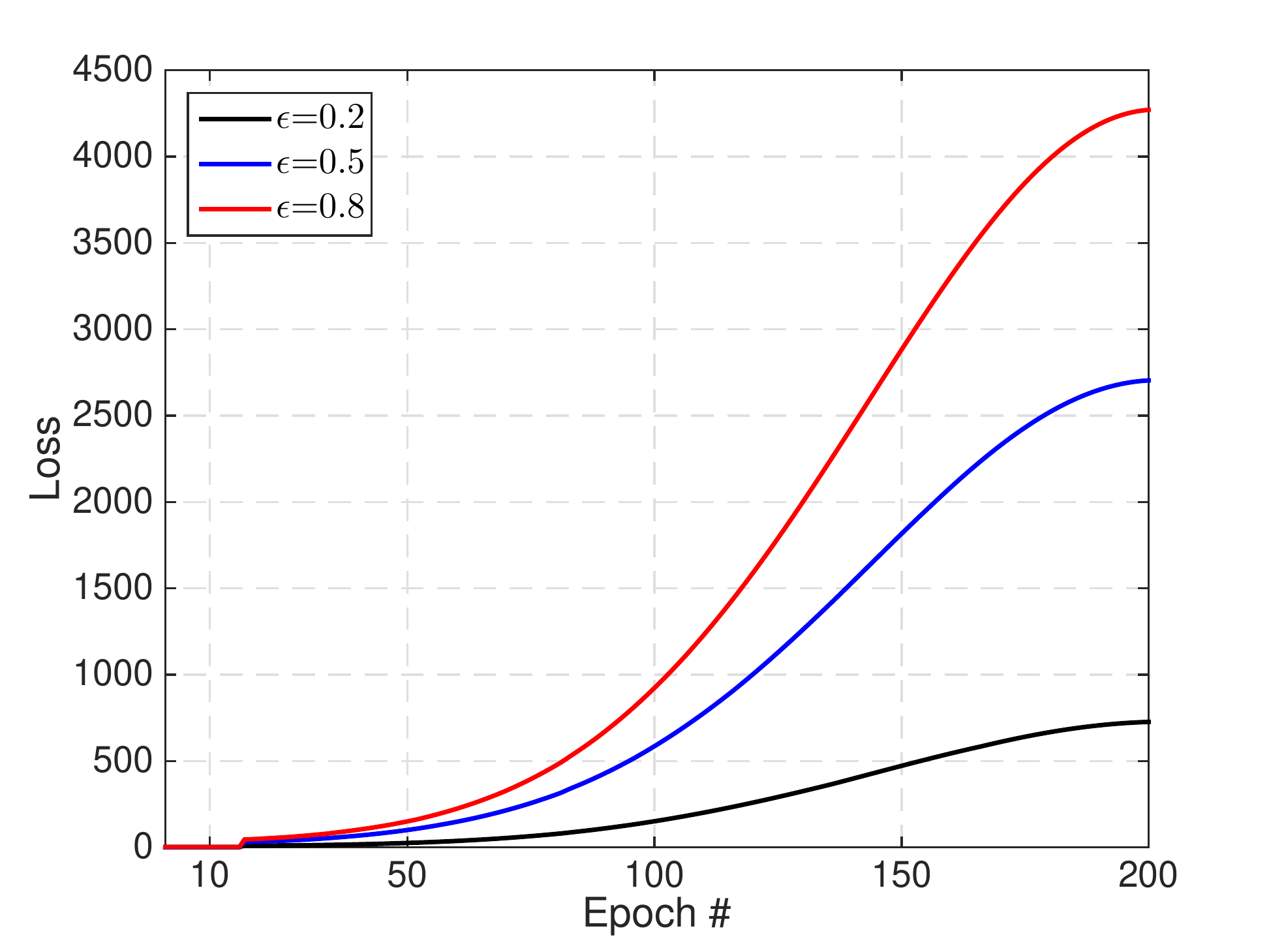}}
	\subfigure[(b) CIFAR-100@\(T_{1}=40, \gamma_d=5\)]{\includegraphics[trim={0.5em 0em 2em 1em}, clip, width=0.24\textwidth]{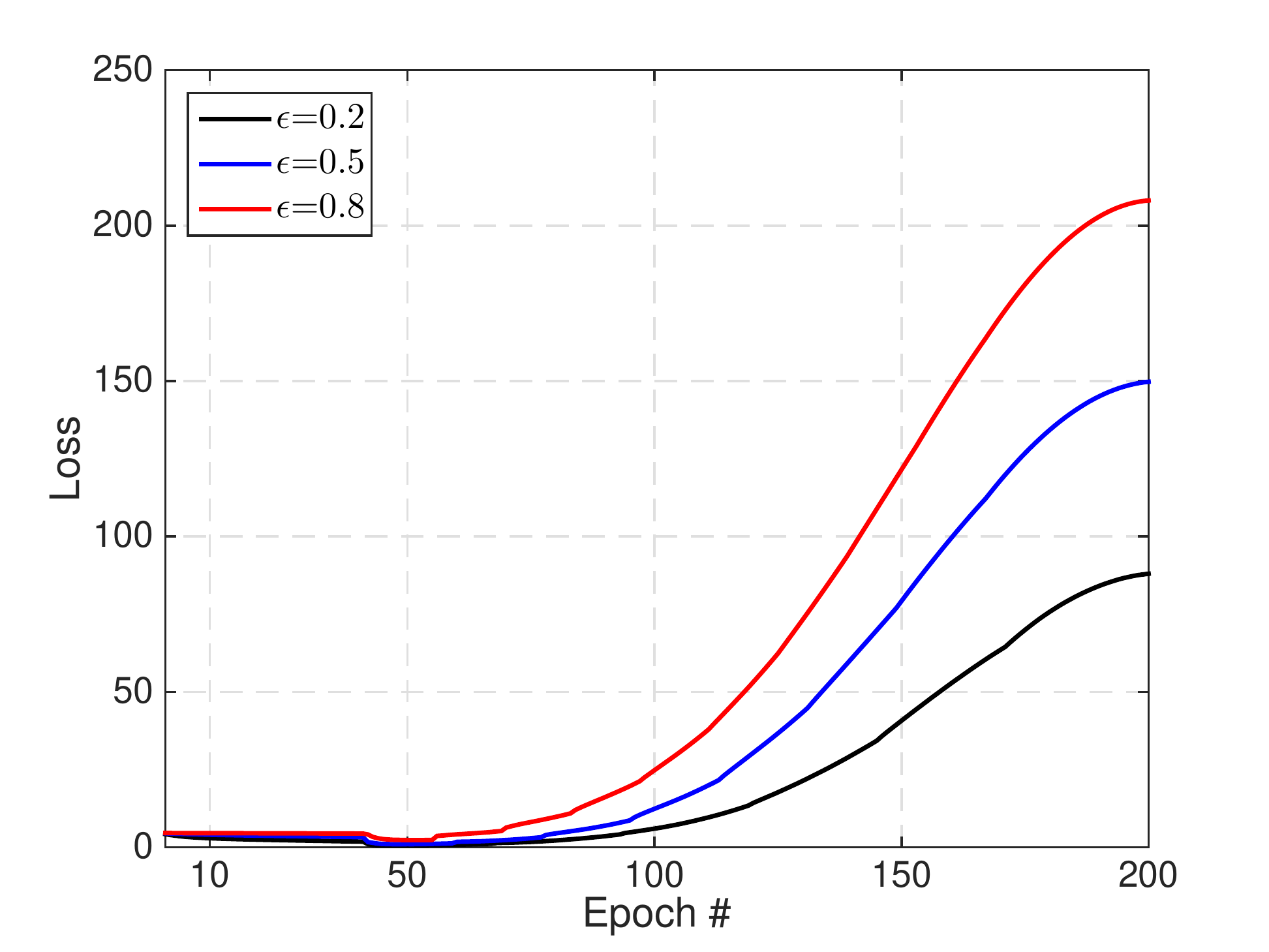}}
	\vspace{-0.5em}
	\caption{The loss of Eq. (\ref{eqn:sub2}) in SPRL with different rates of symmetric label noise on MNIST and CIFAR-100 during training.} 
	\label{fig:loss}
	\vspace{-1.5em}
\end{figure}

%

\renewcommand{\thesubfigure}{\relax}
\begin{figure}[!tbp] 
\centering
\subfigure[(a) CIFAR-10@Symmetry \(\epsilon=0.5\)]{\includegraphics[trim={0.5em 0em 2em 1em}, clip, width=0.24\textwidth]{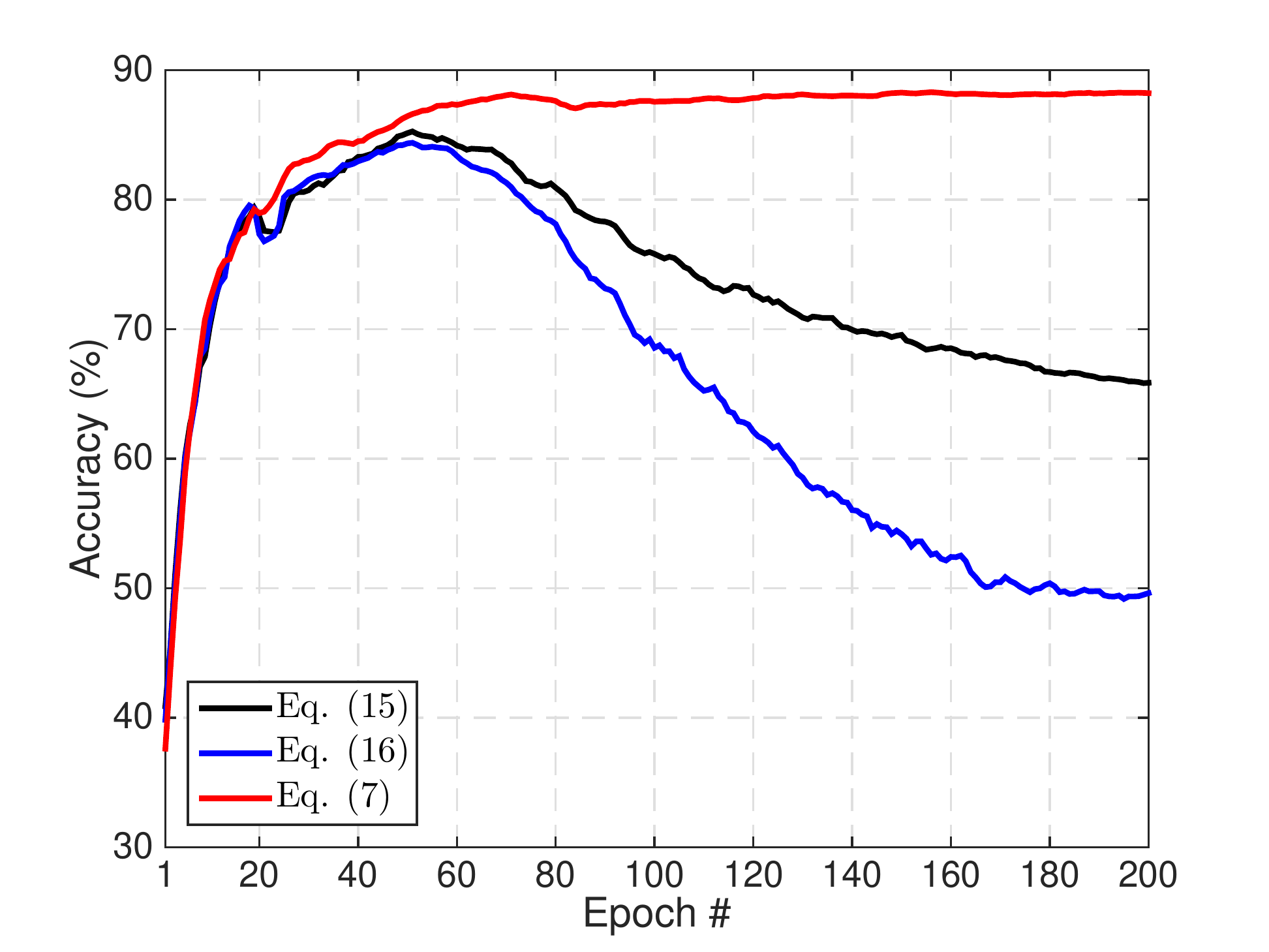}}
\subfigure[(b) CIFAR-100@Symmetry \(\epsilon=0.5\)]{\includegraphics[trim={0.5em 0em 2em 1em}, clip, width=0.24\textwidth]{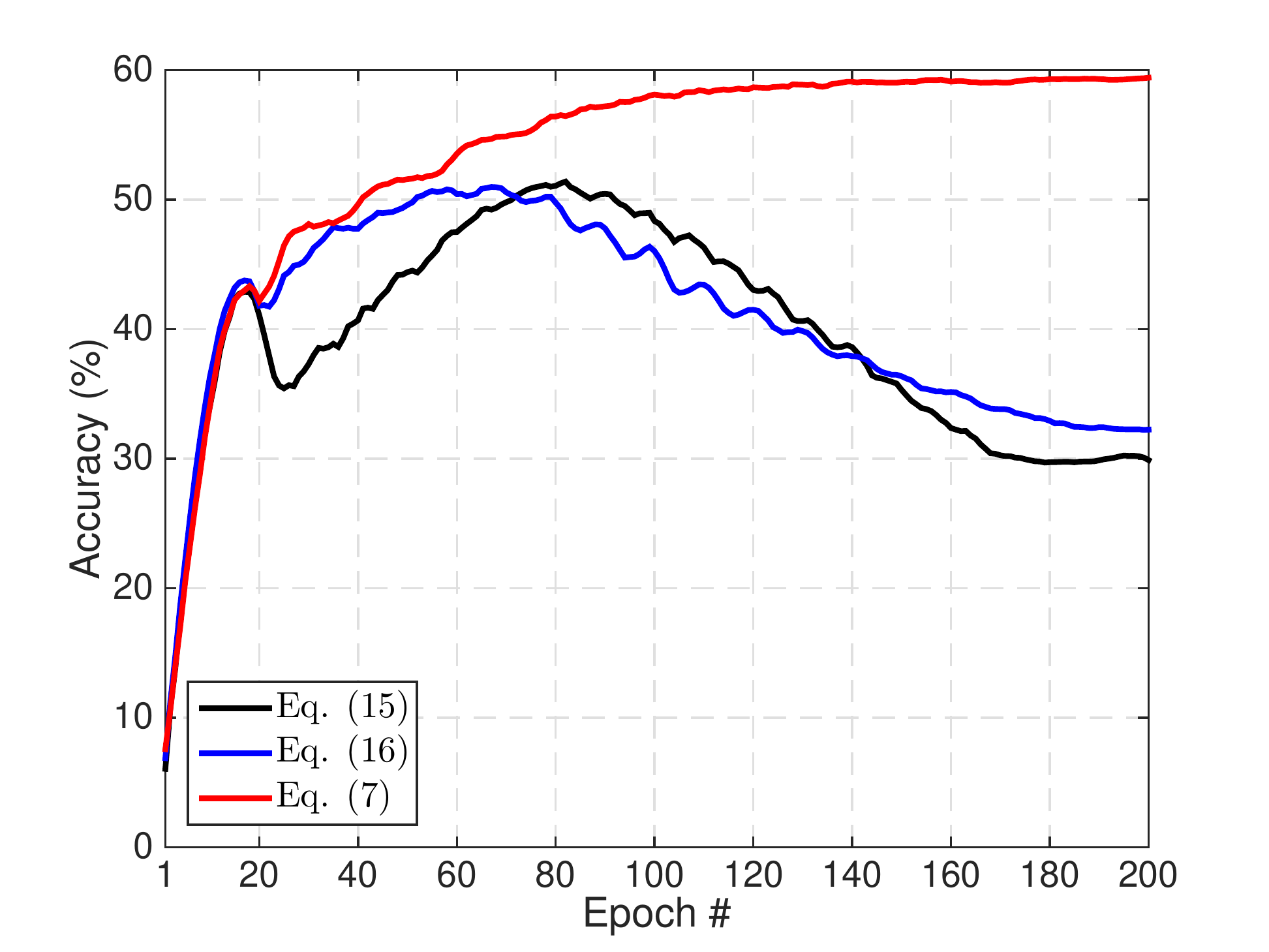}}
\vspace{-1em}
\caption{Testing accuracy of the proposed framework using ResNet18 with Eq. (\ref{eqn:kl}), Eq. (\ref{eqn:lsr}) and Eq. (\ref{eqn:ce_ad}) on CIFAR10 and CIFAR-100 at different numbers of training epochs for symmetry \(\epsilon=0.5\).} 
\label{fig:kl}
\vspace{-1.5em}
\end{figure}

\vspace{-0.5em}
 \subsection{Experiments on Noisy Labels Generated by CNNs}
 \label{section:NL}
In practice, labels might be not only symmetric or pair flipping. To further illustrate the strength of the proposed SPRL, we conduct experiments on noisy labels that are generated by CNNs. Specifically, we uniformly select 4K and 10K images from the training set of CIFAR-10 and CIFAR-100 as labeled data, respectively, and view the remaining images of training sets as unlabeled ones. Next, we only utilize labeled data to train models. Table \ref{table:labeled} presents the accuracy of trained models on training and testing sets of CIFAR-10 and CIFAR-100. Then we apply trained models on the whole training set and utilize predicting labels as noisy labels. Finally, we run the eight methods by utilizing training data with noisy labels to train models.

Table \ref{table:semi} shows the average accuracy of the eight methods on test sets of CIFAR-10 and CIFAR-100 over the last ten epochs. As shown in Tables \ref{table:labeled}-\ref{table:semi}, both SPRL and Co-teaching with noisy labels can consistently outperform ResNet18 and ConvNet with only partially labeled data. However, SPRL always achieves better average accuracy than the best competitor, Co-teaching, on two different deep architectures and datasets, especially on heavy noisy labels, e.g. labels (\(36.84\%\) noise rate) generated by ResNet18, which is trained with only partially labeled data of CIFAR-100.  We also present testing accuracy of the eight methods at different  numbers of training epochs in the supplemental materials (please see Fig. A7). 

\subsection{Experiments on Real-World Noisy Labels}
To be futher demonstrate the strength of the proposed SPRL on boosting model robustnes, we conduct experiments on real-world nosiy labels form the datasets Food101 and Clothing1M, respectively. Specifically, Food101 \cite{bossard14} contains 101,000 images belonging to 101 food categories, with 750 training and 250 testing images per category. Training images are with noisy labels, while testing images have clean labels. Clothing1M \cite{xiao2015learning} contains 1 million clothing images in 14 classes. We utilize training images with noisy labels for model training and 10,000 testing images with clean labels for testing.

Table \ref{table:real} displays the average accuracy of Standard, Co-teaching, Co-teaching+ and SPRL on Food101 and Clothing1M over the last ten epochs. It illustrates that SPRL consistently outperforms the best competitors Co-teaching and Co-teaching+ on real-world noisy labels, especifically using ResNet18 as the backbone network. Additionally, We show their testing accuracy at different  numbers of training epochs in the supplemental materials (please refer to Fig. A8).


\begin{table}[tbp]
	\scriptsize 
	\centering
	\caption{Accuracy (\%) of ResNet18 and ConvNet trained by partially labeled data on training and testing sets of CIFAR-10 and CIFAR-100 datasets (4K for CIFAR-10 and 10K for CIFAR-100).}
	\vspace{-1em}
	\begin{tabular}{|c|c|c|c|c|}
		\hline
		\multirow{3}*{Network}  & \multicolumn{2}{c|}{CIFAR-10} & \multicolumn{2}{c|}{CIFAR-100}  \\      
		\cline{2-5}
		& Training  & Testing   & Training  & Testing     \\
		\cline{2-5}
		\hline
		ResNet18  & 81.97  & 80.93 & 63.16 &54.64  \\
		\hline
		ConvNet  & 82.02 & 80.52   & 64.23  & 54.95 \\
		\hline
	\end{tabular}
	\label{table:labeled}
	\vspace{-1em}
\end{table}

\begin{table}[!tbp]
\scriptsize 
\centering
\caption{Average of testing accuracy (\%) on CIFAR-10 and CIFAR-100 over the last ten epochs by CNN generated noisy labels. We bold the best results and highlight the second best ones via underlines.}
\vspace{-1em}
\begin{tabular}{|c|c|c|c|c|}
\hline
\multirow{3}*{Method}  & \multicolumn{2}{c|}{CIFAR-10} & \multicolumn{2}{c|}{CIFAR-100}  \\      
 \cline{2-5}
 & ResNet18   & ConvNet & ResNet18   & ConvNet   \\
\hline
Standard    &81.94   &81.88   &54.86  & 54.52   \\
\hline
Boostrap  & 81.24   & 81.71  & 54.59 & 55.18 \\
\hline
F-correction  &83.40   & 81.28  &54.20  & 55.25   \\
\hline
Decoupling  & 79.31  &78.46   &49.80  &50.79  \\
\hline
MentorNet &81.27   &80.24   &52.39  & 54.06  \\
\hline
Co-teaching  & \underline{82.80} & \underline{82.80}   & \underline{55.44}  & \underline{55.56}    \\
\hline
Co-teaching+  &82.26 & 81.73 &54.66   & \(55.16\)    \\
\hline
\textbf{SPRL}  & \(\mathbf{85.63}\) & \(\mathbf{84.00}\)   & \(\mathbf{62.08}\)  & \(\mathbf{58.73}\)   \\
\hline
\end{tabular}
\label{table:semi}
\vspace{-1em}
\end{table}

\begin{table}[tbp]
	\scriptsize 
	\centering
	\caption{Average accuracy (\%) of four methods over the last ten epochs on real-world noisy labels.}
	\vspace{-1em}
	\begin{tabular}{|c|c|c|c|c|}
		\hline
		\multirow{3}*{Network}  & \multicolumn{2}{c|}{Food101} & \multicolumn{2}{c|}{Clothing1M}  \\      
		\cline{2-5}
		& ResNet18  & ConvNet   &ResNet18  & ConvNet    \\
		\cline{2-5}
		\hline
		Standard  & \(71.01\)   &\(73.56\)  &\(66.59\)  & \(68.15\) \\
		\hline
		Co-teaching &\(\underline{71.36}\)  &  \(\underline{74.34}\)  & \(\underline{69.78}\)  & \(69.94\)  \\
		\hline
		Co-teaching+ &\(69.84\)  & \(71.10\)   &\(67.93\)   & \(\underline{70.08}\) \\
		\hline
		 \textbf{SPRL} &\(\mathbf{76.14}\)  & \(\mathbf{74.61}\)    &\(\mathbf{71.63}\)   & \(\mathbf{71.81}\) \\
	    \hline	
	\end{tabular}
	\label{table:real}
	\vspace{-1em}
\end{table}
%

\vspace{-0.5em}
\subsection{Discussion and Future Work}
Experiments on multiple large-scale benchmark datasets and two different backbone network architectures demonstrate that SPRL can significantly reduce the effects of various types of corrupted labels by using the resistance loss to alleviate model overfitting, thus avoiding the performance degradation of CNNs during training. Additionally, experiments on noisy labels generated by CNNs suggest that SPRL can be potentially utilized to further improve the performance of semi-supervised and unsupervised deep methods. 

Although SPRL has achieved robust and better accuracy than many state-of-the-art methods, SPRL cannot be directly applied on multi-label datasets with noisy labels, because it calculates the class probability of each sample by using the softmax function, which usually performs poorly on multi-label classification tasks. However, SPRL might be extended to handle multi-label tasks by replacing the softmax function with a sigmoid function. In the future, SPRL might be further improved based on the following two potential directions: (i) Introducing a small amount of clean validation data for training \cite{zhang2020distilling}, instead of training models with only noisy training data without using any clean validation data. (ii) Employing model predictions of SPRL to generate labels to further boost the model performance (like Section \ref{section:NL}), or distinguishing and changing the possibly corrupted labels by using the other popular methods \cite{arazo2019unsupervised} \cite{berthelot2019mixmatch} \cite{li2020dividemix} \cite{sohn2020fixmatch}.
\vspace{-0.5em}
\section{Conclusion}
In this paper, we propose a novel framework, SPRL, to alleviate model overfitting for robustly training CNNs on noisy labels. The proposed framework contains two major modules: curriculum learning, which utilizes the memorization skill of deep neural networks to learn a curriculum to provide meaningful supervision for other training samples; parameters update, which leverages the selected confident samples  and a resistance loss to simultaneously update model parameters and significantly reduce the effect of corrupted labels. Experiments on multiple large-scale benchmark datasets and typical deep architectures demonstrate the effectiveness of the proposed framework, and its significantly superior performance over recent state-of-the-art methods.

\vspace{-0.5em}
{
\bibliographystyle{IEEEtran}
\bibliography{egbib}

\begin{thebibliography}{10}
\providecommand{\url}[1]{#1}
\csname url@samestyle\endcsname
\providecommand{\newblock}{\relax}
\providecommand{\bibinfo}[2]{#2}
\providecommand{\BIBentrySTDinterwordspacing}{\spaceskip=0pt\relax}
\providecommand{\BIBentryALTinterwordstretchfactor}{4}
\providecommand{\BIBentryALTinterwordspacing}{\spaceskip=\fontdimen2\font plus
\BIBentryALTinterwordstretchfactor\fontdimen3\font minus
  \fontdimen4\font\relax}
\providecommand{\BIBforeignlanguage}[2]{{%
\expandafter\ifx\csname l@#1\endcsname\relax
\typeout{** WARNING: IEEEtran.bst: No hyphenation pattern has been}%
\typeout{** loaded for the language `#1'. Using the pattern for}%
\typeout{** the default language instead.}%
\else
\language=\csname l@#1\endcsname
\fi
#2}}
\providecommand{\BIBdecl}{\relax}
\BIBdecl

\bibitem{krizhevsky2012imagenet}
A.~Krizhevsky, I.~Sutskever, and G.~E. Hinton, ``Imagenet classification with
  deep convolutional neural networks,'' in \emph{Advances in Neural Information
  Processing Systems}, 2012, pp. 1097--1105.

\bibitem{szegedy2016rethinking}
C.~Szegedy, V.~Vanhoucke, S.~Ioffe, J.~Shlens, and Z.~Wojna, ``Rethinking the
  inception architecture for computer vision,'' in \emph{IEEE conference on
  Computer Vision and Pattern Recognition}, 2016, pp. 2818--2826.

\bibitem{he2016deep}
K.~He, X.~Zhang, S.~Ren, and J.~Sun, ``Deep residual learning for image
  recognition,'' in \emph{IEEE Conference on Computer Vision and Pattern
  Recognition}, 2016, pp. 770--778.

\bibitem{Feng_2018_CVPR}
Y.~Feng, Z.~Zhang, X.~Zhao, R.~Ji, and Y.~Gao, ``Gvcnn: Group-view
  convolutional neural networks for 3d shape recognition,'' in
  \emph{Proceedings of the IEEE Conference on Computer Vision and Pattern
  Recognition (CVPR)}, June 2018.

\bibitem{feng2019hypergraph}
Y.~Feng, H.~You, Z.~Zhang, R.~Ji, and Y.~Gao, ``Hypergraph neural networks,''
  in \emph{Proceedings of the AAAI Conference on Artificial Intelligence},
  vol.~33, no.~01, 2019, pp. 3558--3565.

\bibitem{shi2018pairwise}
X.~Shi, M.~Sapkota, F.~Xing, F.~Liu, L.~Cui, and L.~Yang, ``Pairwise based deep
  ranking hashing for histopathology image classification and retrieval,''
  \emph{Pattern Recognition}, vol.~81, pp. 14--22, 2018.

\bibitem{shi2020anchor}
X.~Shi, Z.~Guo, F.~Xing, Y.~Liang, and L.~Yang, ``Anchor-based self-ensembling
  for semi-supervised deep pairwise hashing,'' \emph{International Journal of
  Computer Vision}, pp. 1--18, 2020.

\bibitem{girshick2014rich}
R.~Girshick, J.~Donahue, T.~Darrell, and J.~Malik, ``Rich feature hierarchies
  for accurate object detection and semantic segmentation,'' in \emph{IEEE
  Conference on Computer Vision and Pattern Recognition}, 2014, pp. 580--587.

\bibitem{long2015fully}
J.~Long, E.~Shelhamer, and T.~Darrell, ``Fully convolutional networks for
  semantic segmentation,'' in \emph{IEEE Conference on Computer Vision and
  Pattern Recognition}, 2015, pp. 3431--3440.

\bibitem{zhang2016understanding}
C.~Zhang, S.~Bengio, M.~Hardt, B.~Recht, and O.~Vinyals, ``Understanding deep
  learning requires rethinking generalization,'' in \emph{International
  Conference on Learning Representations}, 2017.

\bibitem{reed2014training}
S.~Reed, H.~Lee, D.~Anguelov, C.~Szegedy, D.~Erhan, and A.~Rabinovich,
  ``Training deep neural networks on noisy labels with bootstrapping,'' in
  \emph{Workshop of International Conference on Learning Representations},
  2015.

\bibitem{laine2016temporal}
S.~Laine and T.~Aila, ``Temporal ensembling for semi-supervised learning,'' in
  \emph{International Conference on Learning Representations}, 2016.

\bibitem{patrini2017making}
G.~Patrini, A.~Rozza, A.~Krishna~Menon, R.~Nock, and L.~Qu, ``Making deep
  neural networks robust to label noise: A loss correction approach,'' in
  \emph{IEEE Conference on Computer Vision and Pattern Recognition}, 2017, pp.
  1944--1952.

\bibitem{jiang2017mentornet}
L.~Jiang, Z.~Zhou, T.~Leung, L.-J. Li, and L.~Fei-Fei, ``Mentornet: Learning
  data-driven curriculum for very deep neural networks on corrupted labels,''
  in \emph{International Conference on Machine Learning}, 2018.

\bibitem{han2018co}
B.~Han, Q.~Yao, X.~Yu, G.~Niu, M.~Xu, W.~Hu, I.~Tsang, and M.~Sugiyama,
  ``Co-teaching: Robust training of deep neural networks with extremely noisy
  labels,'' in \emph{Advances in Neural Information Processing Systems}, 2018,
  pp. 8527--8537.

\bibitem{chapelle2009semi}
O.~Chapelle, B.~Scholkopf, and A.~Zien, ``Semi-supervised learning (chapelle,
  o. et al., eds.; 2006)[book reviews],'' \emph{IEEE Transactions on Neural
  Networks}, vol.~20, no.~3, pp. 542--542, 2009.

\bibitem{arpit2017closer}
D.~e.~a. Arpit, ``A closer look at memorization in deep networks,'' in
  \emph{International Conference on Machine Learning}.\hskip 1em plus 0.5em
  minus 0.4em\relax PMLR, 2017, pp. 233--242.

\bibitem{malach2017decoupling}
E.~Malach and S.~Shalev-Shwartz, ``Decoupling" when to update" from" how to
  update",'' in \emph{Advances in Neural Information Processing Systems}, 2017,
  pp. 960--970.

\bibitem{yu2019does}
X.~Yu, B.~Han, J.~Yao, G.~Niu, I.~Tsang, and M.~Sugiyama, ``How does
  disagreement help generalization against label corruption?'' in
  \emph{International Conference on Machine Learning}, 2019, pp. 7164--7173.

\bibitem{frenay2013classification}
B.~Fr{\'e}nay and M.~Verleysen, ``Classification in the presence of label
  noise: a survey,'' \emph{IEEE Transactions on Neural Networks and Learning
  Systems}, vol.~25, no.~5, pp. 845--869, 2013.

\bibitem{raykar2010learning}
V.~C. Raykar, S.~Yu, L.~H. Zhao, G.~H. Valadez, C.~Florin, L.~Bogoni, and
  L.~Moy, ``Learning from crowds,'' \emph{Journal of Machine Learning
  Research}, vol.~11, no. Apr, pp. 1297--1322, 2010.

\bibitem{natarajan2013learning}
N.~Natarajan, I.~S. Dhillon, P.~K. Ravikumar, and A.~Tewari, ``Learning with
  noisy labels,'' in \emph{Advances in Neural Information Processing Systems},
  2013, pp. 1196--1204.

\bibitem{masnadi2009design}
H.~Masnadi-Shirazi and N.~Vasconcelos, ``On the design of loss functions for
  classification: theory, robustness to outliers, and savageboost,'' in
  \emph{Advances in Neural Information Processing Systems}, 2009, pp.
  1049--1056.

\bibitem{van2015learning}
B.~Van~Rooyen, A.~Menon, and R.~C. Williamson, ``Learning with symmetric label
  noise: The importance of being unhinged,'' in \emph{Advances in Neural
  Information Processing Systems}, 2015, pp. 10--18.

\bibitem{scott2013classification}
C.~Scott, G.~Blanchard, and G.~Handy, ``Classification with asymmetric label
  noise: Consistency and maximal denoising,'' in \emph{Conference On Learning
  Theory}, 2013, pp. 489--511.

\bibitem{ramaswamy2016mixture}
H.~Ramaswamy, C.~Scott, and A.~Tewari, ``Mixture proportion estimation via
  kernel embeddings of distributions,'' in \emph{International Conference on
  Machine Learning}, 2016, pp. 2052--2060.

\bibitem{sanderson2014class}
T.~Sanderson and C.~Scott, ``Class proportion estimation with application to
  multiclass anomaly rejection,'' in \emph{Artificial Intelligence and
  Statistics}, 2014, pp. 850--858.

\bibitem{liu2015classification}
T.~Liu and D.~Tao, ``Classification with noisy labels by importance
  reweighting,'' \emph{IEEE Transactions on Pattern Analysis and Machine
  Intelligence}, vol.~38, no.~3, pp. 447--461, 2015.

\bibitem{mnih2012learning}
V.~Mnih and G.~E. Hinton, ``Learning to label aerial images from noisy data,''
  in \emph{International Conference on Machine Learning}, 2012, pp. 567--574.

\bibitem{ghosh2017robust}
A.~Ghosh, H.~Kumar, and P.~Sastry, ``Robust loss functions under label noise
  for deep neural networks,'' in \emph{AAAI Conference on Artificial
  Intelligence}, vol.~31, no.~1, 2017.

\bibitem{ghosh2015making}
A.~Ghosh, N.~Manwani, and P.~Sastry, ``Making risk minimization tolerant to
  label noise,'' \emph{Neurocomputing}, vol. 160, pp. 93--107, 2015.

\bibitem{manwani2013noise}
N.~Manwani and P.~Sastry, ``Noise tolerance under risk minimization,''
  \emph{IEEE transactions on Cybernetics}, vol.~43, no.~3, pp. 1146--1151,
  2013.

\bibitem{shi2020graph}
X.~Shi, H.~Su, F.~Xing, Y.~Liang, G.~Qu, and L.~Yang, ``Graph temporal
  ensembling based semi-supervised convolutional neural network with noisy
  labels for histopathology image analysis,'' \emph{Medical Image Analysis},
  vol.~60, p. 101624, 2020.

\bibitem{sukhbaatar2014training}
S.~Sukhbaatar, J.~Bruna, M.~Paluri, L.~Bourdev, and R.~Fergus, ``Training
  convolutional networks with noisy labels,'' 2015.

\bibitem{ma2018dimensionality}
X.~Ma, Y.~Wang, M.~E. Houle, S.~Zhou, S.~M. Erfani, S.-T. Xia, S.~Wijewickrema,
  and J.~Bailey, ``Dimensionality-driven learning with noisy labels,'' in
  \emph{International Conference on Machine Learning}, 2018.

\bibitem{wang2018iterative}
Y.~Wang, W.~Liu, X.~Ma, J.~Bailey, H.~Zha, L.~Song, and S.-T. Xia, ``Iterative
  learning with open-set noisy labels,'' in \emph{IEEE Conference on Computer
  Vision and Pattern Recognition}, 2018, pp. 8688--8696.

\bibitem{xiao2015learning}
T.~Xiao, T.~Xia, Y.~Yang, C.~Huang, and X.~Wang, ``Learning from massive noisy
  labeled data for image classification,'' in \emph{Proceedings of the IEEE
  conference on computer vision and pattern recognition}, 2015, pp. 2691--2699.

\bibitem{li2017learning}
Y.~Li, J.~Yang, Y.~Song, L.~Cao, J.~Luo, and L.-J. Li, ``Learning from noisy
  labels with distillation,'' in \emph{IEEE International Conference on
  Computer Vision}, 2017, pp. 1910--1918.

\bibitem{veit2017learning}
A.~Veit, N.~Alldrin, G.~Chechik, I.~Krasin, A.~Gupta, and S.~Belongie,
  ``Learning from noisy large-scale datasets with minimal supervision,'' in
  \emph{IEEE Conference on Computer Vision and Pattern Recognition}, 2017, pp.
  839--847.

\bibitem{vahdat2017toward}
A.~Vahdat, ``Toward robustness against label noise in training deep
  discriminative neural networks,'' in \emph{Advances in Neural Information
  Processing Systems}, 2017, pp. 5596--5605.

\bibitem{ren2018learning}
M.~Ren, W.~Zeng, B.~Yang, and R.~Urtasun, ``Learning to reweight examples for
  robust deep learning,'' in \emph{International Conference on Machine
  Learning}, 2018.

\bibitem{northcutt2017learning}
C.~G. Northcutt, T.~Wu, and I.~L. Chuang, ``Learning with confident examples:
  Rank pruning for robust classification with noisy labels,'' in
  \emph{Uncertainty in Artificial Intelligence}, 2017.

\bibitem{zhang2018generalized}
Z.~Zhang and M.~Sabuncu, ``Generalized cross entropy loss for training deep
  neural networks with noisy labels,'' in \emph{Advances in Neural Information
  Processing Systems}, 2018, pp. 8778--8788.

\bibitem{yao2020searching}
Q.~Yao, H.~Yang, B.~Han, G.~Niu, and J.~T.-Y. Kwok, ``Searching to exploit
  memorization effect in learning with noisy labels,'' in \emph{International
  Conference on Machine Learning}.\hskip 1em plus 0.5em minus 0.4em\relax PMLR,
  2020, pp. 10\,789--10\,798.

\bibitem{bengio2009curriculum}
Y.~Bengio, J.~Louradour, R.~Collobert, and J.~Weston, ``Curriculum learning,''
  in \emph{International Conference on Machine Learning}.\hskip 1em plus 0.5em
  minus 0.4em\relax ACM, 2009, pp. 41--48.

\bibitem{kumar2010self}
M.~P. Kumar, B.~Packer, and D.~Koller, ``Self-paced learning for latent
  variable models,'' in \emph{Advances in Neural Information Processing
  Systems}, 2010, pp. 1189--1197.

\bibitem{jiang2015self}
L.~Jiang, D.~Meng, Q.~Zhao, S.~Shan, and A.~G. Hauptmann, ``Self-paced
  curriculum learning,'' in \emph{AAAI Conference on Artificial Intelligence},
  2015.

\bibitem{hinton2015distilling}
G.~Hinton, O.~Vinyals, and J.~Dean, ``Distilling the knowledge in a neural
  network,'' \emph{arXiv preprint arXiv:1503.02531}, 2015.

\bibitem{ashok2017n2n}
A.~Ashok, N.~Rhinehart, F.~Beainy, and K.~M. Kitani, ``N2n learning: Network to
  network compression via policy gradient reinforcement learning,'' \emph{arXiv
  preprint arXiv:1709.06030}, 2017.

\bibitem{polino2018model}
A.~Polino, R.~Pascanu, and D.~Alistarh, ``Model compression via distillation
  and quantization,'' \emph{arXiv preprint arXiv:1802.05668}, 2018.

\bibitem{lee2019overcoming}
K.~Lee, K.~Lee, J.~Shin, and H.~Lee, ``Overcoming catastrophic forgetting with
  unlabeled data in the wild,'' in \emph{IEEE/CVF International Conference on
  Computer Vision}, 2019, pp. 312--321.

\bibitem{dong2019distillation}
B.~Dong, J.~Hou, Y.~Lu, and Z.~Zhang, ``Distillation $\approx$ early stopping?
  harvesting dark knowledge utilizing anisotropic information retrieval for
  overparameterized neural network,'' \emph{arXiv preprint arXiv:1910.01255},
  2019.

\bibitem{kim2020self}
K.~Kim, B.~Ji, D.~Yoon, and S.~Hwang, ``Self-knowledge distillation: A simple
  way for better generalization,'' \emph{arXiv preprint arXiv:2006.12000},
  2020.

\bibitem{lecun1998gradient}
Y.~LeCun, L.~Bottou, Y.~Bengio, P.~Haffner \emph{et~al.}, ``Gradient-based
  learning applied to document recognition,'' \emph{Proceedings of the IEEE},
  vol.~86, no.~11, pp. 2278--2324, 1998.

\bibitem{krizhevsky2009learning}
A.~Krizhevsky and G.~Hinton, ``Learning multiple layers of features from tiny
  images,'' Citeseer, Tech. Rep., 2009.

\bibitem{kingma2014adam}
D.~P. Kingma and J.~Ba, ``Adam: A method for stochastic optimization,'' in
  \emph{International Conference on Learning Representations}, 2015.

\bibitem{vinyals2016matching}
O.~Vinyals, C.~Blundell, T.~Lillicrap, D.~Wierstra \emph{et~al.}, ``Matching
  networks for one shot learning,'' in \emph{Advances in Neural Information
  Processing Systems}, 2016, pp. 3630--3638.

\bibitem{russakovsky2015imagenet}
O.~Russakovsky, J.~Deng, H.~Su, J.~Krause, S.~Satheesh, S.~Ma, Z.~Huang,
  A.~Karpathy, A.~Khosla, M.~Bernstein \emph{et~al.}, ``Imagenet large scale
  visual recognition challenge,'' \emph{International Journal of Computer
  Vision}, vol. 115, no.~3, pp. 211--252, 2015.

\bibitem{rasmus2015semi}
A.~Rasmus, M.~Berglund, M.~Honkala, H.~Valpola, and T.~Raiko, ``Semi-supervised
  learning with ladder networks,'' in \emph{Advances in Neural Information
  Processing Systems}, 2015, pp. 3546--3554.

\bibitem{bossard14}
L.~Bossard, M.~Guillaumin, and L.~Van~Gool, ``Food-101 -- mining discriminative
  components with random forests,'' in \emph{European Conference on Computer
  Vision}, 2014.

\bibitem{zhang2020distilling}
Z.~Zhang, H.~Zhang, S.~O. Arik, H.~Lee, and T.~Pfister, ``Distilling effective
  supervision from severe label noise,'' in \emph{Proceedings of the IEEE/CVF
  Conference on Computer Vision and Pattern Recognition}, 2020, pp. 9294--9303.

\bibitem{arazo2019unsupervised}
E.~Arazo, D.~Ortego, P.~Albert, N.~O?Connor, and K.~McGuinness, ``Unsupervised
  label noise modeling and loss correction,'' in \emph{International Conference
  on Machine Learning}.\hskip 1em plus 0.5em minus 0.4em\relax PMLR, 2019, pp.
  312--321.

\bibitem{berthelot2019mixmatch}
D.~Berthelot, N.~Carlini, I.~Goodfellow, N.~Papernot, A.~Oliver, and C.~Raffel,
  ``Mixmatch: A holistic approach to semi-supervised learning,'' \emph{arXiv
  preprint arXiv:1905.02249}, 2019.

\bibitem{li2020dividemix}
J.~Li, R.~Socher, and S.~C. Hoi, ``Dividemix: Learning with noisy labels as
  semi-supervised learning,'' \emph{arXiv preprint arXiv:2002.07394}, 2020.

\bibitem{sohn2020fixmatch}
K.~Sohn, D.~Berthelot, C.-L. Li, Z.~Zhang, N.~Carlini, E.~D. Cubuk, A.~Kurakin,
  H.~Zhang, and C.~Raffel, ``Fixmatch: Simplifying semi-supervised learning
  with consistency and confidence,'' \emph{arXiv preprint arXiv:2001.07685},
  2020.

\bibitem{he2015delving}
K.~He, X.~Zhang, S.~Ren, and J.~Sun, ``Delving deep into rectifiers: Surpassing
  human-level performance on imagenet classification,'' in \emph{IEEE
  Conference on Computer Vision}, 2015, pp. 1026--1034.

\end{thebibliography}
}

\appendix
\renewcommand{\appendixname}{Appendix~\Alph{section}}

\begin{table}[htbp] \renewcommand\thetable{A1}
	\scriptsize 
	\centering
	\caption{ResNet18.}
	\vspace{0.2em}
	\begin{tabular}{c|c}
		\hline
		\textbf{Layer} & \textbf{Hyperparameters} \\
		\hline
		1 & conv(3, 1, 1)-64+ReLU \\
		\hline
		3 &  conv(3, 1, 1)-64+ReLU \\ 
		\hline  
		4 & conv(3, 1, 1)-64+ReLU \\ 
		\hline
		5 &  conv(3, 1, 1)-64+ReLU \\ 
		\hline  
		6 & conv(3, 1, 1)-64+ReLU \\ 
		\hline
		7 &  conv(3, 2, 1)-128+ReLU     \\
		\hline
		8 &  conv(3, 1, 1)-128+ReLU     \\
		\hline
		9 &  conv(3, 1, 1)-128+ReLU     \\
		\hline
		10 &  conv(3, 1, 1)-128+ReLU     \\
		\hline
		11 &  conv(3, 2, 1)-256+ReLU     \\
		\hline
		12 &  conv(3, 1, 1)-256+ReLU     \\
		\hline
		13 &  conv(3, 1, 1)-256+ReLU     \\
		\hline
		14 &  conv(3, 1, 1)-256+ReLU     \\
		\hline
		15 &  conv(3, 2, 1)-512+ReLU     \\
		\hline
		16 &  conv(3, 1, 1)-512+ReLU     \\
		\hline
		17 &  conv(3, 1, 1)-512+ReLU     \\
		\hline
		18 &  conv(3, 1, 1)-512+ReLU     \\
		\hline
		19 & avgpool \\
		\hline
		20 & fc-\(c\)   \\
		\hline
	\end{tabular}
	\label{table:resnet18}
\end{table}

\begin{table}[htbp]\renewcommand\thetable{A2}
	\scriptsize 
	\centering
	\caption{ConvNet.}
	\begin{tabular}{c|c}
		\hline
		\textbf{Layer} & \textbf{Hyperparameters} \\
		\hline
		1 & conv(3, 1, 1)-128+LReLU(\(\alpha =0.1\)) \\
		\hline
		2 & conv(3, 1, 1)-128+LReLU(\(\alpha =0.1\)) \\
		\hline
		3 & conv(3, 1, 1)-128+LReLU(\(\alpha =0.1\)) \\
		\hline  
		4 & maxpool(2, 2) \\
		\hline
		5 & dropout (\(p=0.5\))  \\
		\hline
		6 &  conv(3, 1, 1)-256+LReLU(\(\alpha =0.1\))    \\
		\hline
		7 &  conv(3, 1, 1)-256+LReLU(\(\alpha =0.1\))   \\
		\hline
		8 &  conv(3, 1, 1)-256+LReLU(\(\alpha =0.1\))   \\
		\hline
		9 & maxpool(2, 2) \\
		\hline
		10 & dropout (\(p=0.5\))  \\
		\hline
		11 &  conv(3, 1, 0)-512+LReLU(\(\alpha =0.1\))    \\
		\hline
		12 &  conv(1, 1, 0)-256+LReLU(\(\alpha =0.1\))   \\
		\hline
		13 &  conv(1, 1, 0)-128+LReLU(\(\alpha =0.1\))   \\
		\hline
		14 & avgpool  \\
		\hline
		\hline
		15& fc-\(c\)  \\
		\hline
	\end{tabular}
	\label{table:convnet}
\end{table}

\begin{table*}[htb] \renewcommand\thetable{A5}
	\scriptsize 
	\centering
	\caption{Average of testing accuracy (\%) of SPRL, Co-teaching and Co-teaching+ on CIFAR-10 and CIFAR-100 using ResNet18 and without using data augmentation.}
	\vspace{-0.5em}
	\begin{tabular}{|c|c|c|c|c|}
		\hline
		\multirow{2}*{Method}&\multicolumn{3}{c|}{Symmetry}   & Pair \\
		\cline{2-5}
		&\(\epsilon=0.2\) & \(\epsilon=0.5\)  & \(\epsilon=0.8\)   & \(\epsilon=0.45\)  \\
		\hline
		\multicolumn{5}{|c|}{CIFAR-10} \\
		\hline
		Co-teaching& \(78.48\pm 0.18\) &\(68.55\pm 0.06\) &\(19.63\pm 0.14\)  & \(67.99\pm 0.31\)   \\
		\hline
		Co-teaching+& \(74.14\pm 0.22\) & \(46.69\pm 0.59\) & \(16.74\pm 0.08\)  & \(45.44\pm 0.32\)   \\
		\hline
		SPRL&\(\mathbf{84.51\pm0.12}\) & \(\mathbf{71.91\pm 0.29}\) &\(\mathbf{32.39\pm 0.41}\)  & \(\mathbf{79.20\pm 0.14}\)  \\
		\hline
		\multicolumn{5}{|c|}{CIFAR-100}  \\
		\hline
		Co-teaching&\(47.12\pm 0.16\) &\(33.95\pm 0.17\)  &\(13.34\pm 0.08\)  &\(30.19\pm 0.10\)    \\
		\hline
		Co-teaching+& \(48.39\pm 0.08\) &\(30.81\pm 0.36\)  & \(6.68\pm 0.08\) & \(25.89\pm 0.14\)   \\
		\hline
		SPRL&\(\mathbf{59.68\pm 0.11}\) &\(\mathbf{42.47\pm 0.21}\)  & \(\mathbf{16.05\pm 0.14}\) & \(\mathbf{41.64\pm 0.04}\)   \\
		\hline
	\end{tabular}
	\label{table:noaug}
\end{table*}

\begin{figure*}[tbp]\renewcommand\thefigure{A1}
	\includegraphics[trim={39em 5em 36em 23em}, clip, width=0.1\textwidth]{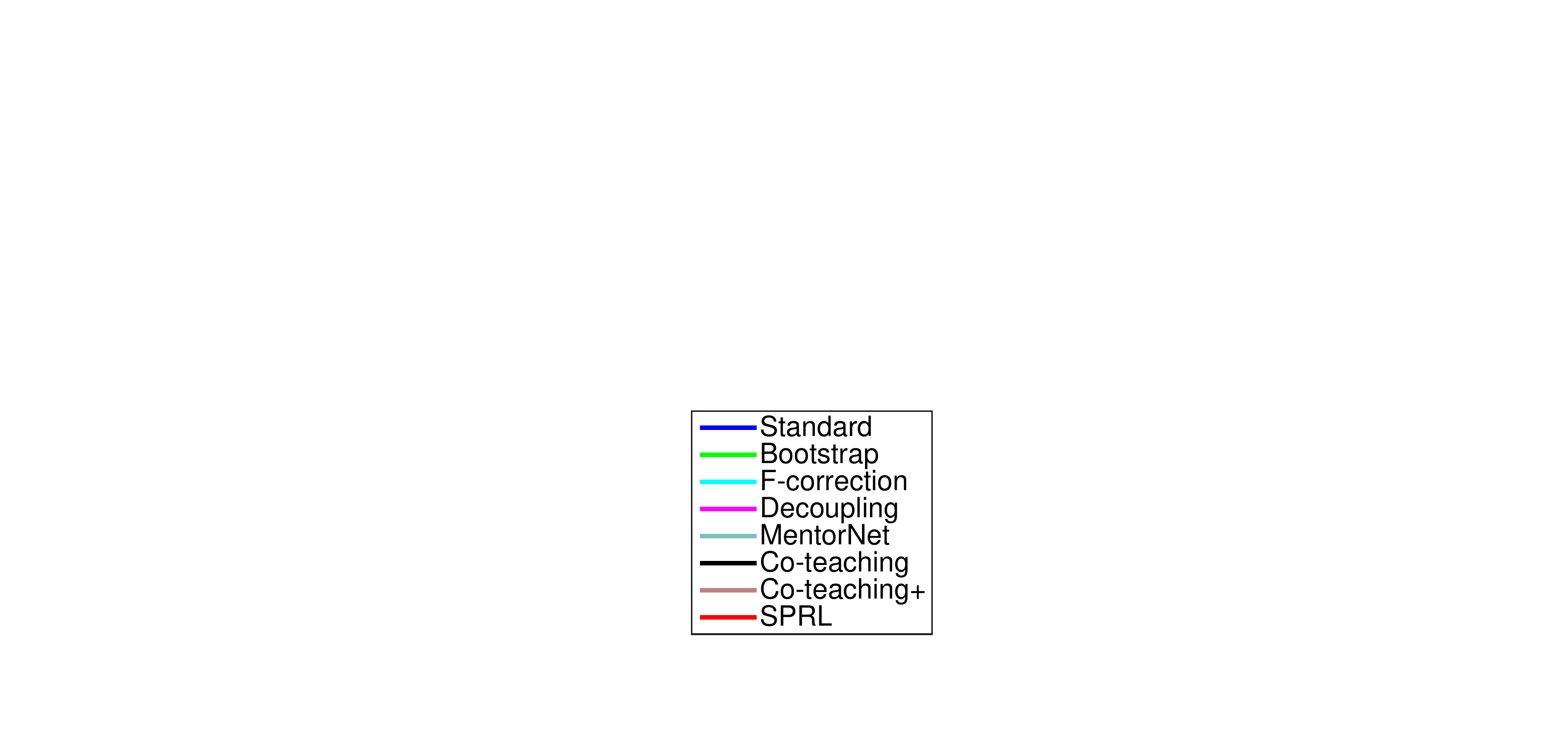}
	\subfigure[ResNet18/Symmetry \(\epsilon\)=0.2]{\includegraphics[trim={0.5em 0em 3em 1em}, clip, width=0.22\textwidth]{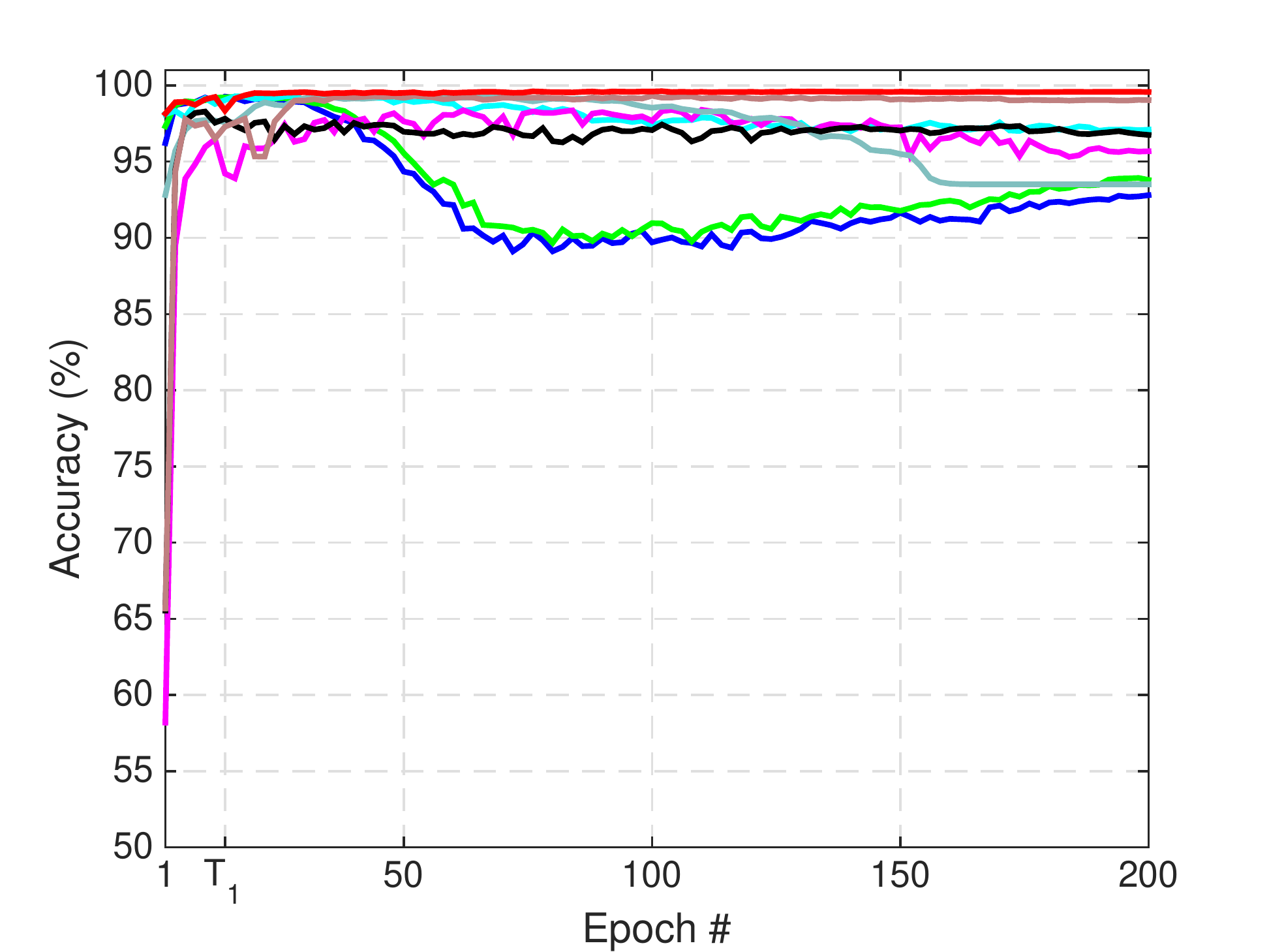}}
	\subfigure[ResNet18/Symmetry \(\epsilon\)=0.5]{\includegraphics[trim={0.5em 0em 3em 1em}, clip, width=0.22\textwidth]{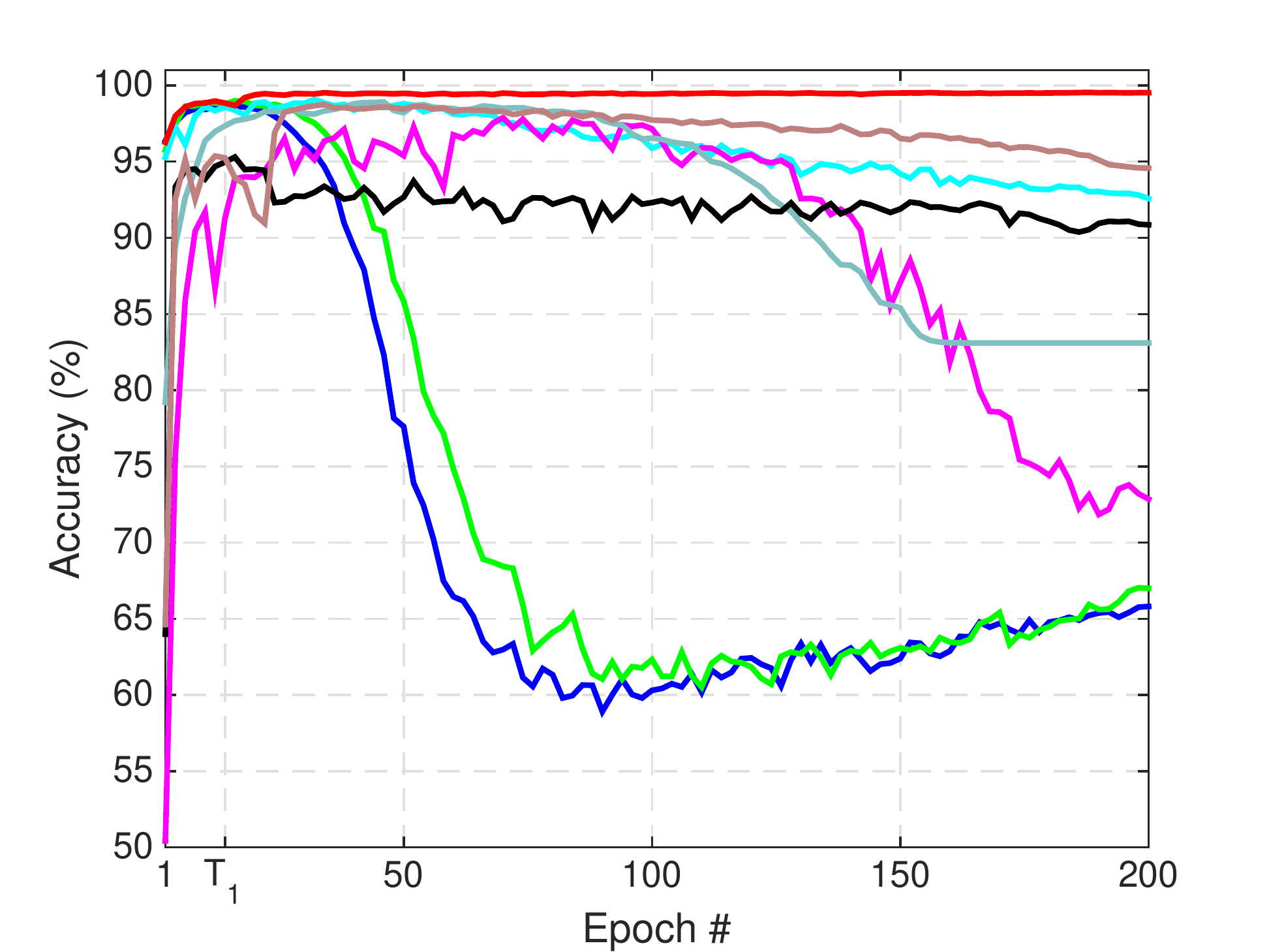}}
	\subfigure[ResNet18/Symmetry \(\epsilon\)=0.8]{\includegraphics[trim={0.5em 0em 3em 1em}, clip, width=0.22\textwidth]{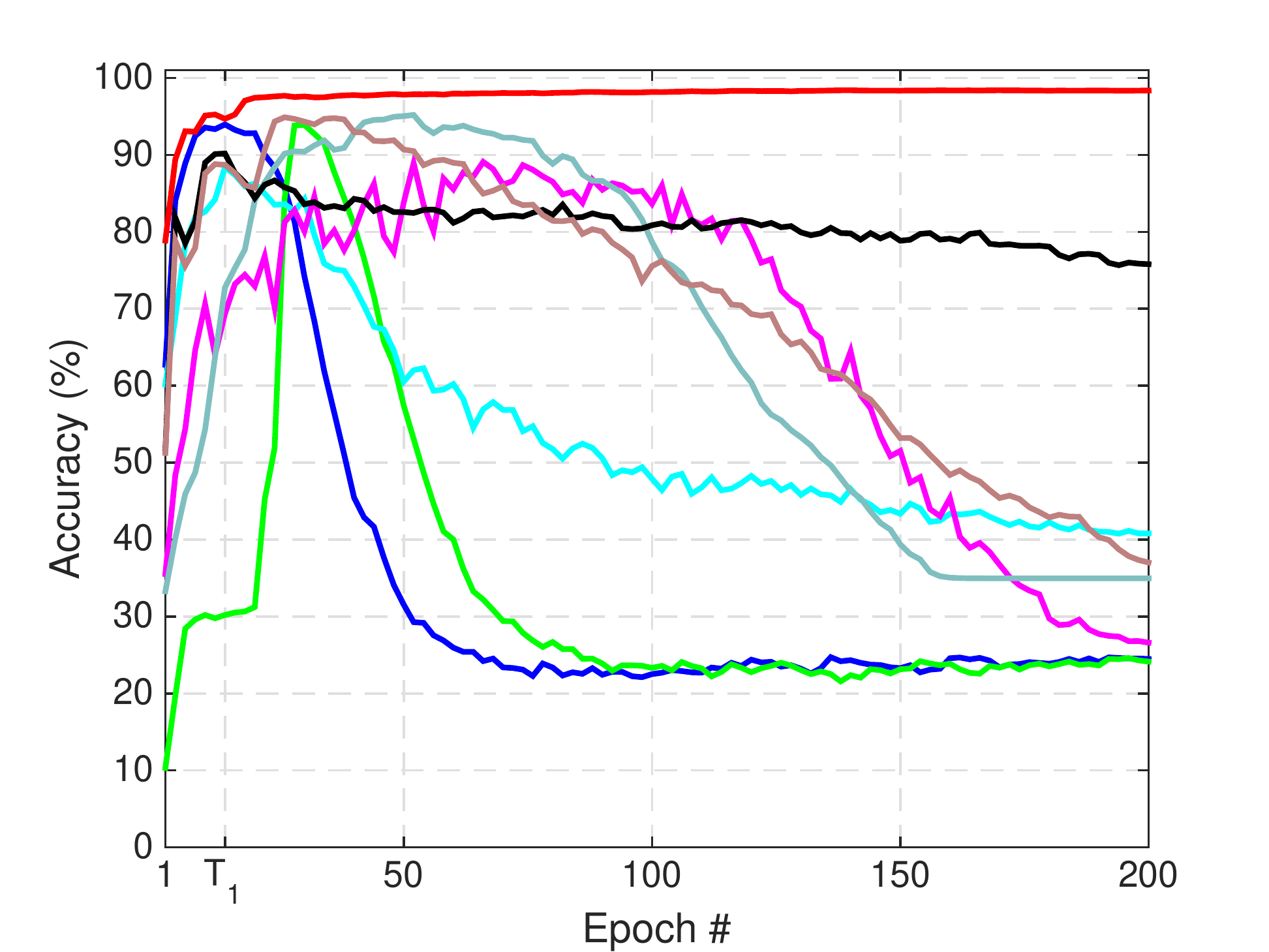}}
	\subfigure[ResNet18/Pair \(\epsilon\)=0.45]{\includegraphics[trim={0.5em 0em 3em 1em}, clip, width=0.22\textwidth]{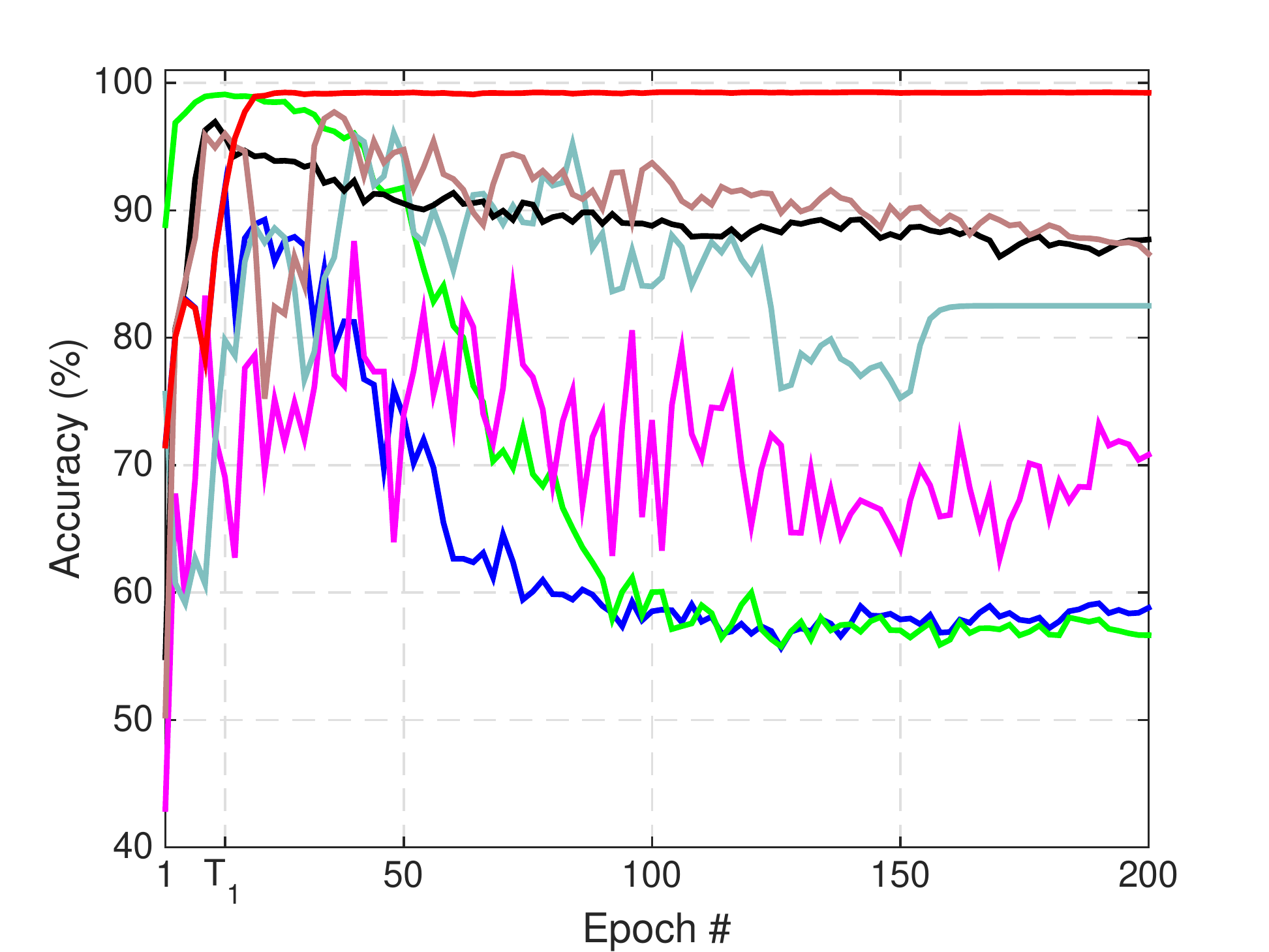}}
	\includegraphics[trim={39em 5em 36em 23em}, clip, width=0.1\textwidth]{figures/legend.pdf}
	\subfigure[ConvNet/Symmetry \(\epsilon\)=0.2]{\includegraphics[trim={0.5em 0em 3em 1em}, clip, width=0.22\textwidth]{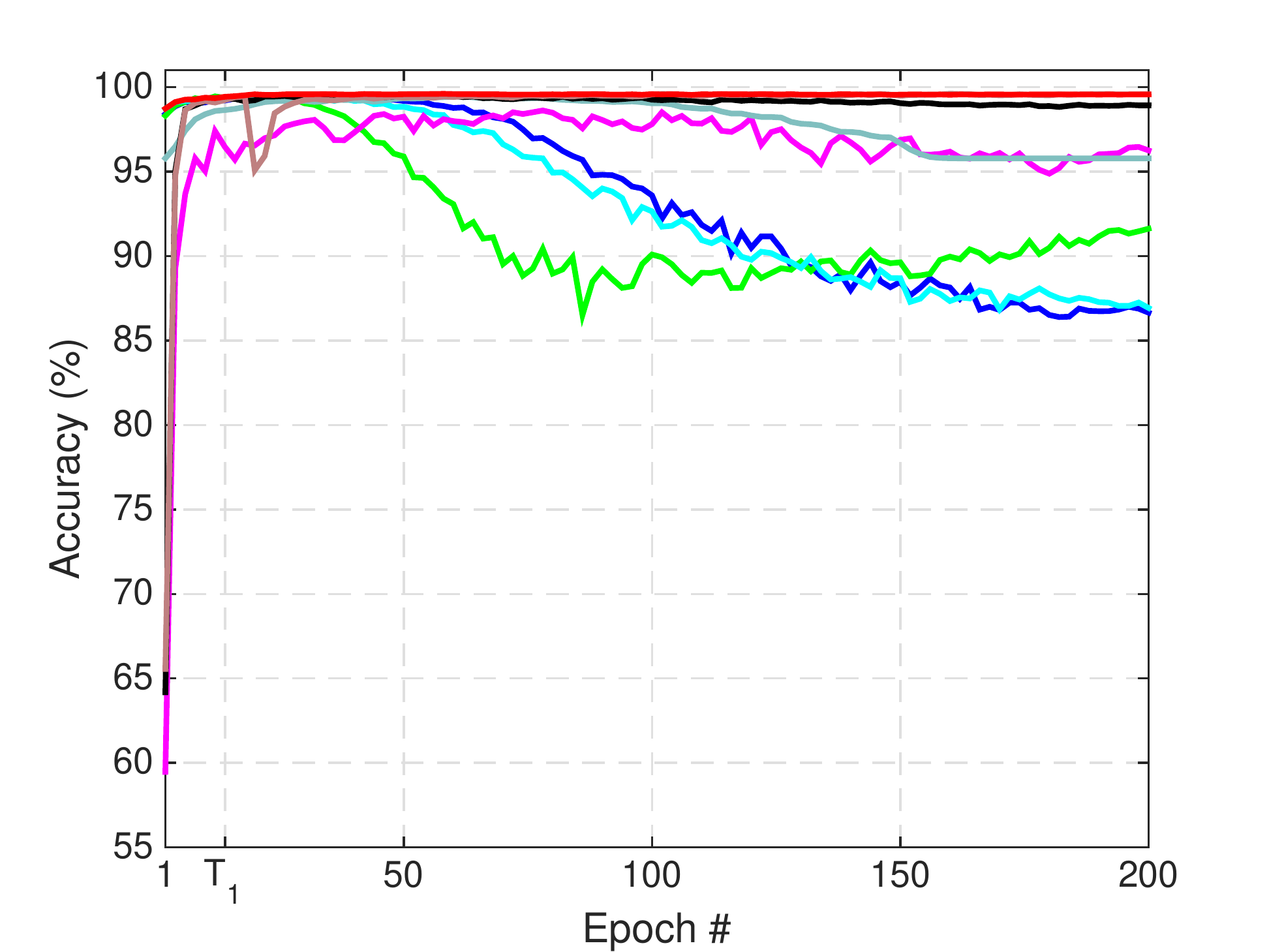}}
	\subfigure[ConvNet/Symmetry \(\epsilon\)=0.5]{\includegraphics[trim={0.5em 0em 3em 1em}, clip, width=0.22\textwidth]{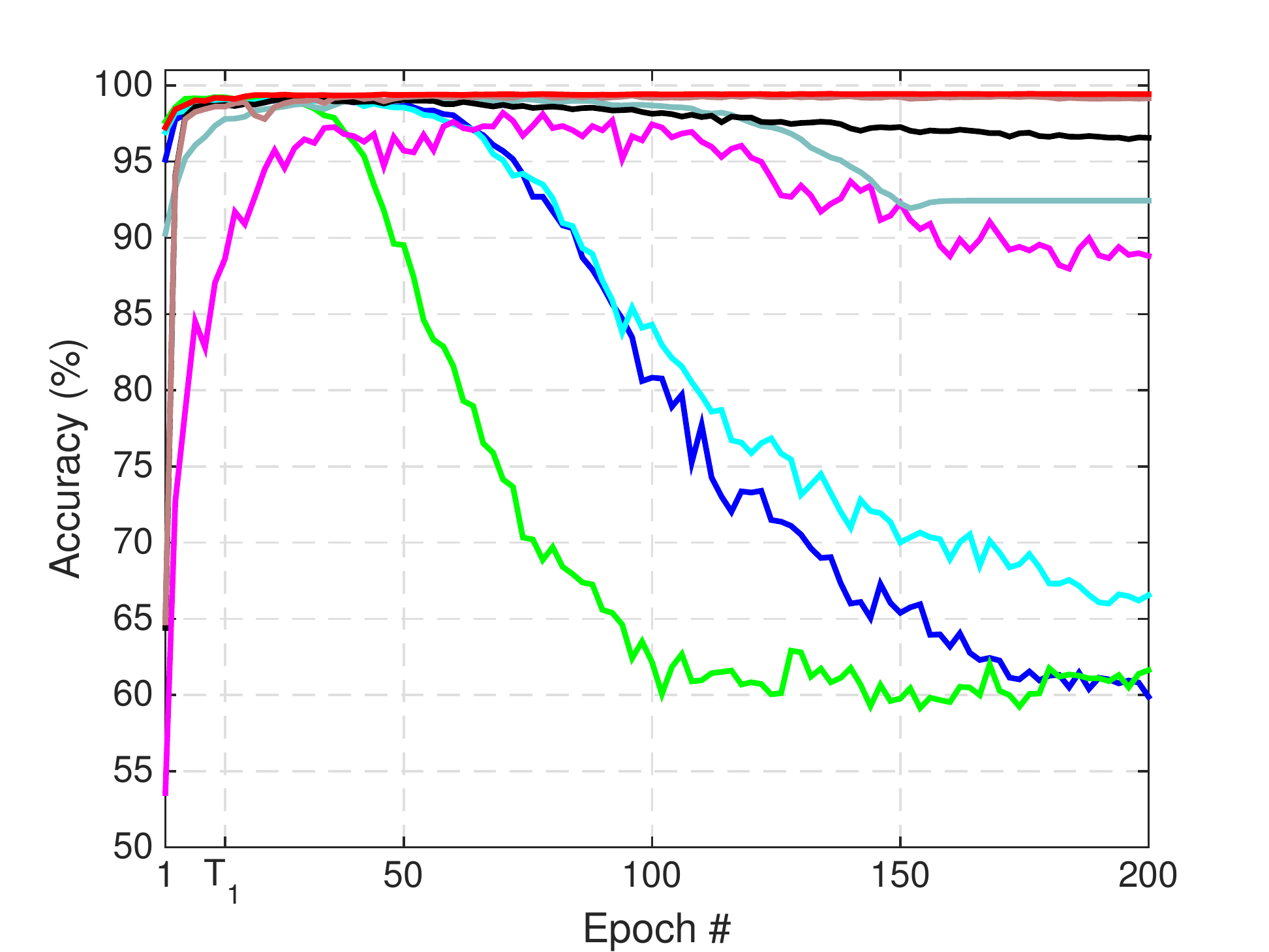}}
	\subfigure[ConvNet/Symmetry \(\epsilon\)=0.8]{\includegraphics[trim={0.5em 0em 3em 1em}, clip, width=0.22\textwidth]{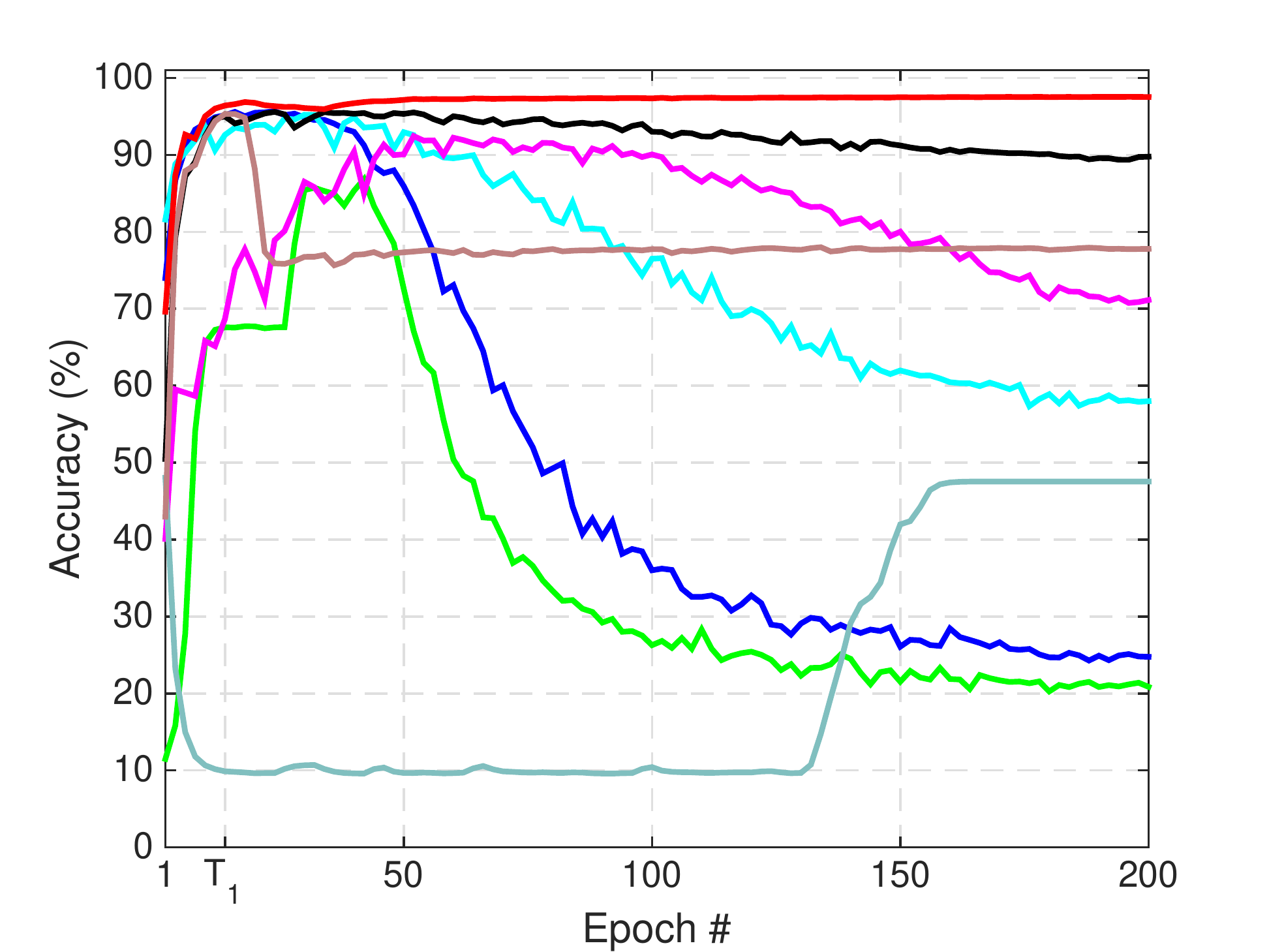}}
	\subfigure[ConvNet/Pair \(\epsilon\)=0.45]{\includegraphics[trim={0.5em 0em 3em 1em}, clip, width=0.22\textwidth]{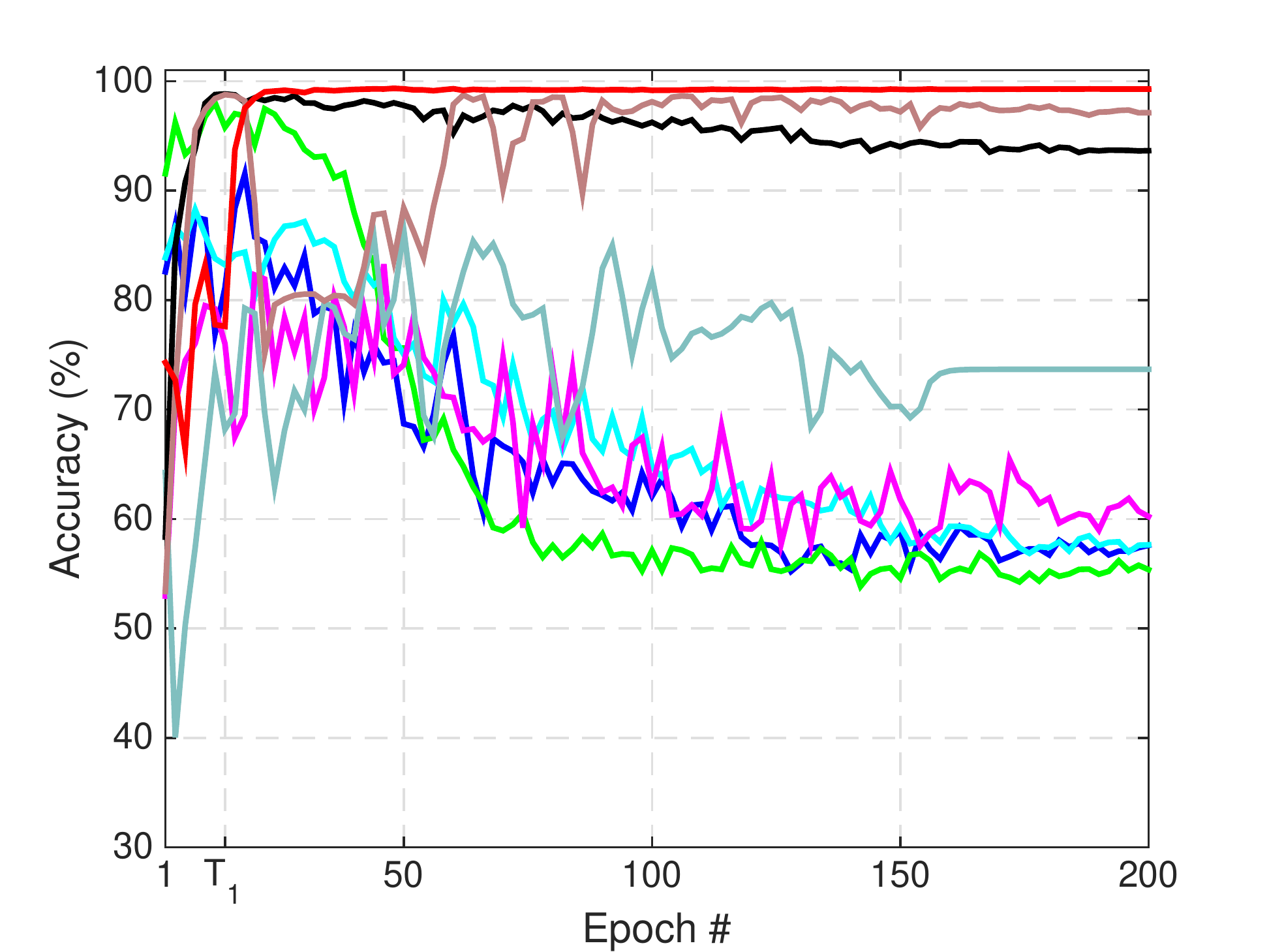}}
	\vspace{-1.3em}
	\caption{Testing accuracy of seven methods at different numbers of epochs on MNIST.} 
	\vspace{-1em}
	\label{fig:mnist}
\end{figure*}

\begin{figure*}[!tbp] \renewcommand\thefigure{A2}
	\includegraphics[trim={39em 5em 36em 23em}, clip, width=0.1\textwidth]{figures/legend.pdf}
	\subfigure[ResNet18/Symmetry \(\epsilon\)=0.2]{\includegraphics[trim={0.5em 0em 3em 1em}, clip, width=0.22\textwidth]{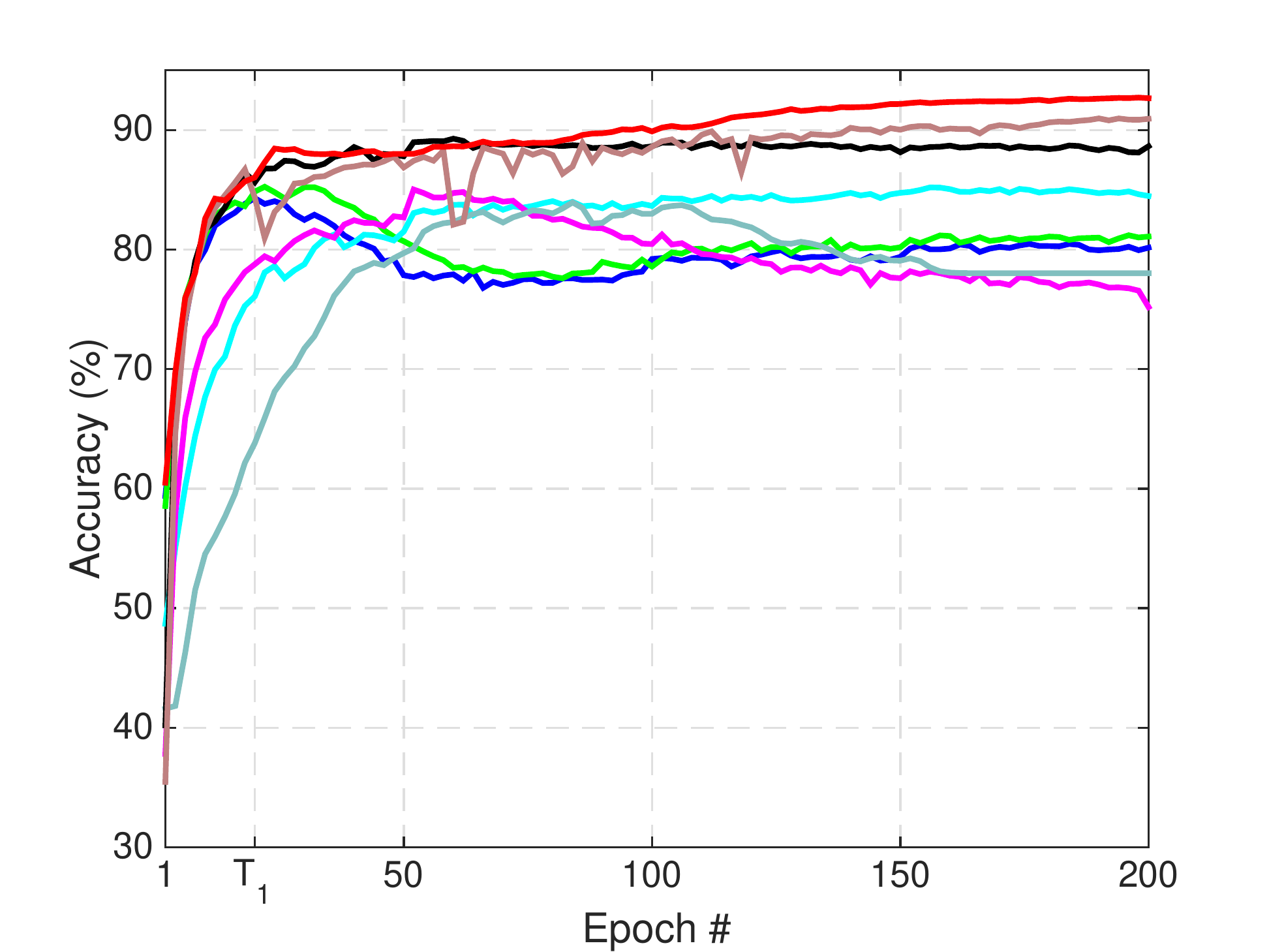}}
	\subfigure[ResNet18/Symmetry \(\epsilon\)=0.5]{\includegraphics[trim={0.5em 0em 3em 1em}, clip, width=0.22\textwidth]{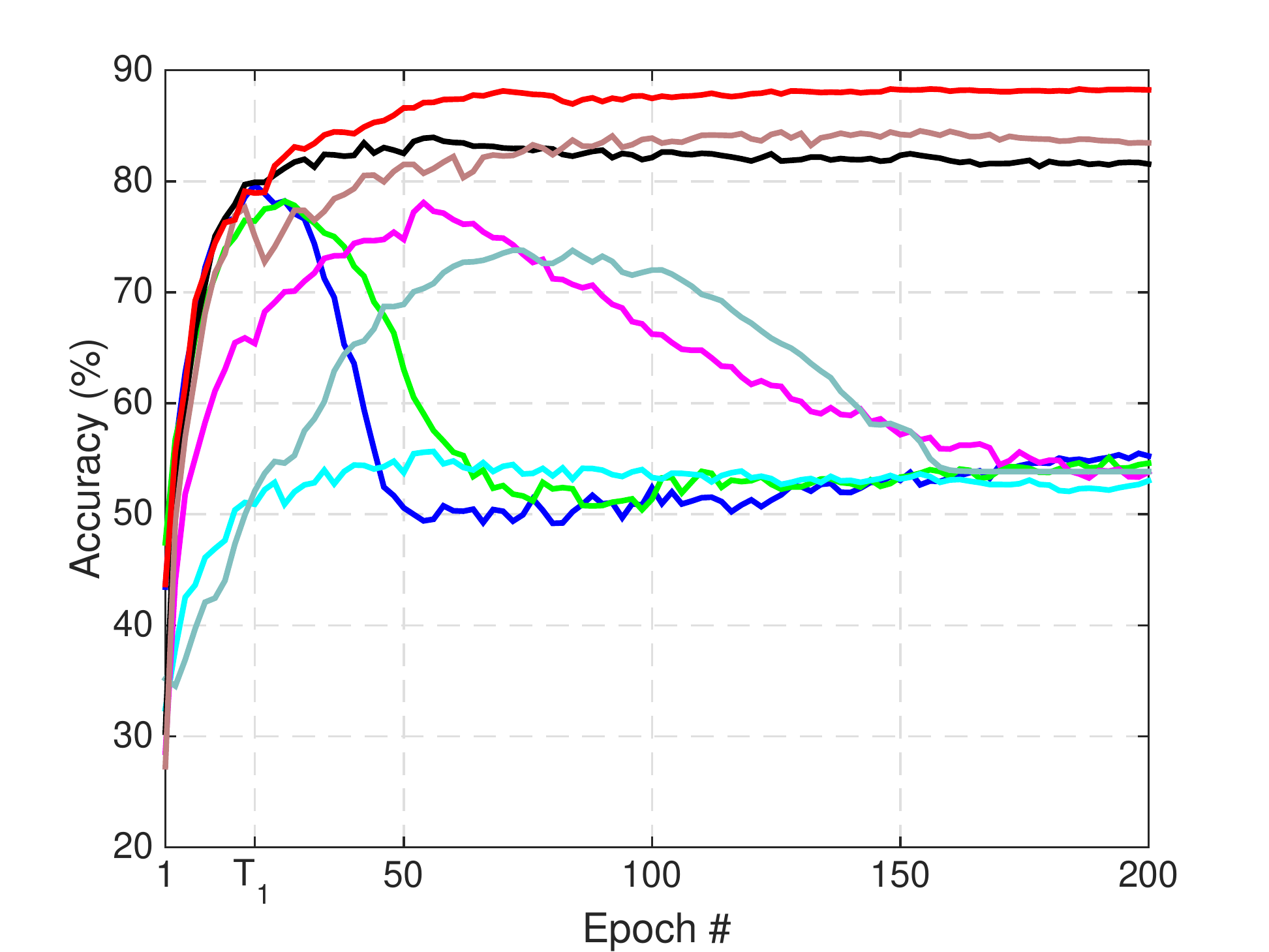}}
	\subfigure[ResNet18/Symmetry \(\epsilon\)=0.8]{\includegraphics[trim={0.5em 0em 3em 1em}, clip, width=0.22\textwidth]{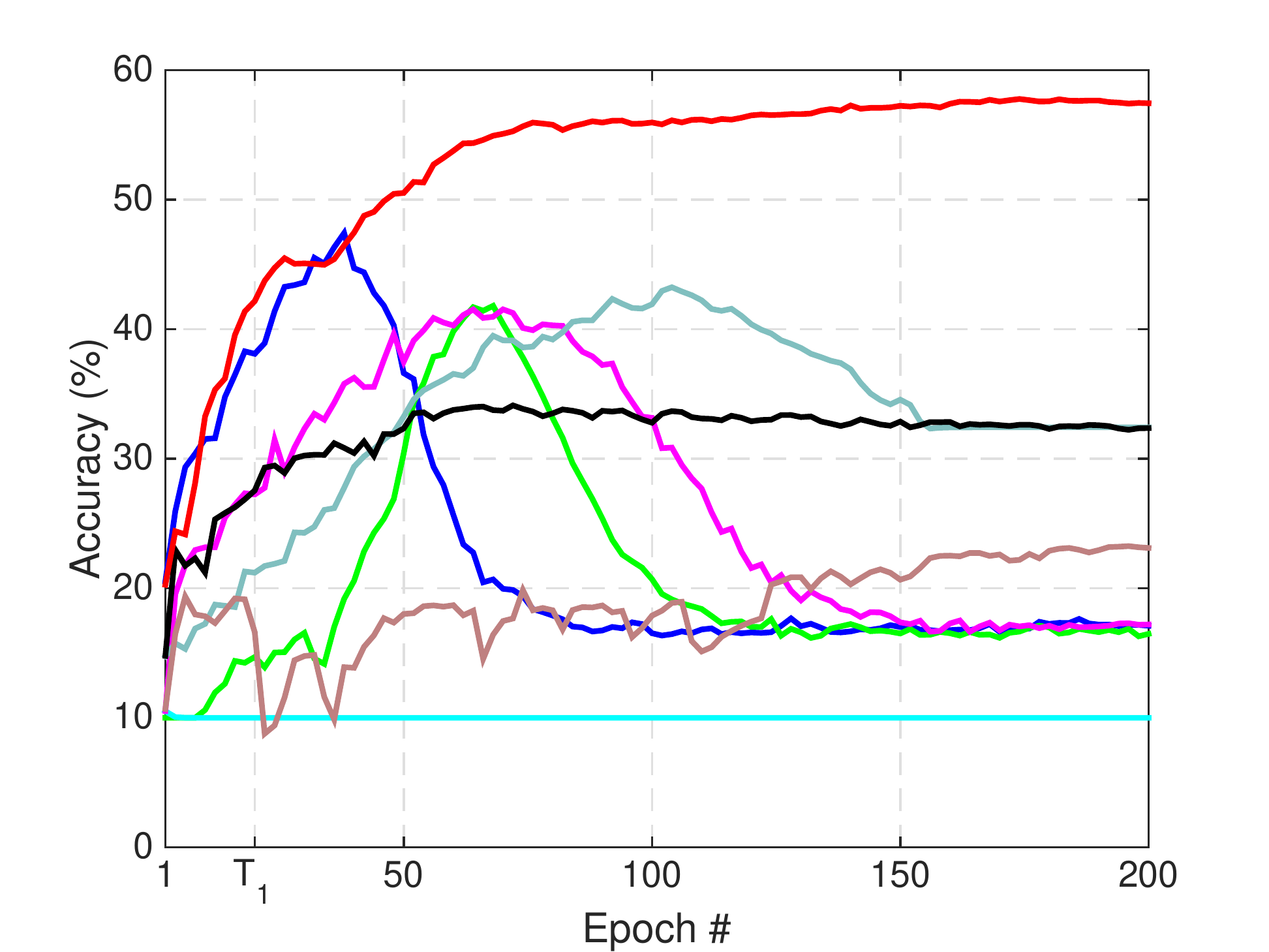}}
	\subfigure[ResNet18/Pair \(\epsilon\)=0.45]{\includegraphics[trim={0.5em 0em 3em 1em}, clip, width=0.22\textwidth]{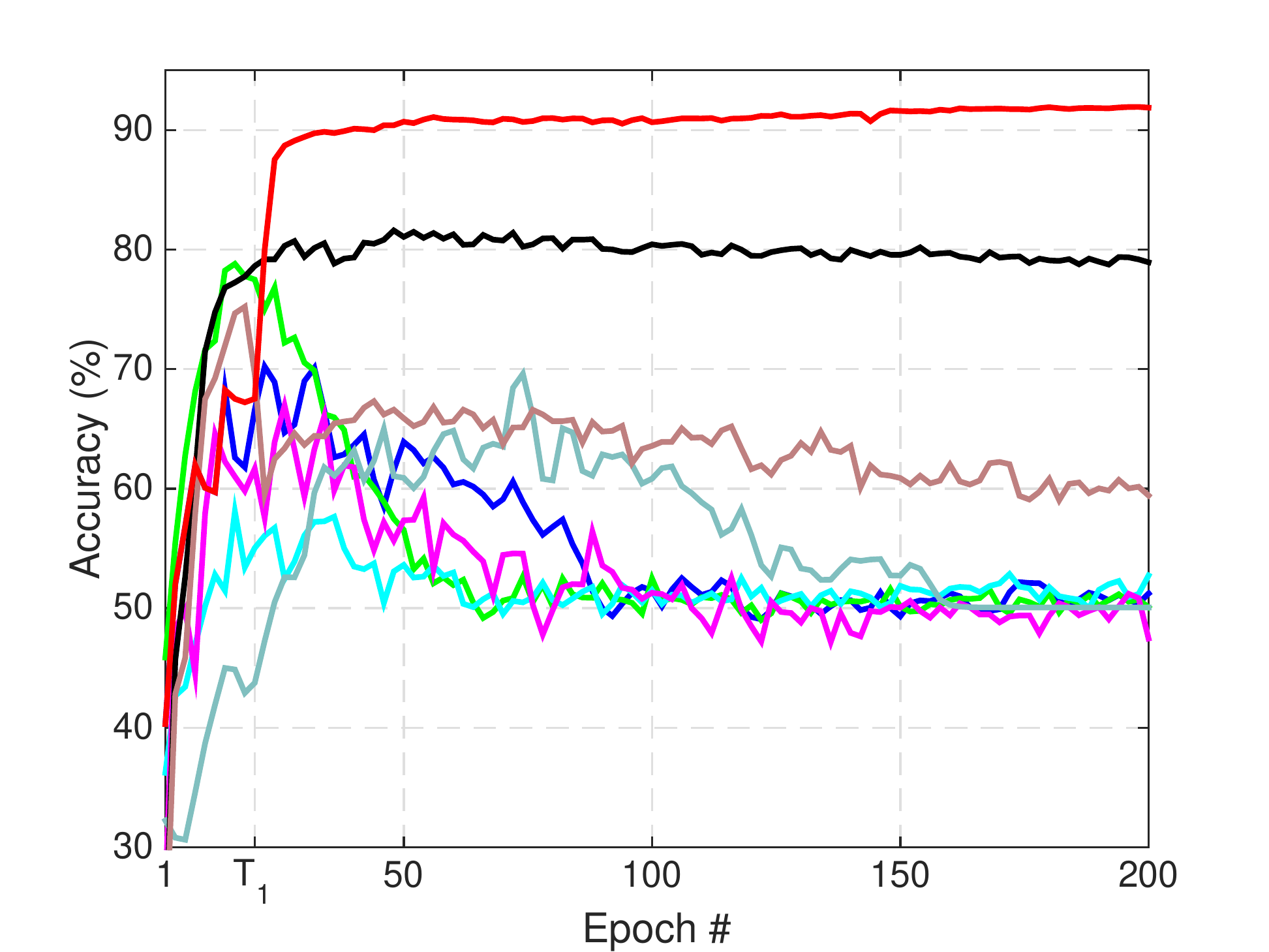}}
	\includegraphics[trim={39em 5em 36em 23em}, clip, width=0.1\textwidth]{figures/legend.pdf}
	\subfigure[ConvNet/Symmetry \(\epsilon\)=0.2]{\includegraphics[trim={0.5em 0em 3em 1em}, clip, width=0.22\textwidth]{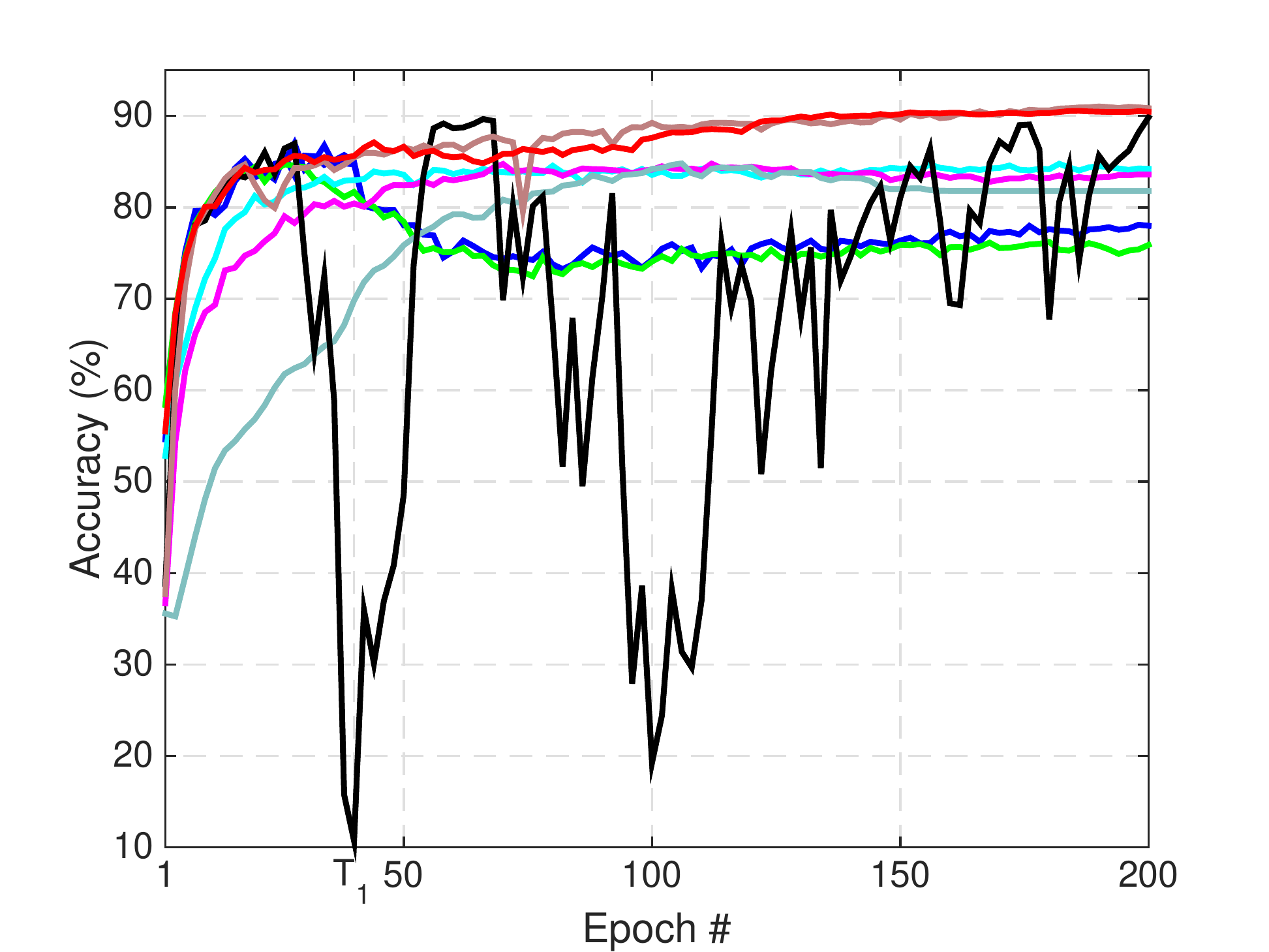}}
	\subfigure[ConvNet/Symmetry \(\epsilon\)=0.5]{\includegraphics[trim={0.5em 0em 3em 1em}, clip, width=0.22\textwidth]{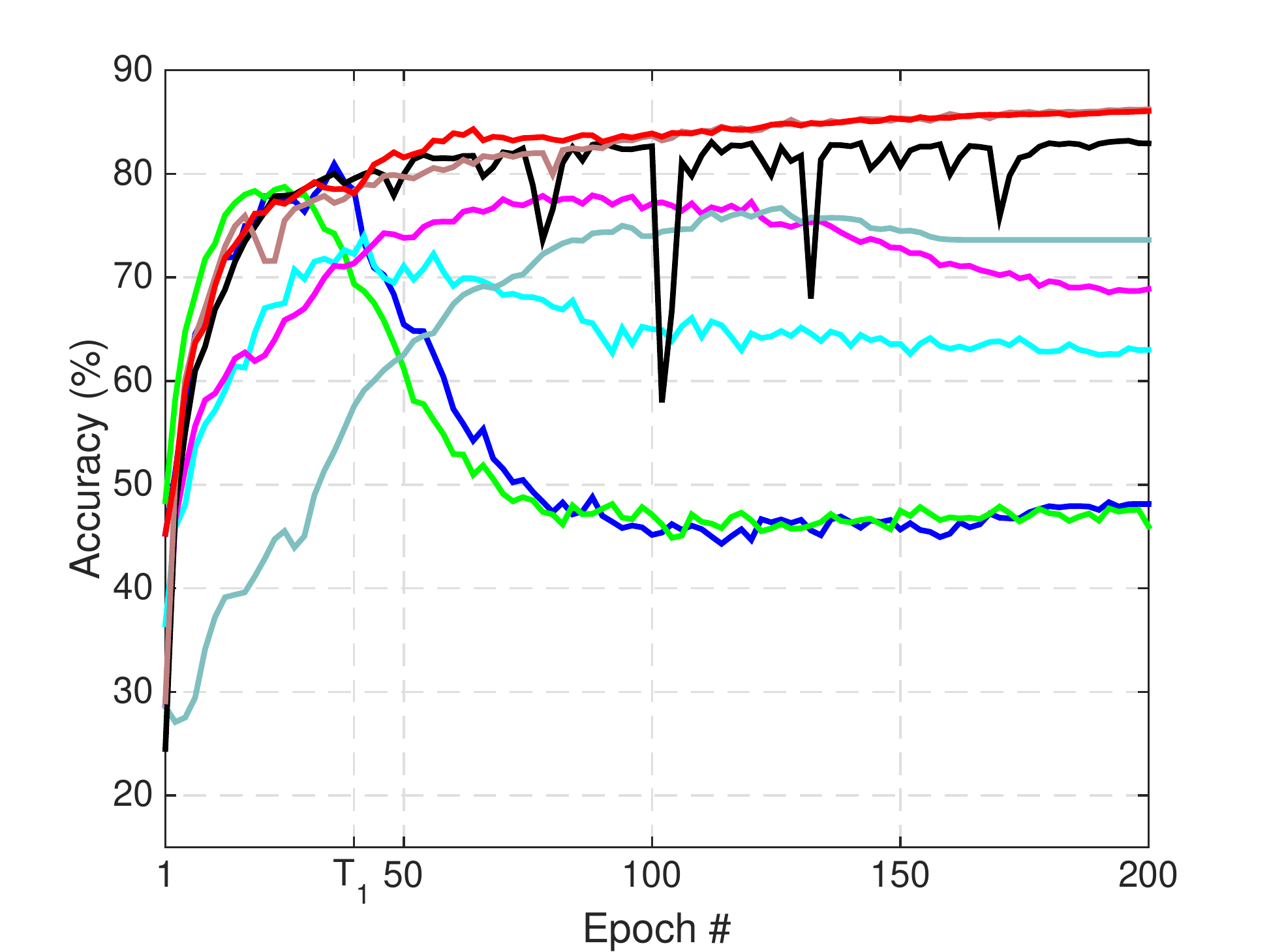}}
	\subfigure[ConvNet/Symmetry \(\epsilon\)=0.8]{\includegraphics[trim={0.5em 0em 3em 1em}, clip, width=0.22\textwidth]{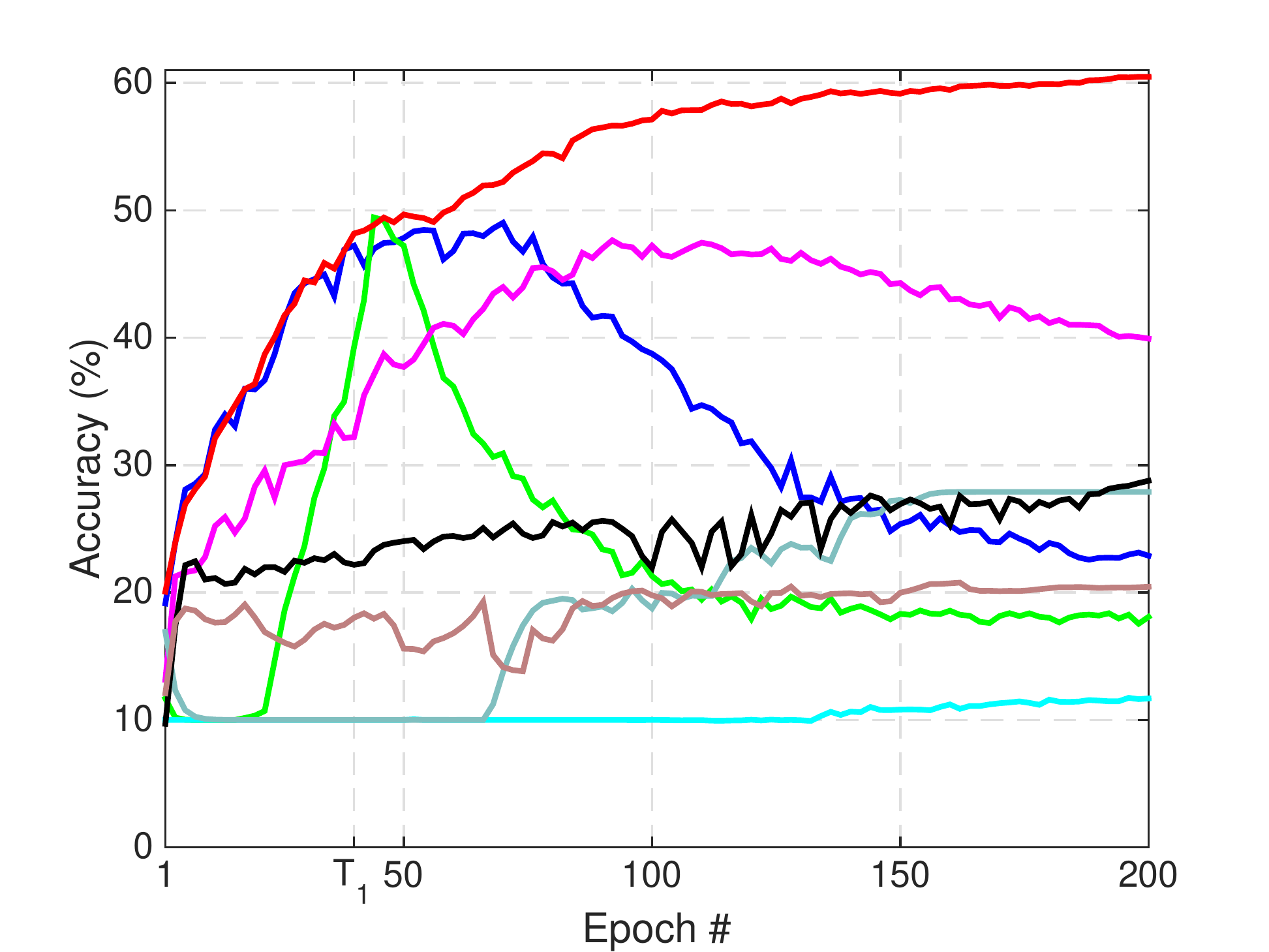}}
	\subfigure[ConvNet/Pair \(\epsilon\)=0.45]{\includegraphics[trim={0.5em 0em 3em 1em}, clip, width=0.22\textwidth]{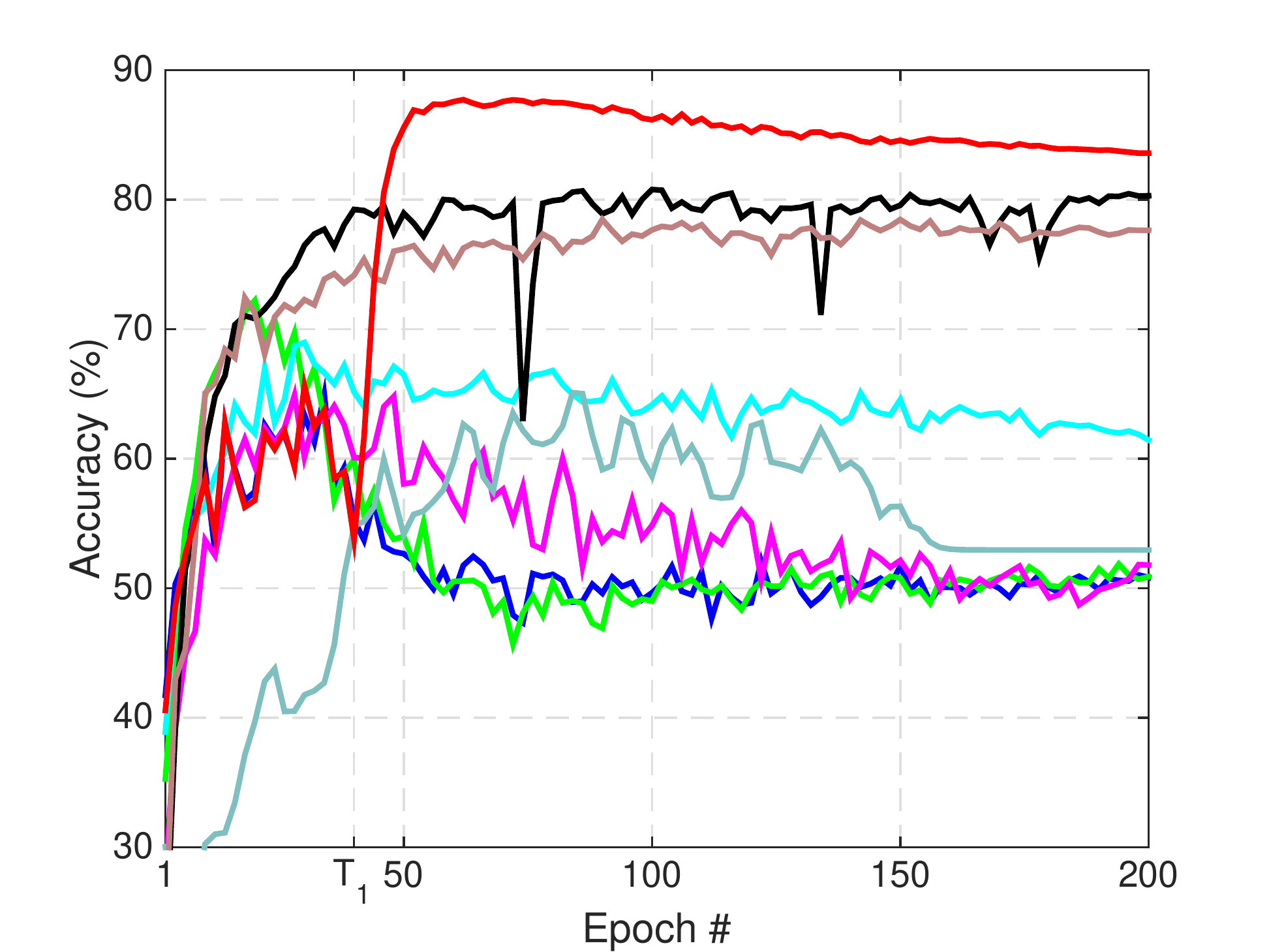}}
	\vspace{-1.3em}
	\caption{Testing accuracy of seven methods at different numbers of epochs on CIFAR-10.} 
	\vspace{-1em}
	\label{fig:cifar10}
\end{figure*}

\begin{figure*}[htb]\renewcommand\thefigure{A3}
	\includegraphics[trim={39em 5em 36em 23em}, clip, width=0.1\textwidth]{figures/legend.pdf}
	\subfigure[ResNet18/Symmetry \(\epsilon\)=0.2]{\includegraphics[trim={0.5em 0em 3em 1em}, clip, width=0.22\textwidth]{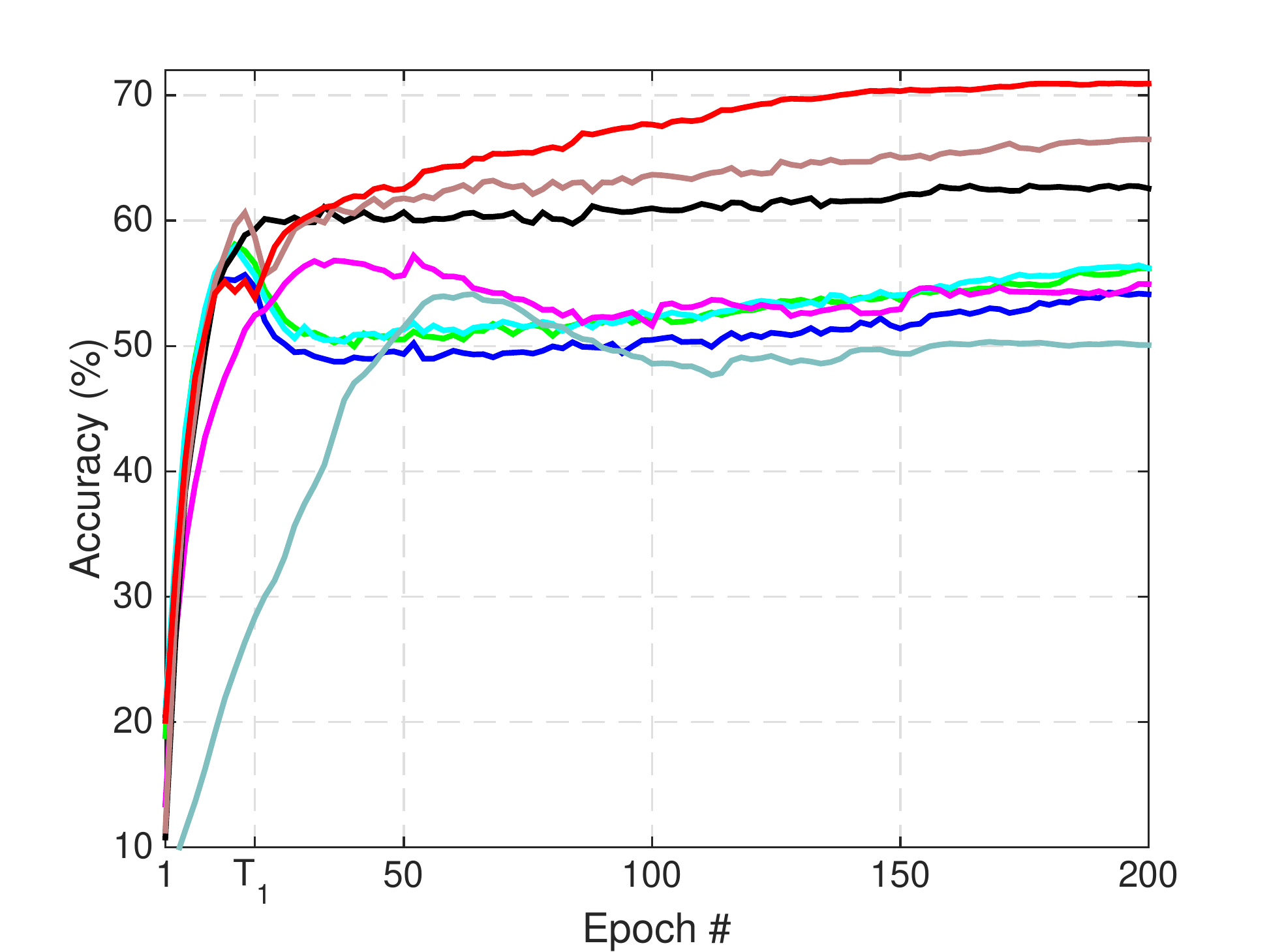}}
	\subfigure[ResNet18/Symmetry \(\epsilon\)=0.5]{\includegraphics[trim={0.5em 0em 3em 1em}, clip, width=0.22\textwidth]{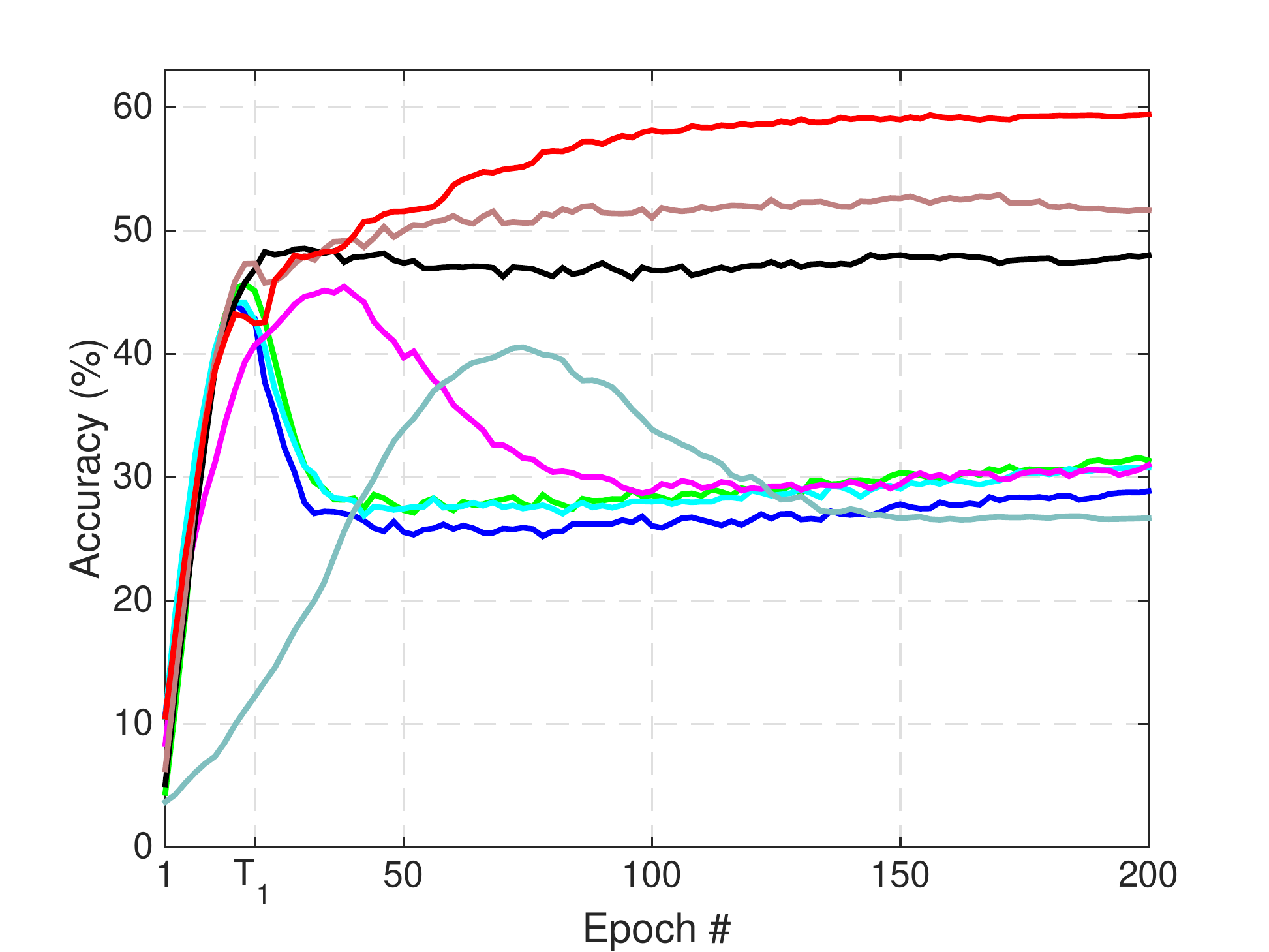}}
	\subfigure[ResNet18/Symmetry \(\epsilon\)=0.8]{\includegraphics[trim={0.5em 0em 3em 1em}, clip, width=0.22\textwidth]{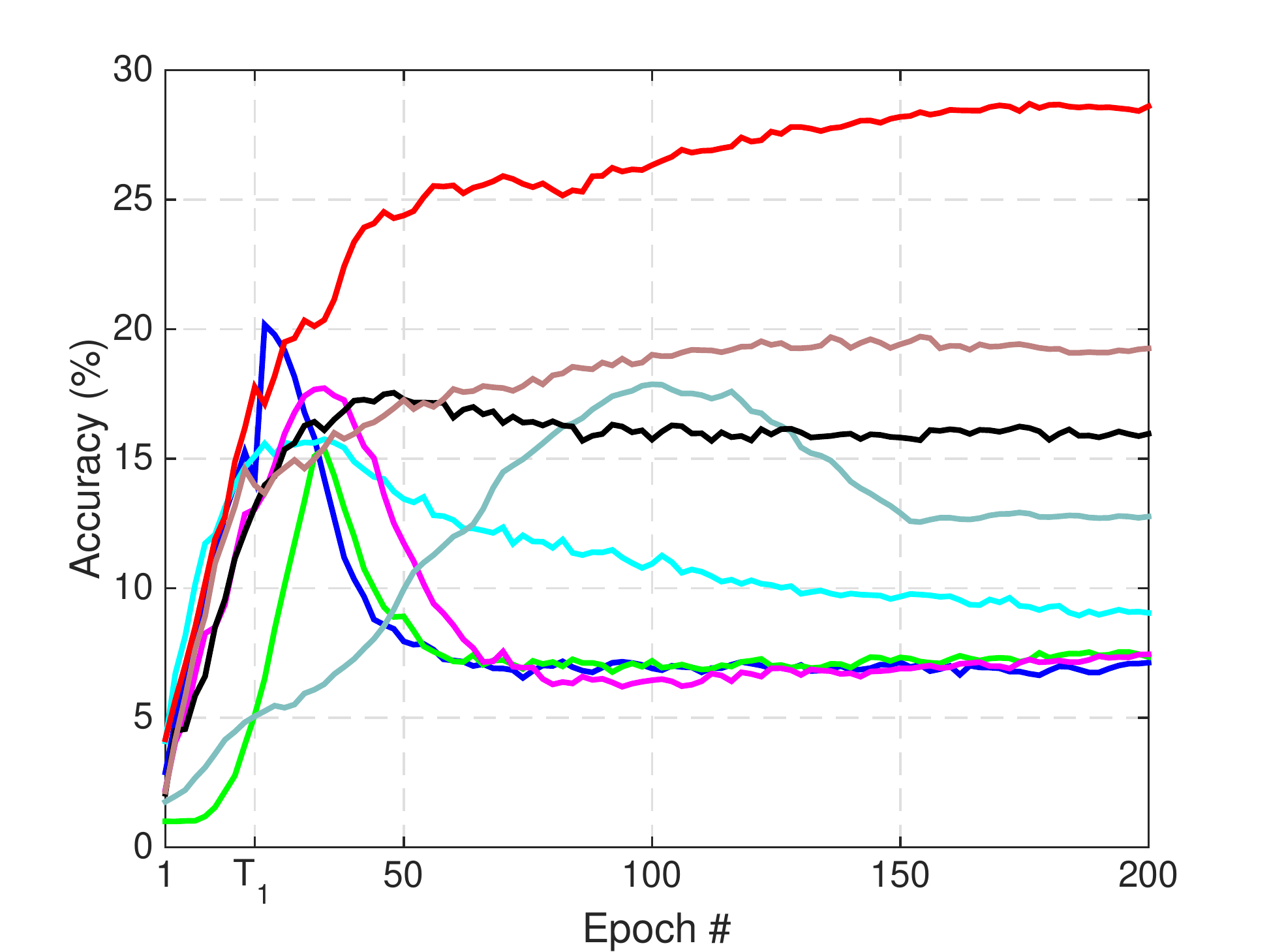}}
	\subfigure[ResNet18/Pair \(\epsilon\)=0.45]{\includegraphics[trim={0.5em 0em 3em 1em}, clip, width=0.22\textwidth]{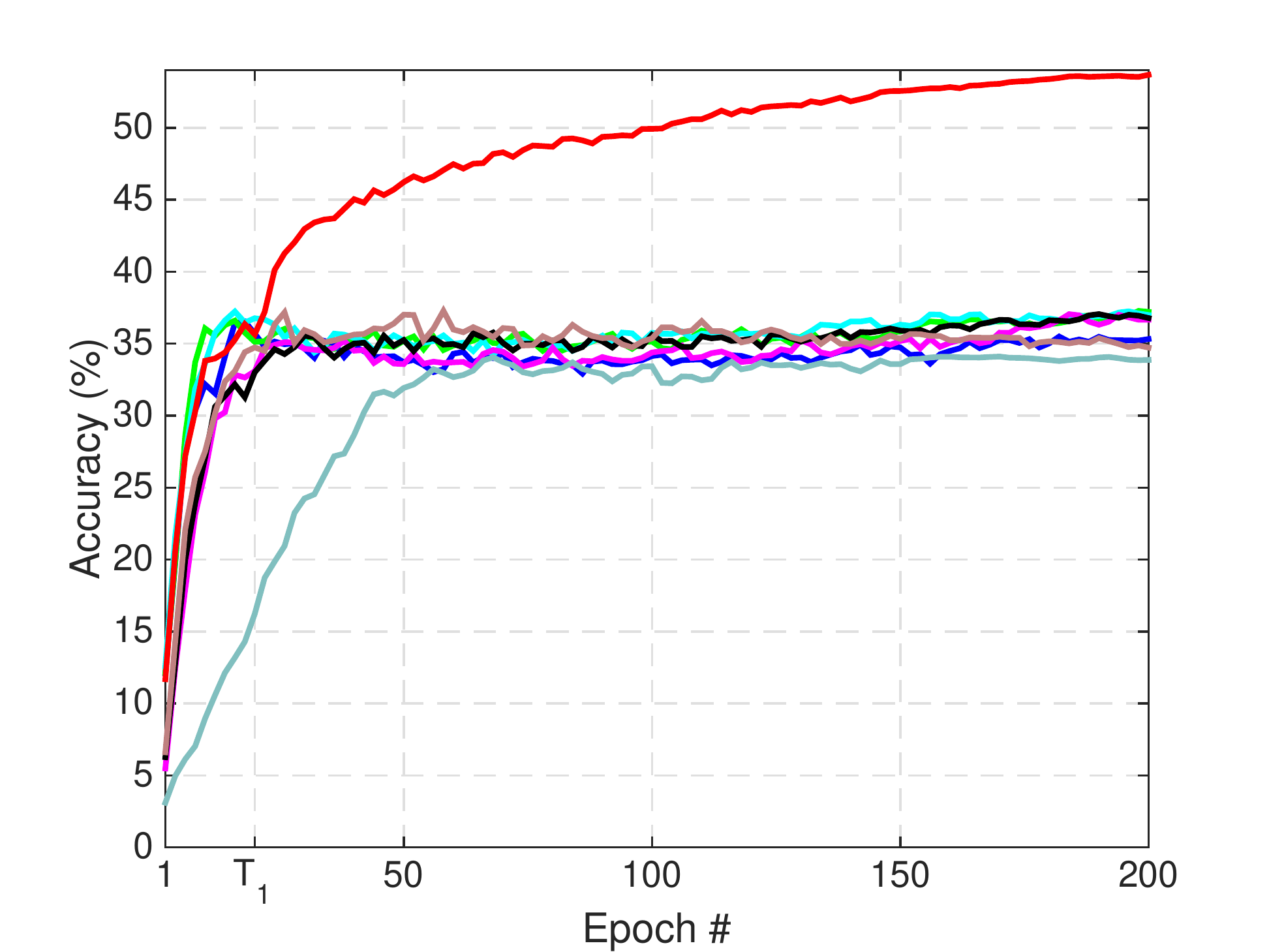}}
	\includegraphics[trim={39em 5em 36em 23em}, clip, width=0.1\textwidth]{figures/legend.pdf}
	\subfigure[ ConvNet/Symmetry \(\epsilon\)=0.2]{\includegraphics[trim={0.5em 0em 3em 1em}, clip, width=0.22\textwidth]{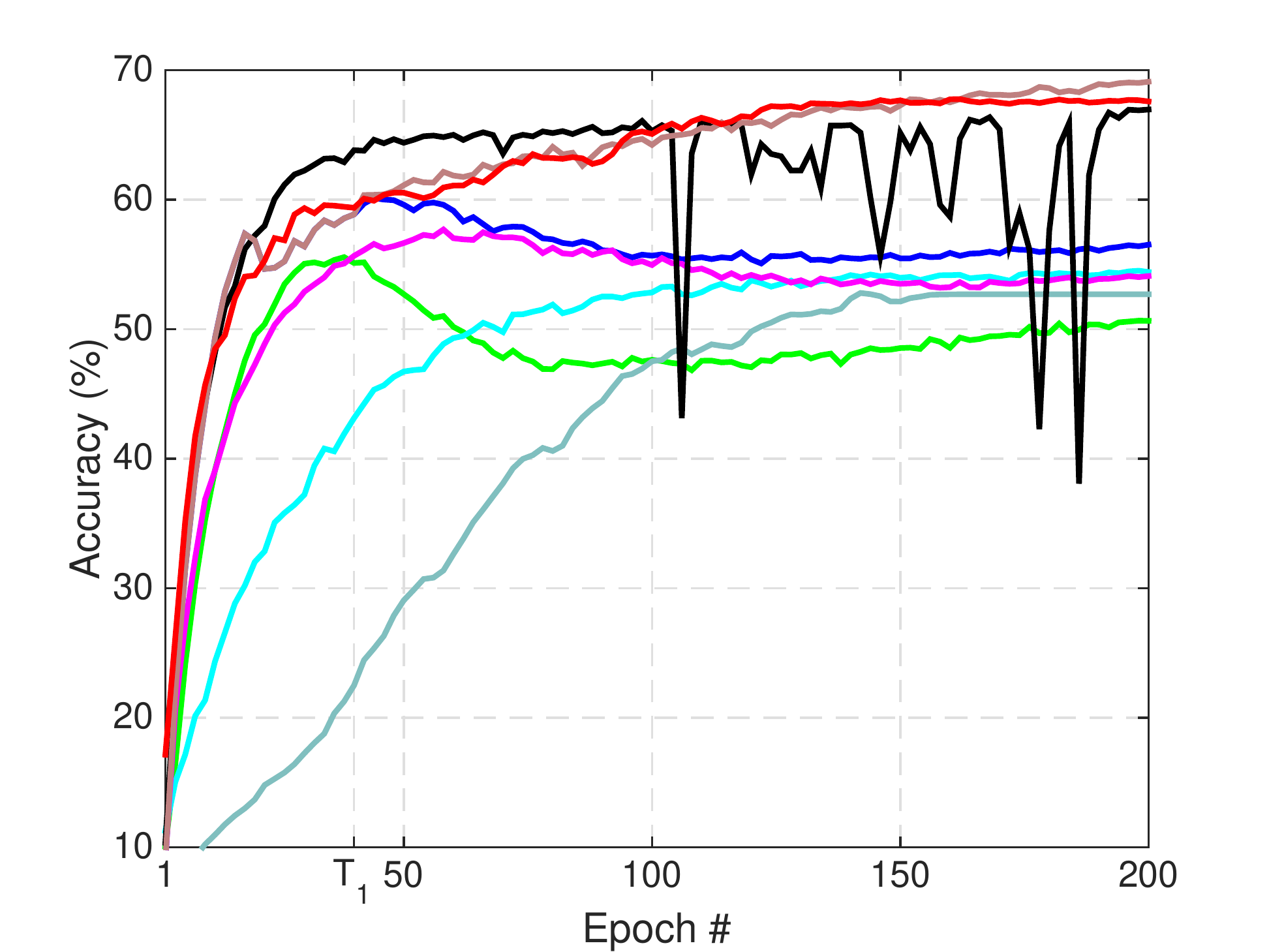}}
	\subfigure[ ConvNet/Symmetry \(\epsilon\)=0.5]{\includegraphics[trim={0.5em 0em 3em 1em}, clip, width=0.22\textwidth]{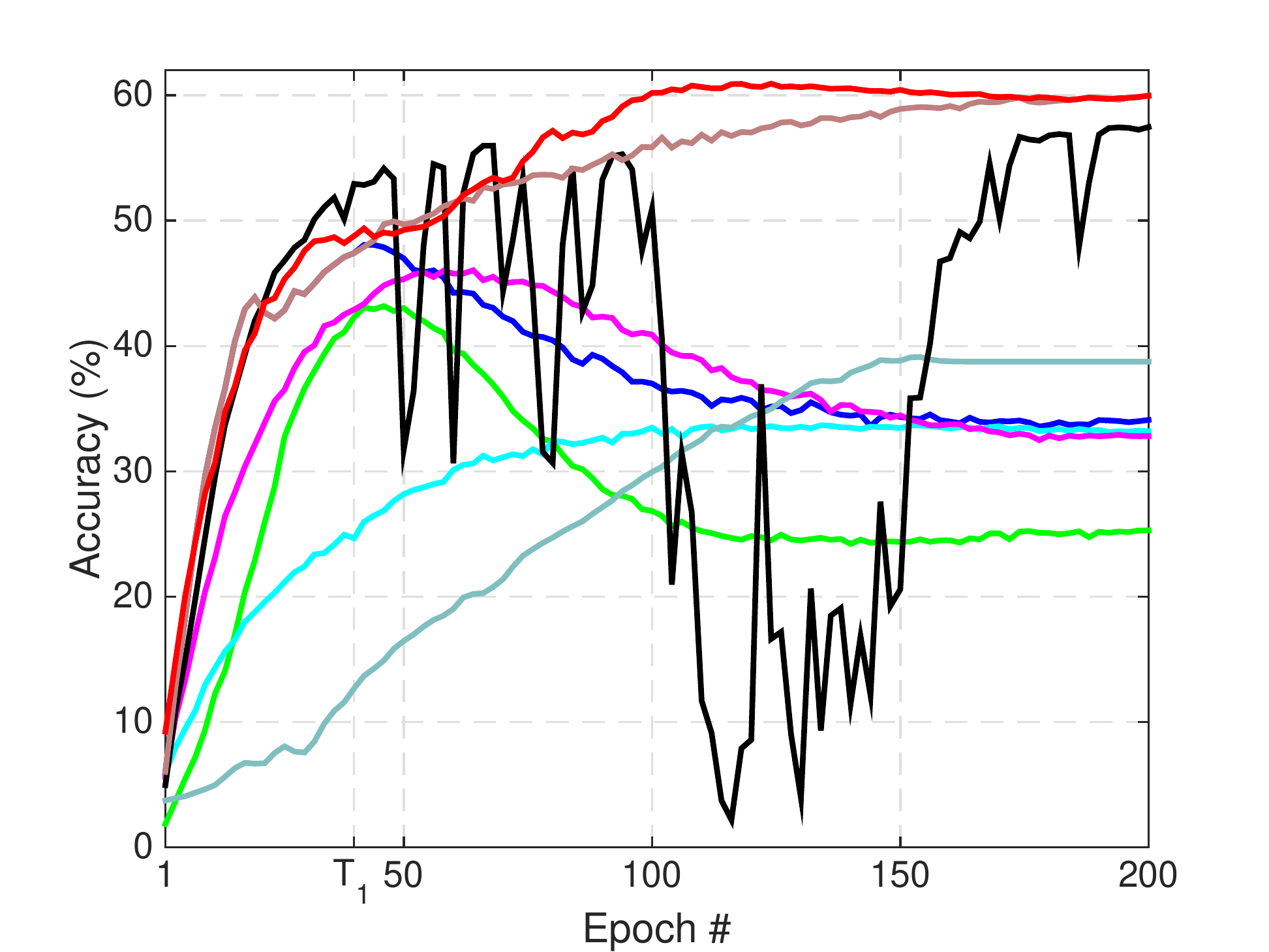}}
	\subfigure[ ConvNet/Symmetry \(\epsilon\)=0.8]{\includegraphics[trim={0.5em 0em 3em 1em}, clip, width=0.22\textwidth]{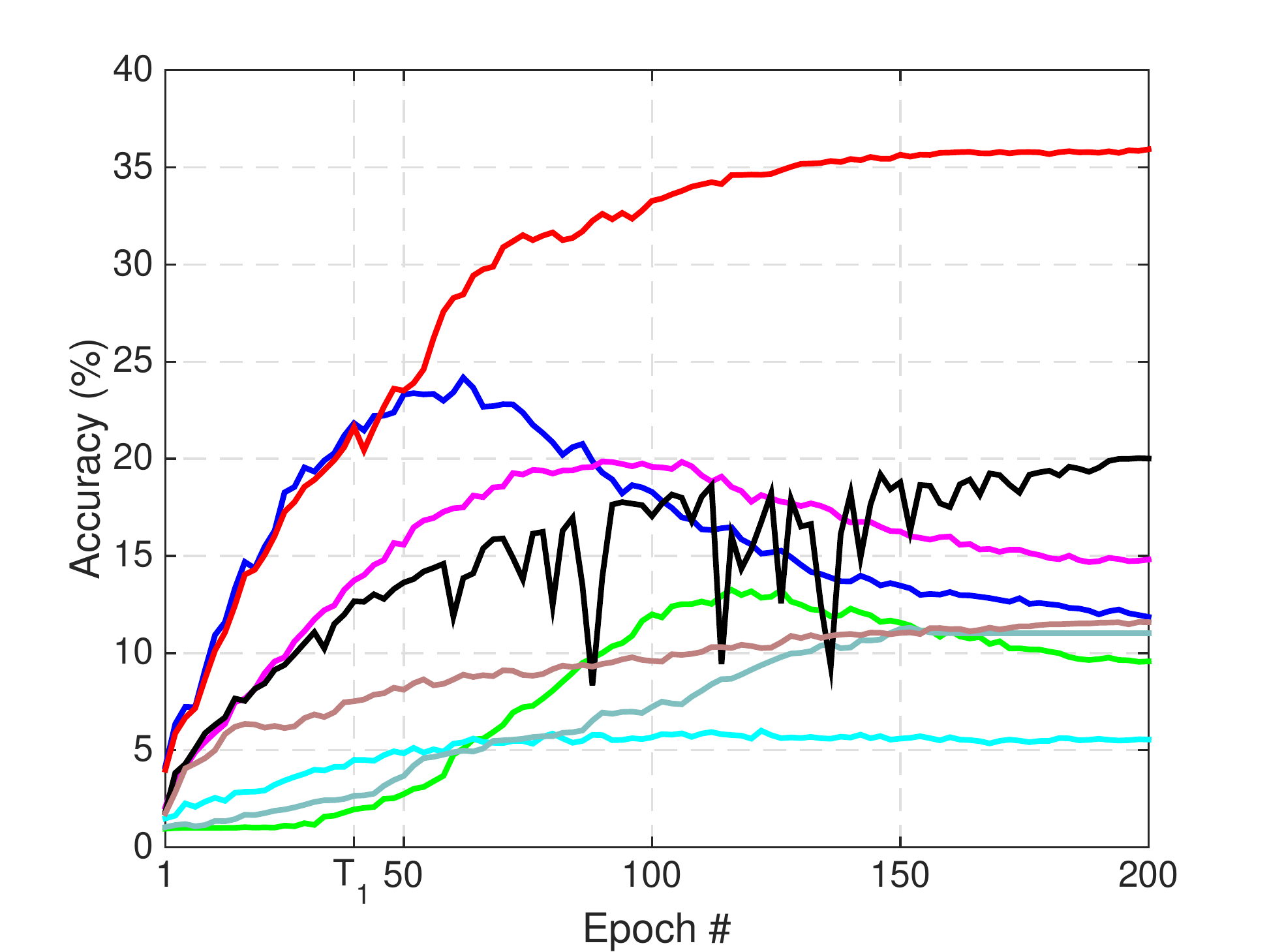}}
	\subfigure[ConvNet/Pair \(\epsilon\)=0.45]{\includegraphics[trim={0.5em 0em 3em 1em}, clip, width=0.22\textwidth]{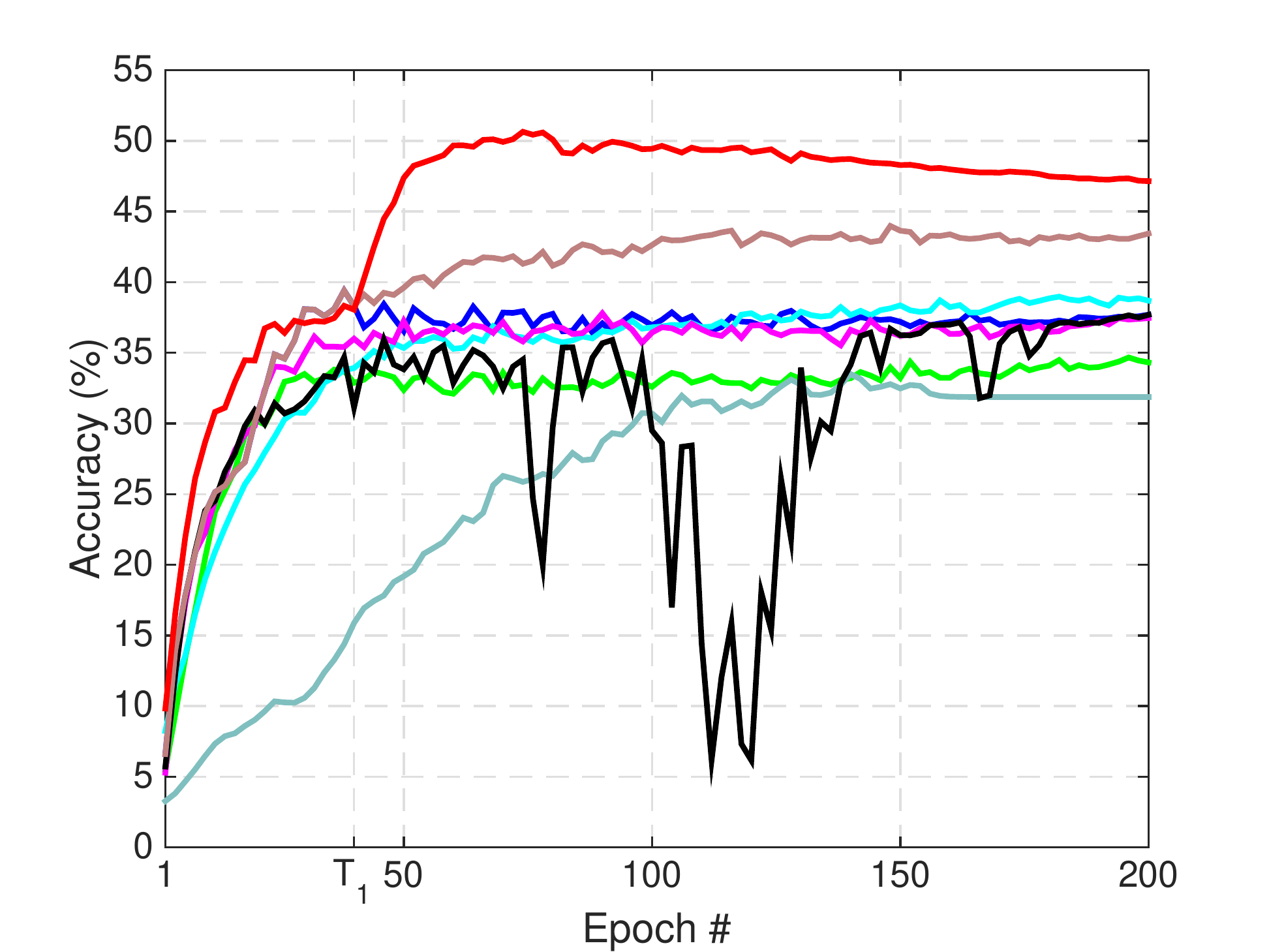}}
	\vspace{-1em}
	\caption{Testing accuracy of seven methods at different numbers of epochs on CIFAR-100.} 
	\label{fig:cifar100}
	\vspace{-1em}
\end{figure*}

\begin{figure*}[ht]\renewcommand\thefigure{A4}
	\includegraphics[trim={39em 5em 36em 23em}, clip, width=0.1\textwidth]{figures/legend.pdf}
	\subfigure[ResNet18/Symmetry \(\epsilon\)=0.2]{\includegraphics[trim={0.5em 0em 3em 1em}, clip, width=0.22\textwidth]{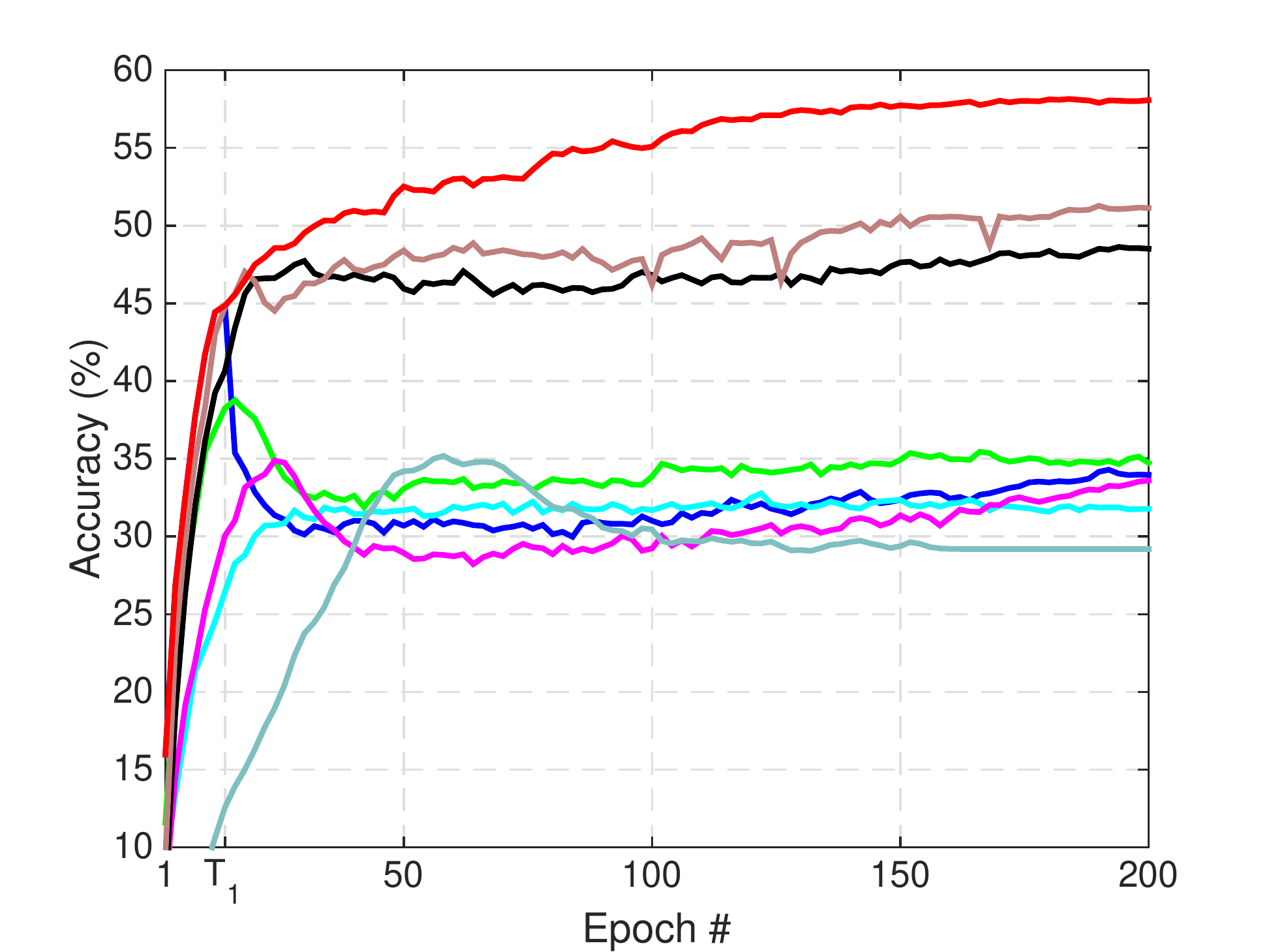}}
	\subfigure[ResNet18/Symmetry \(\epsilon\)=0.5]{\includegraphics[trim={0.5em 0em 3em 1em}, clip, width=0.22\textwidth]{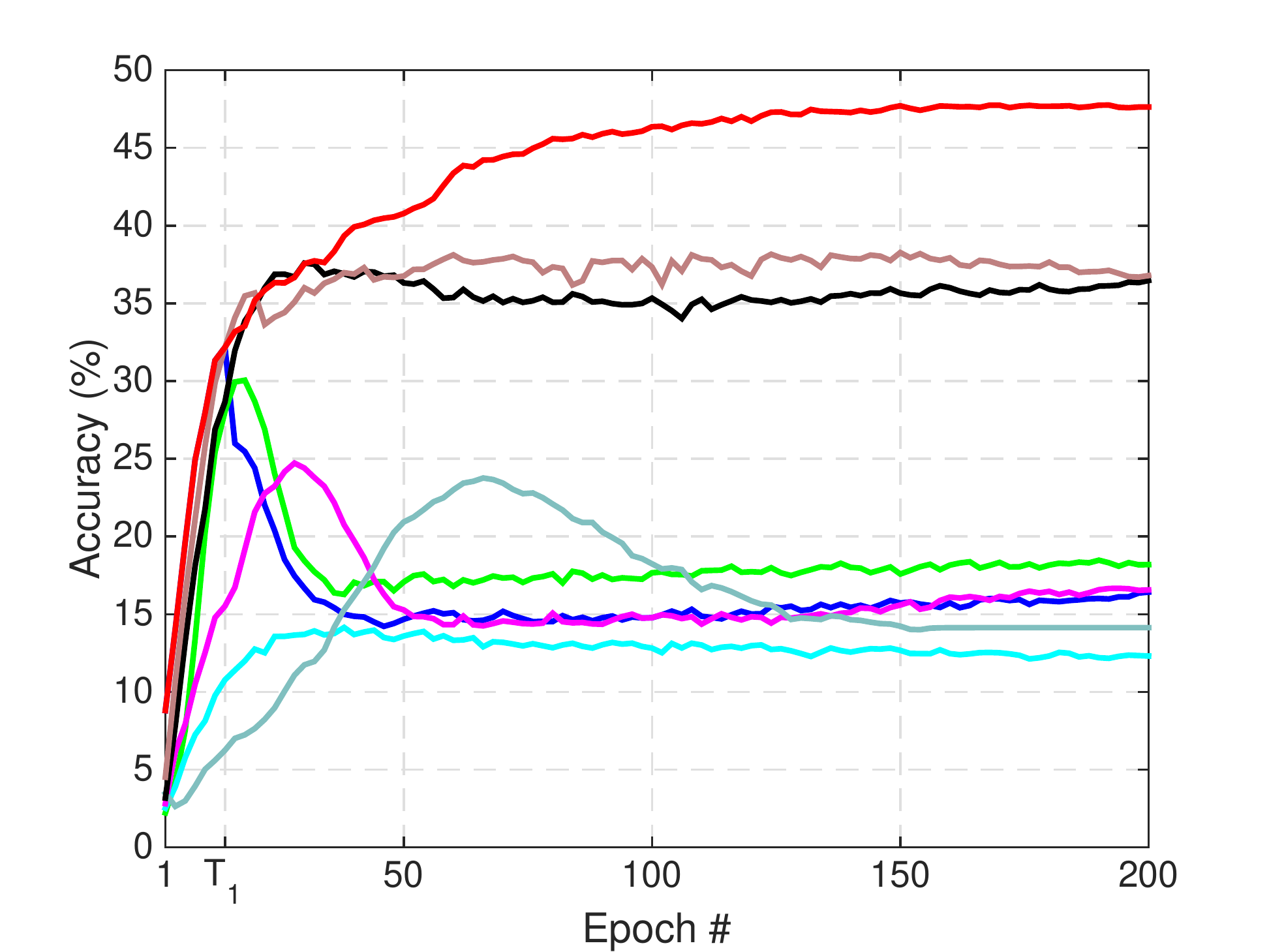}}
	\subfigure[ResNet18/Symmetry \(\epsilon\)=0.8]{\includegraphics[trim={0.5em 0em 3em 1em}, clip, width=0.22\textwidth]{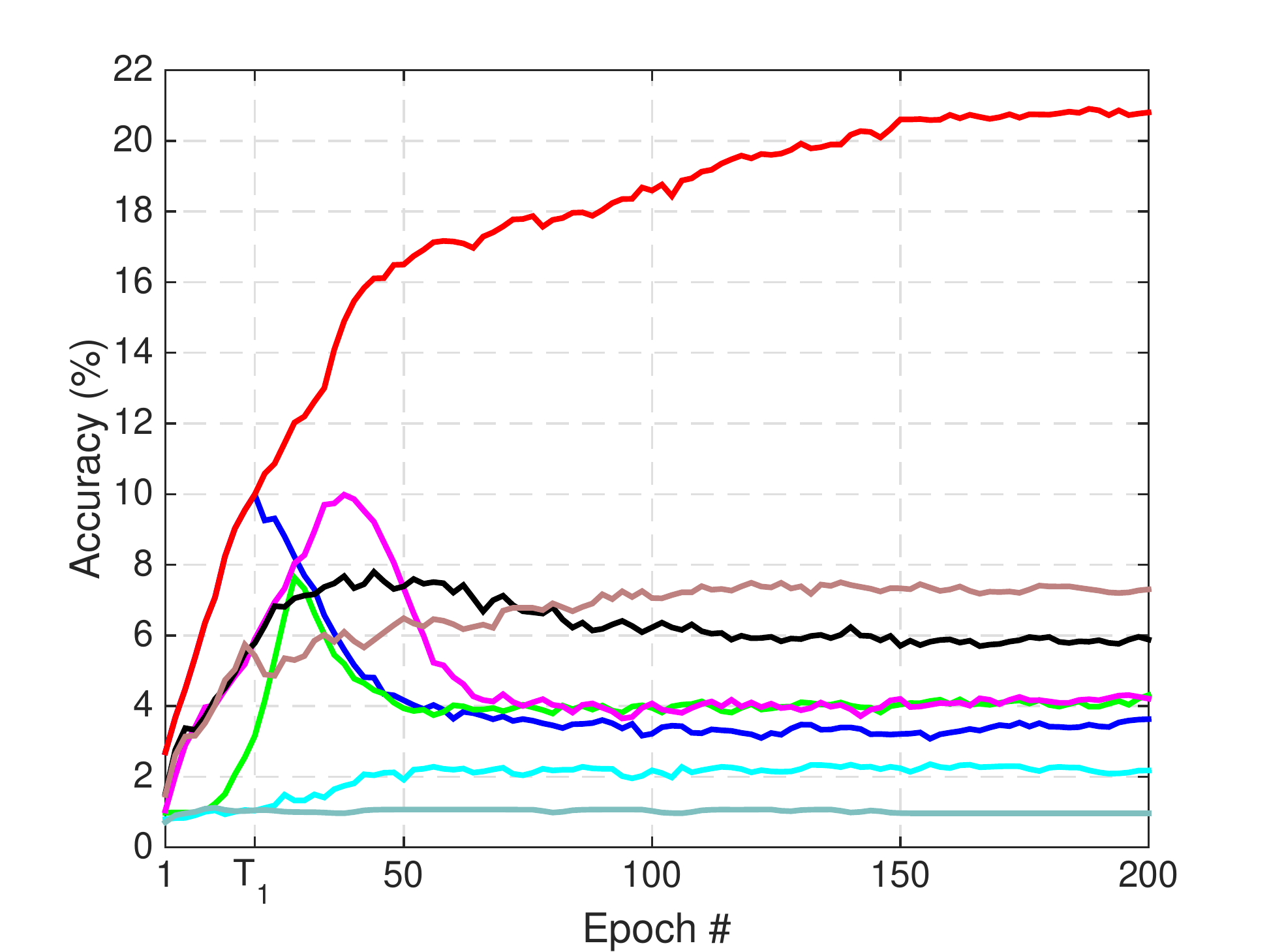}}
	\subfigure[ResNet18/Pair \(\epsilon\)=0.45]{\includegraphics[trim={0.5em 0em 3em 1em}, clip, width=0.22\textwidth]{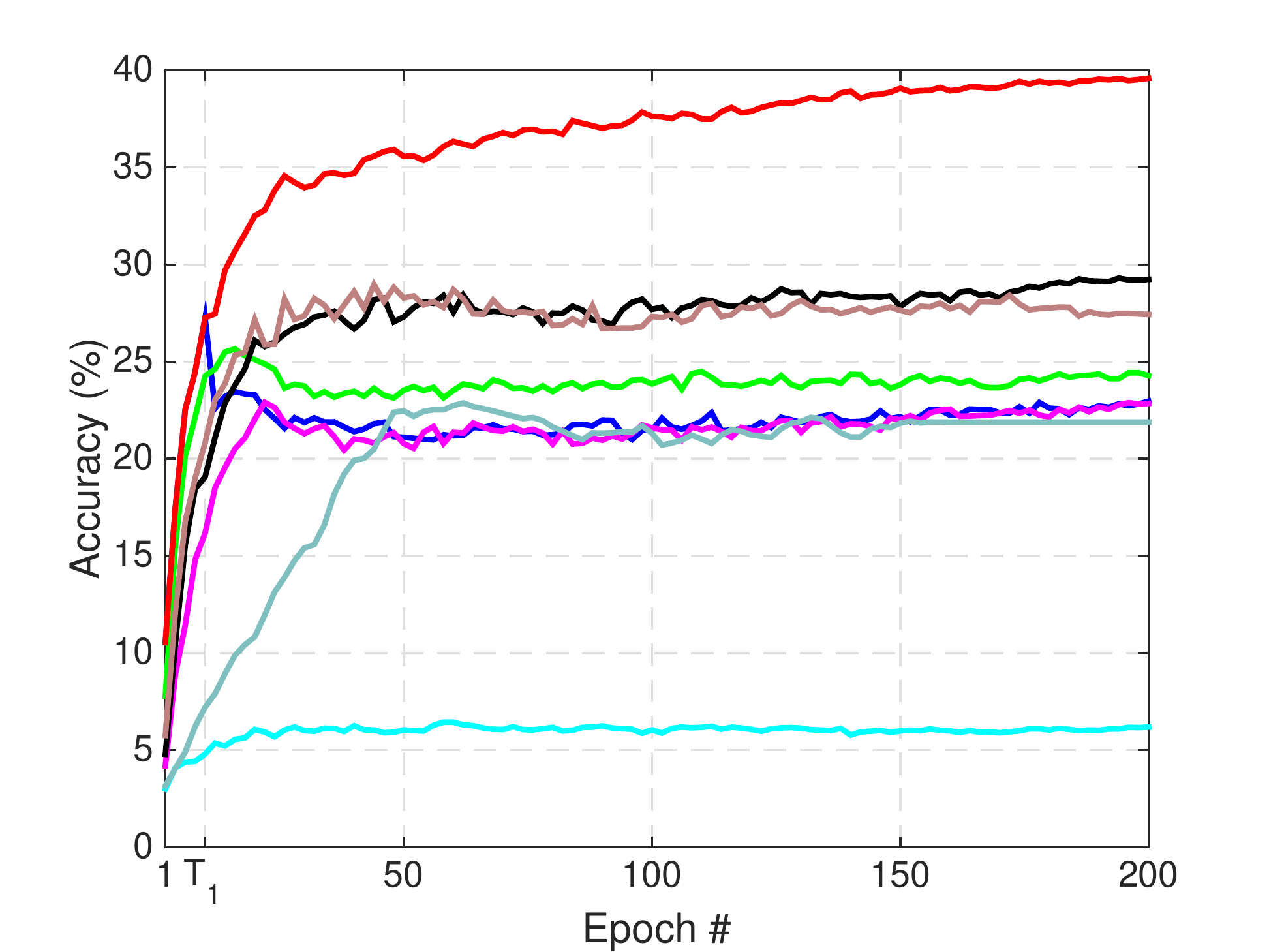}}
	\includegraphics[trim={39em 5em 36em 23em}, clip, width=0.1\textwidth]{figures/legend.pdf}
	\subfigure[ ConvNet/Symmetry \(\epsilon\)=0.2]{\includegraphics[trim={0.5em 0em 3em 1em}, clip, width=0.22\textwidth]{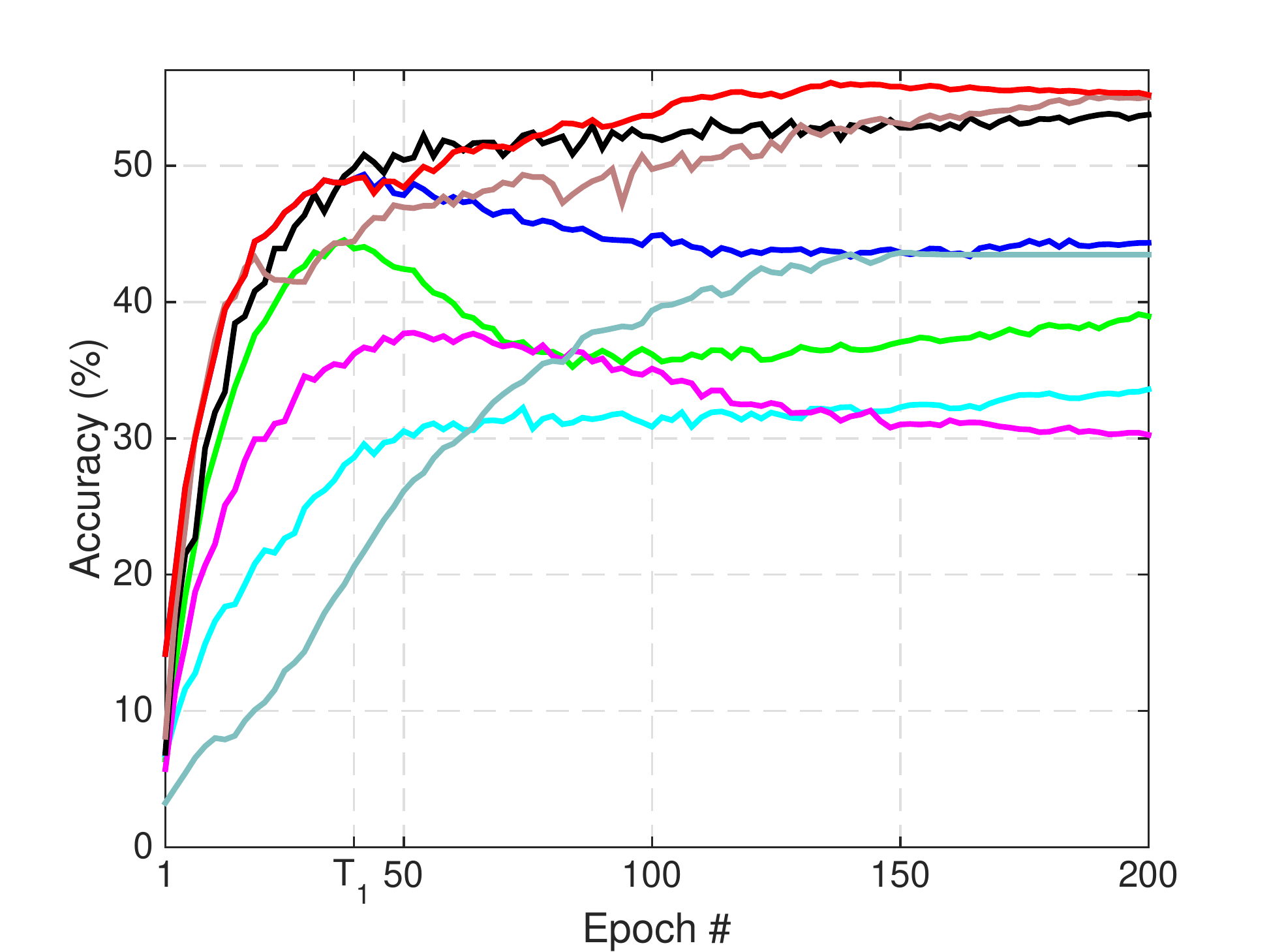}}
	\subfigure[ConvNet/Symmetry \(\epsilon\)=0.5]{\includegraphics[trim={0.5em 0em 3em 1em}, clip, width=0.22\textwidth]{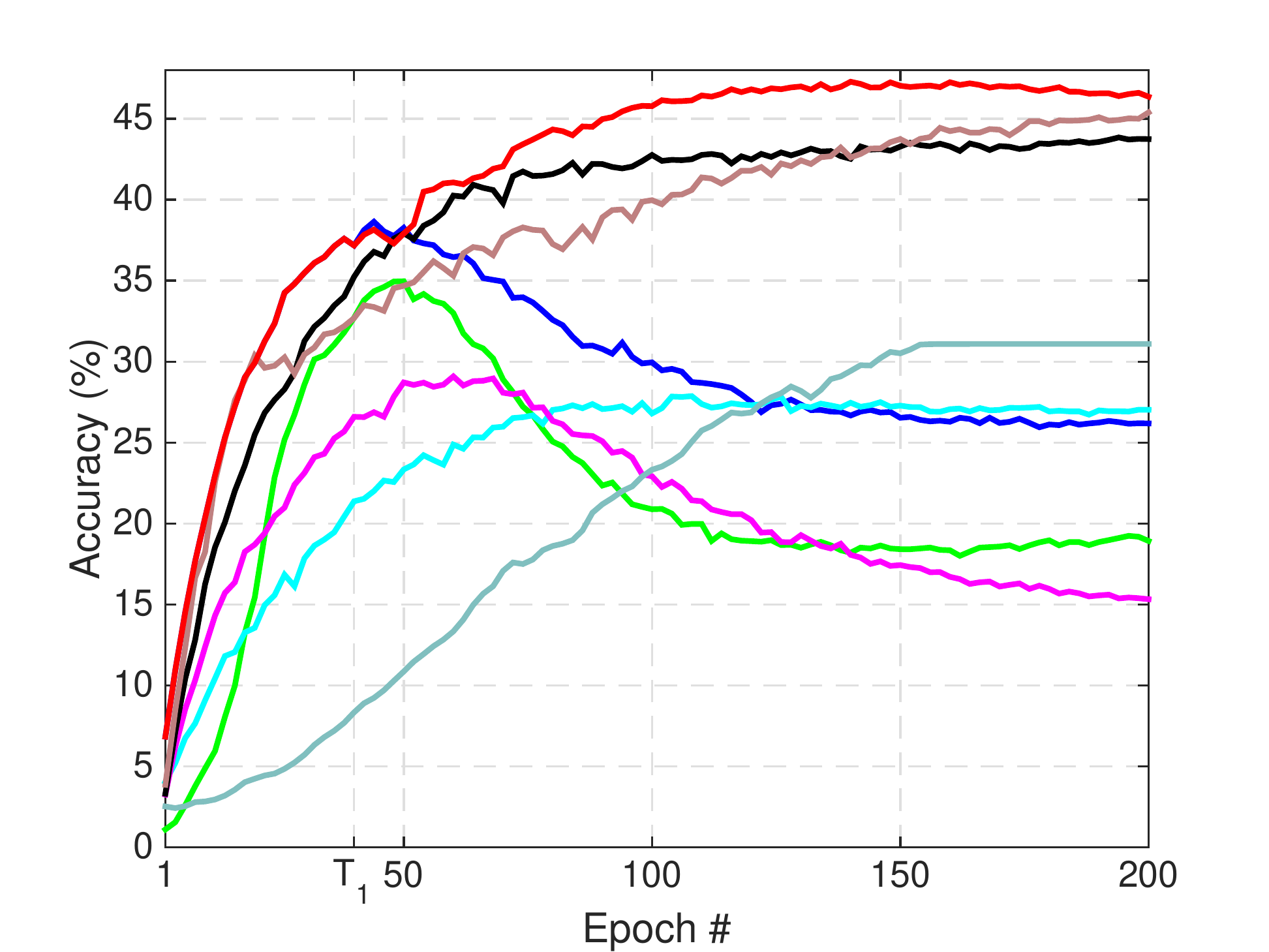}}
	\subfigure[ConvNet/Symmetry \(\epsilon\)=0.8]{\includegraphics[trim={0.5em 0em 3em 1em}, clip, width=0.22\textwidth]{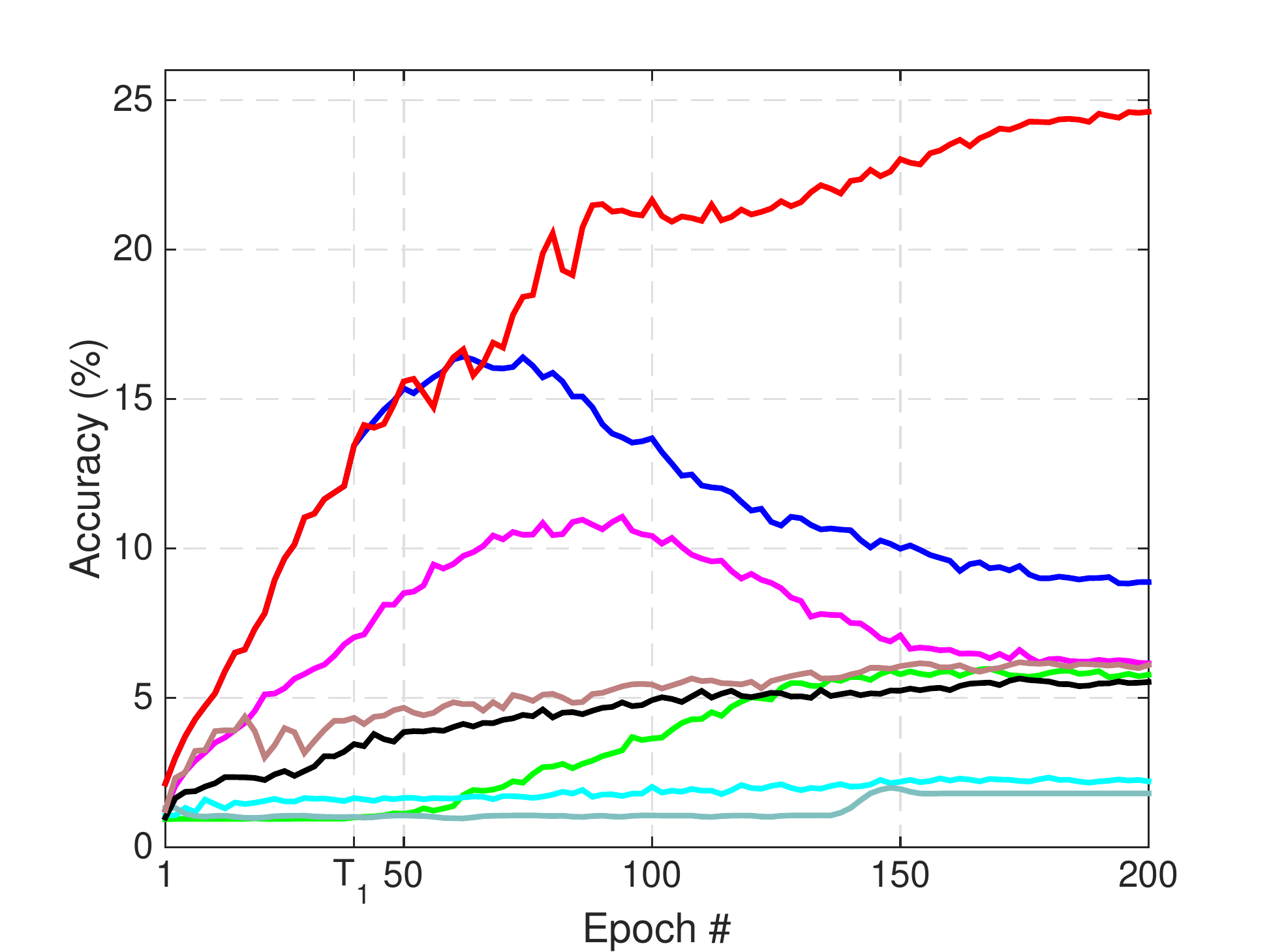}}
	\subfigure[ConvNet/Pair \(\epsilon\)=0.45]{\includegraphics[trim={0.5em 0em 3em 1em}, clip, width=0.22\textwidth]{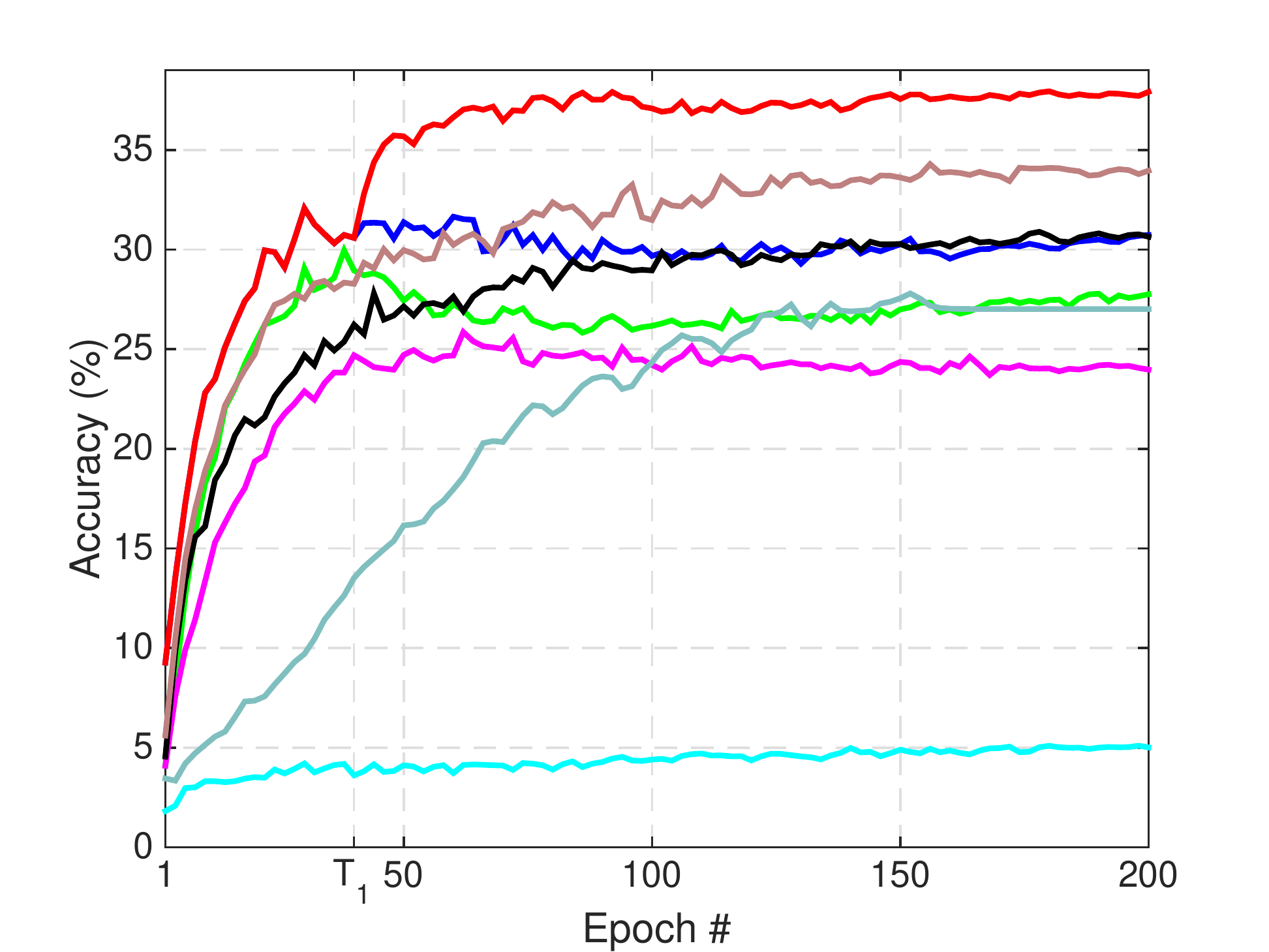}}
	\vspace{-1em}
	\caption{Testing accuracy of seven methods at different numbers of epochs on Mini-ImageNet.} 
	\label{fig:ImgNet}
	\vspace{-1em}
\end{figure*}

\section*{Deep Architectures}
\label{sec:Architectures}
Tables \ref{table:resnet18}-\ref{table:convnet} present the used network architectures of ResNet18 \cite{he2016deep} and ConvNet \cite{rasmus2015semi} \cite{laine2016temporal}, which are re-implemented with the PyTorch framework. A convolutional layer is represented by 'conv', and we display kernel size, stride and padding in brackets, and the number of kernels after a dash. The convolutional average-pooling layer is denoted by `avgpool', and the convolutional max-pooling layer is represented by `maxpool'. We provide the pooling size and stride in brackets. We utilize `fc' to denote the fully-connected layer and provide a number of output hidden units after a dash. The ReLU is used as the non-linearity function in ResNet18, and 'LReLU' denotes the leaky ReLU as the non-linearity in ConvNet and we provide the negative slope (\(\alpha\)=0.1) in brackets.  and \(c\) represents the number of classes.  Additionally, all data layers of ConvNet were initialized following \cite{he2015delving}.

\begin{figure*}[htb]\renewcommand\thefigure{A5}
	\includegraphics[trim={39em 5em 36em 23em}, clip, width=0.1\textwidth]{figures/legend.pdf}
	\subfigure[CIFAR-10/ResNet18]{\includegraphics[trim={0.5em 0em 3em 1em}, clip, width=0.22\textwidth]{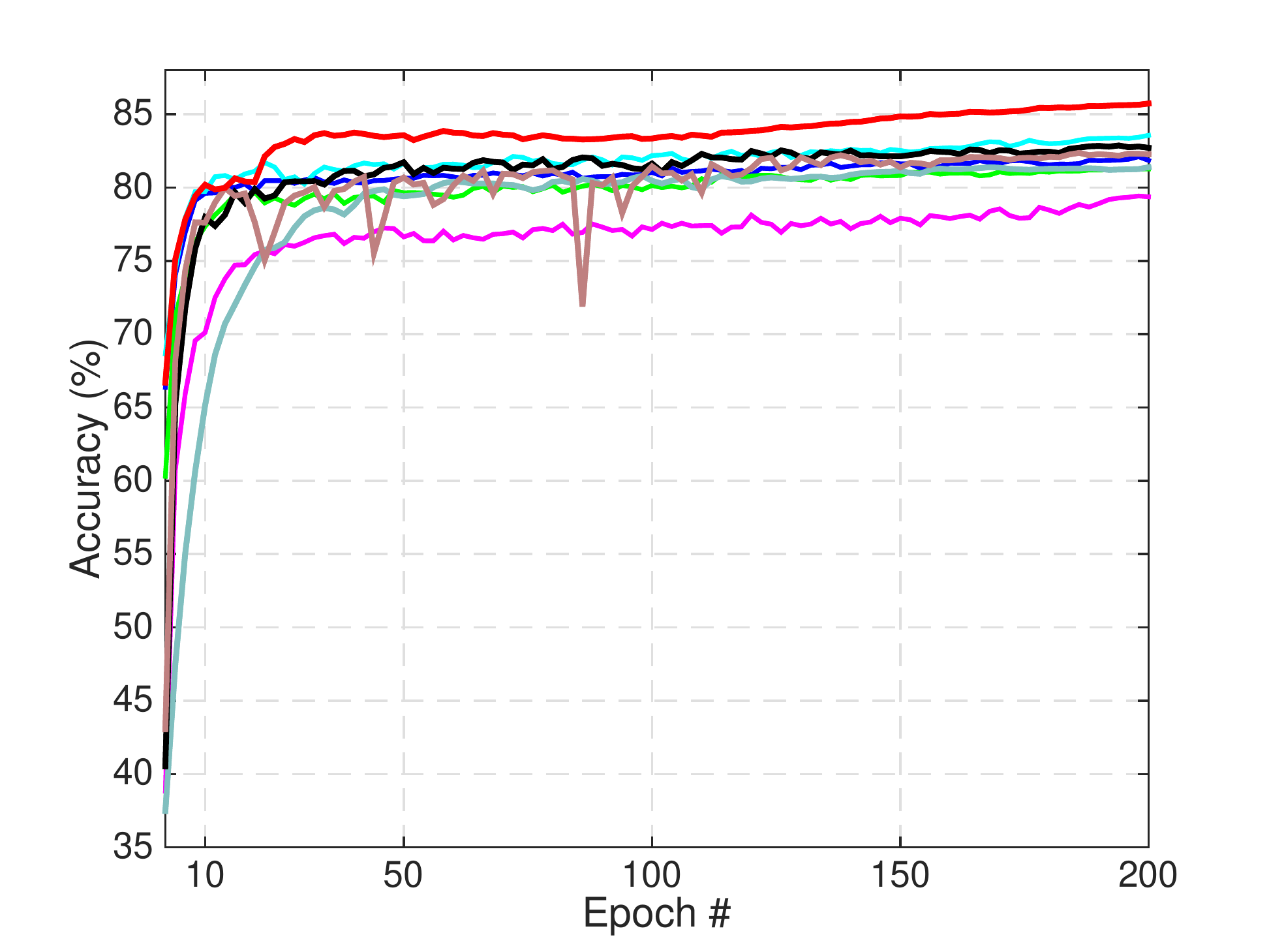}}
	\subfigure[CIFAR-10/ConvNet]{\includegraphics[trim={0.5em 0em 3em 1em}, clip, width=0.22\textwidth]{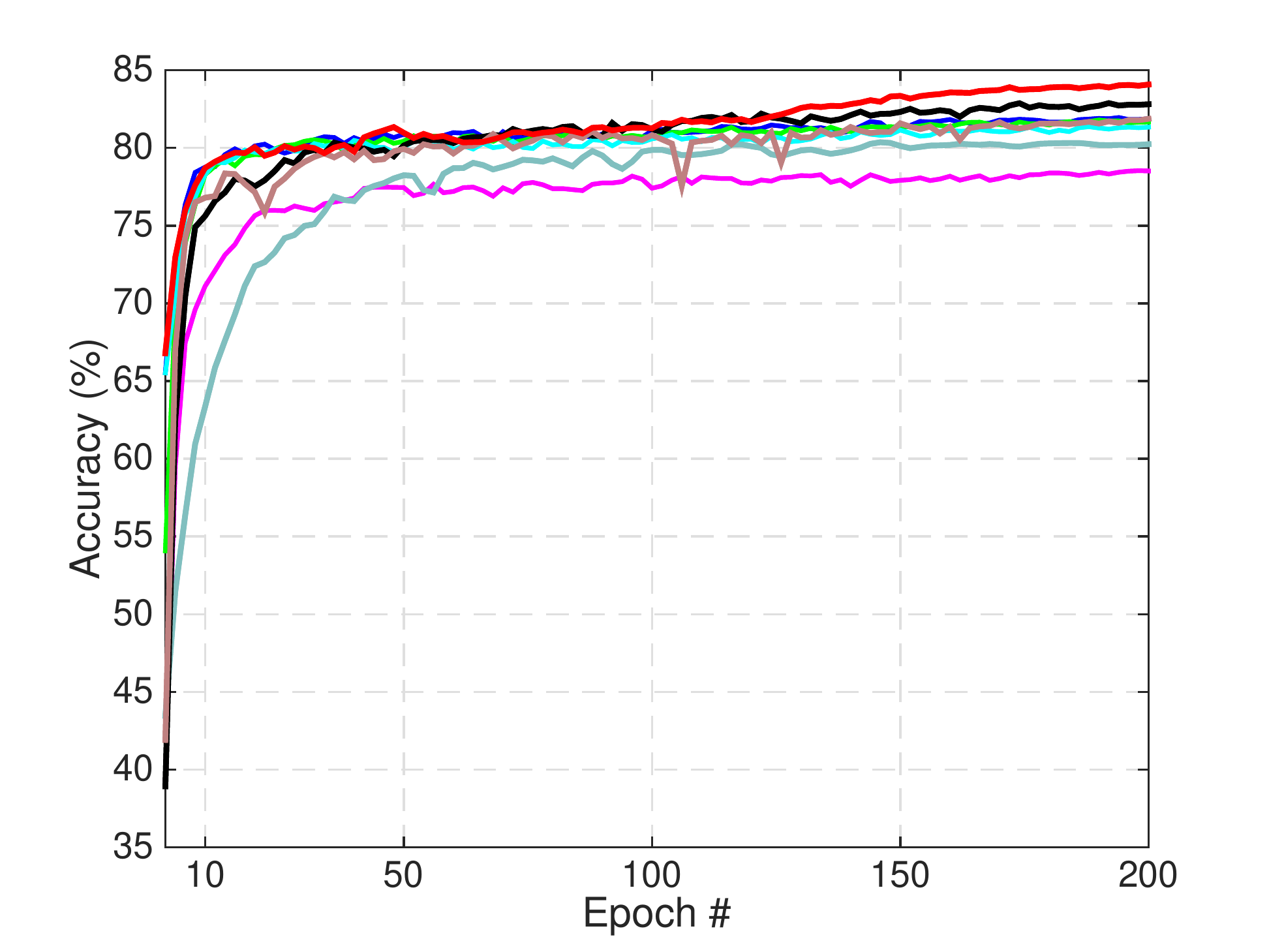}}
	\subfigure[CIFAR-100/ResNet18]{\includegraphics[trim={0.5em 0em 3em 1em}, clip, width=0.22\textwidth]{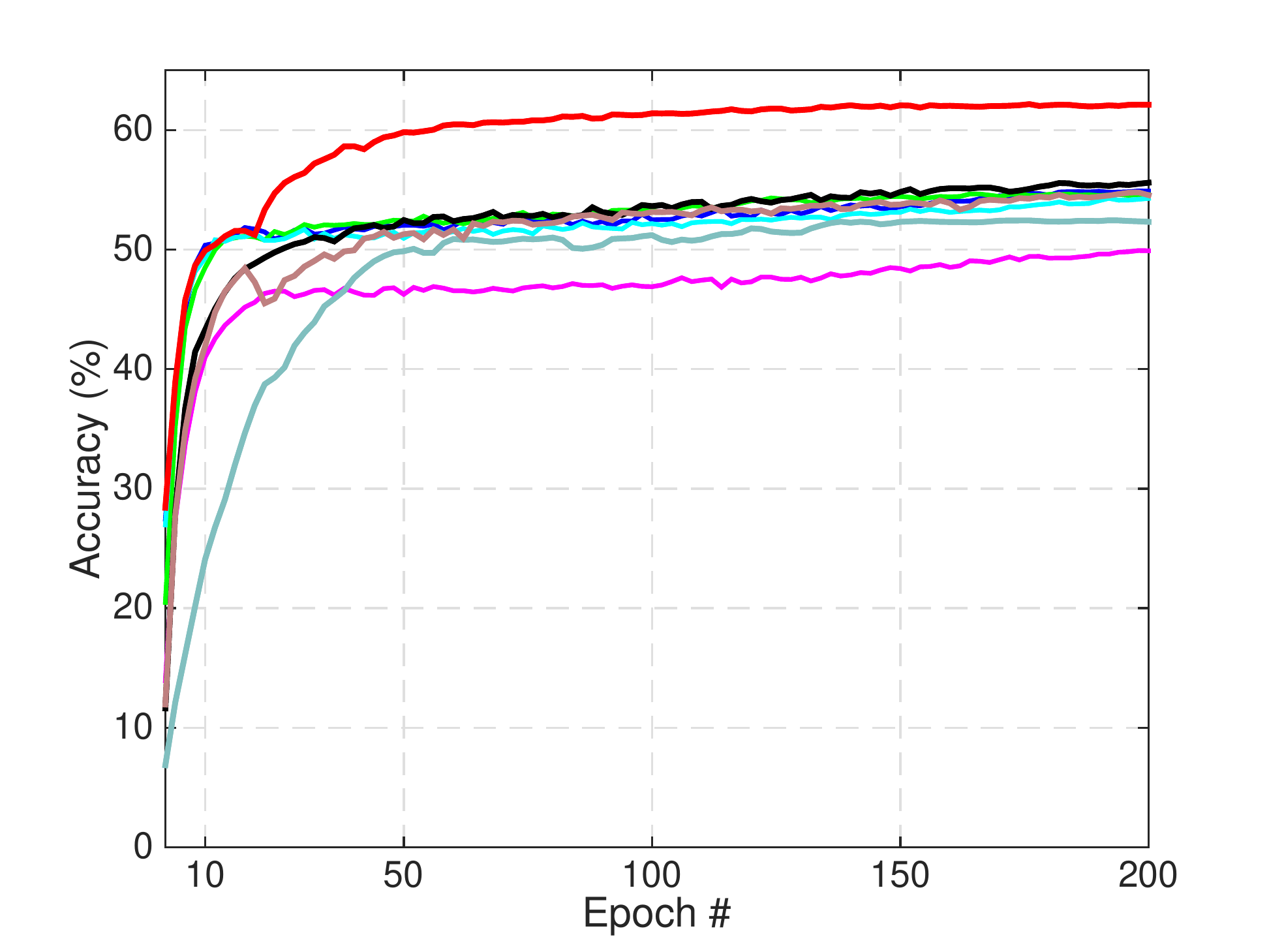}}
	\subfigure[CIFAR-100/ConvNet]{\includegraphics[trim={0.5em 0em 3em 1em}, clip, width=0.22\textwidth]{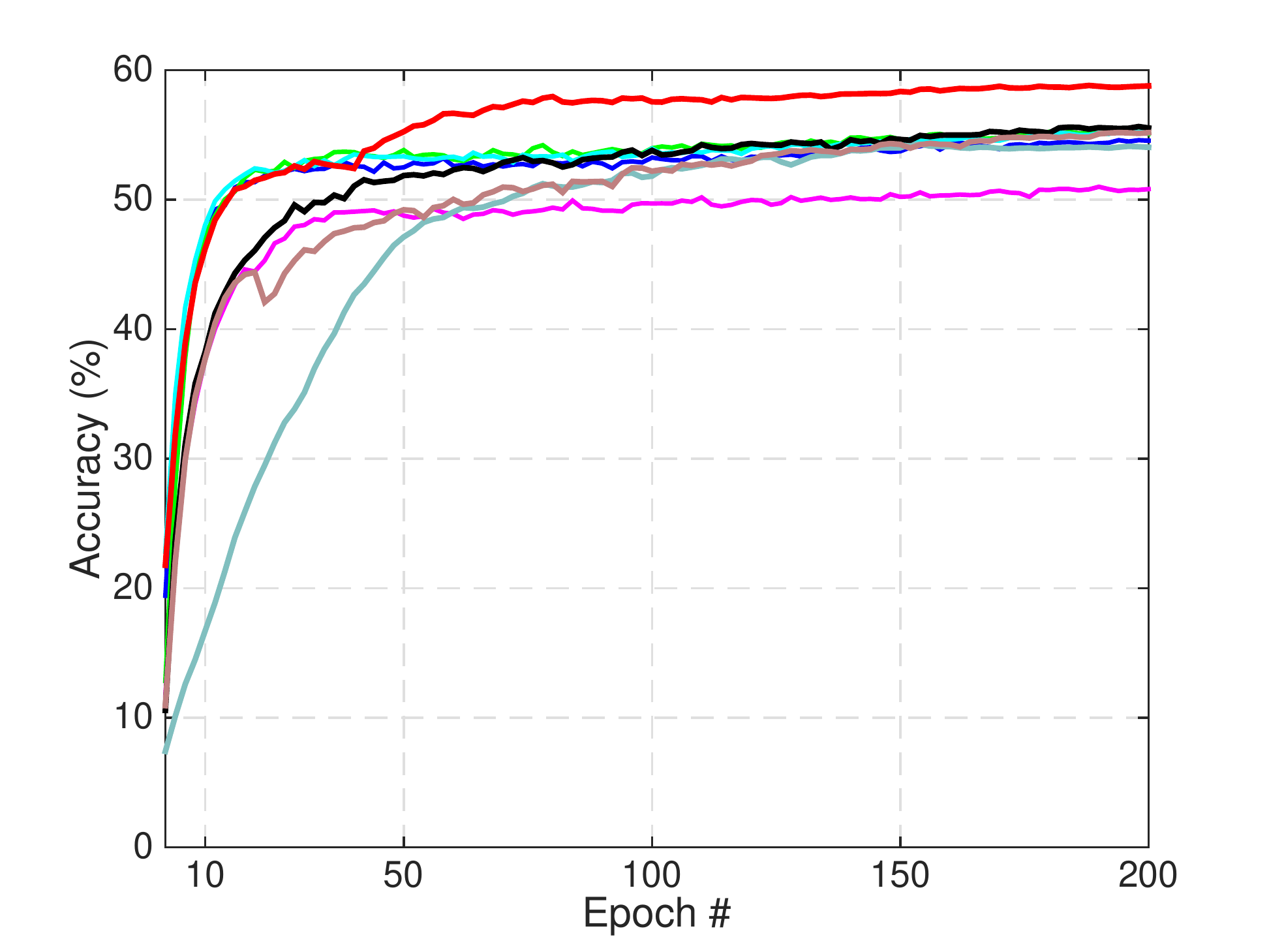}}
	\vspace{-1em}
	\caption{Testing accuracy of seven methods with noisy labels generated by CNNs on CIFAR-10 and CIFAR-100 at different numbers of training epochs.} 
	\label{fig:semi}
	\vspace{-1em}
\end{figure*}

\begin{figure*}[htb]\renewcommand\thefigure{A6}
	\includegraphics[trim={20em 6em 18em 27em}, clip, width=0.1\textwidth]{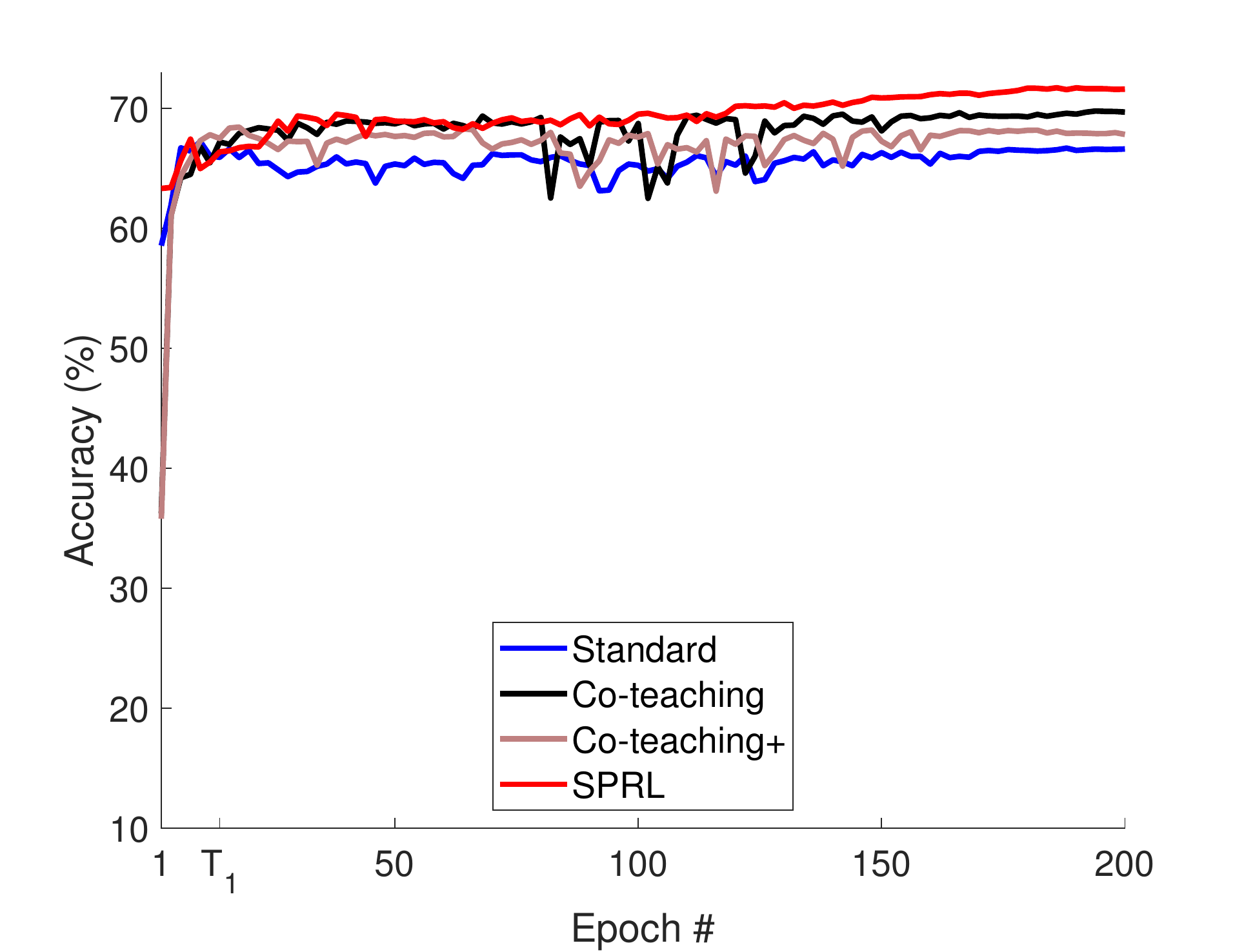}
	\subfigure[Food101/ResNet18]{\includegraphics[trim={0.5em 0em 3em 1em}, clip, width=0.22\textwidth]{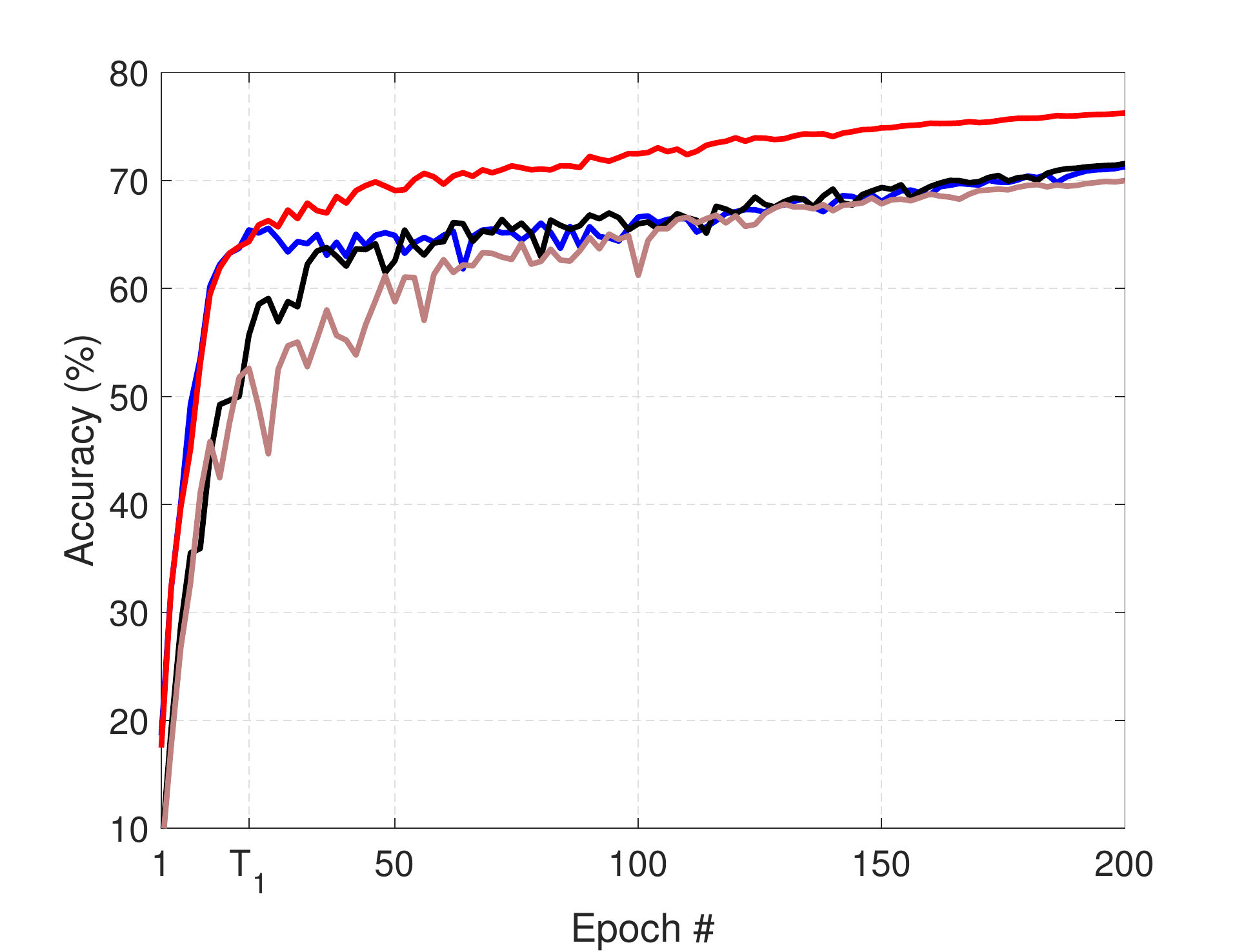}}
	\subfigure[Food101/ConvNet]{\includegraphics[trim={0.5em 0em 3em 1em}, clip, width=0.22\textwidth]{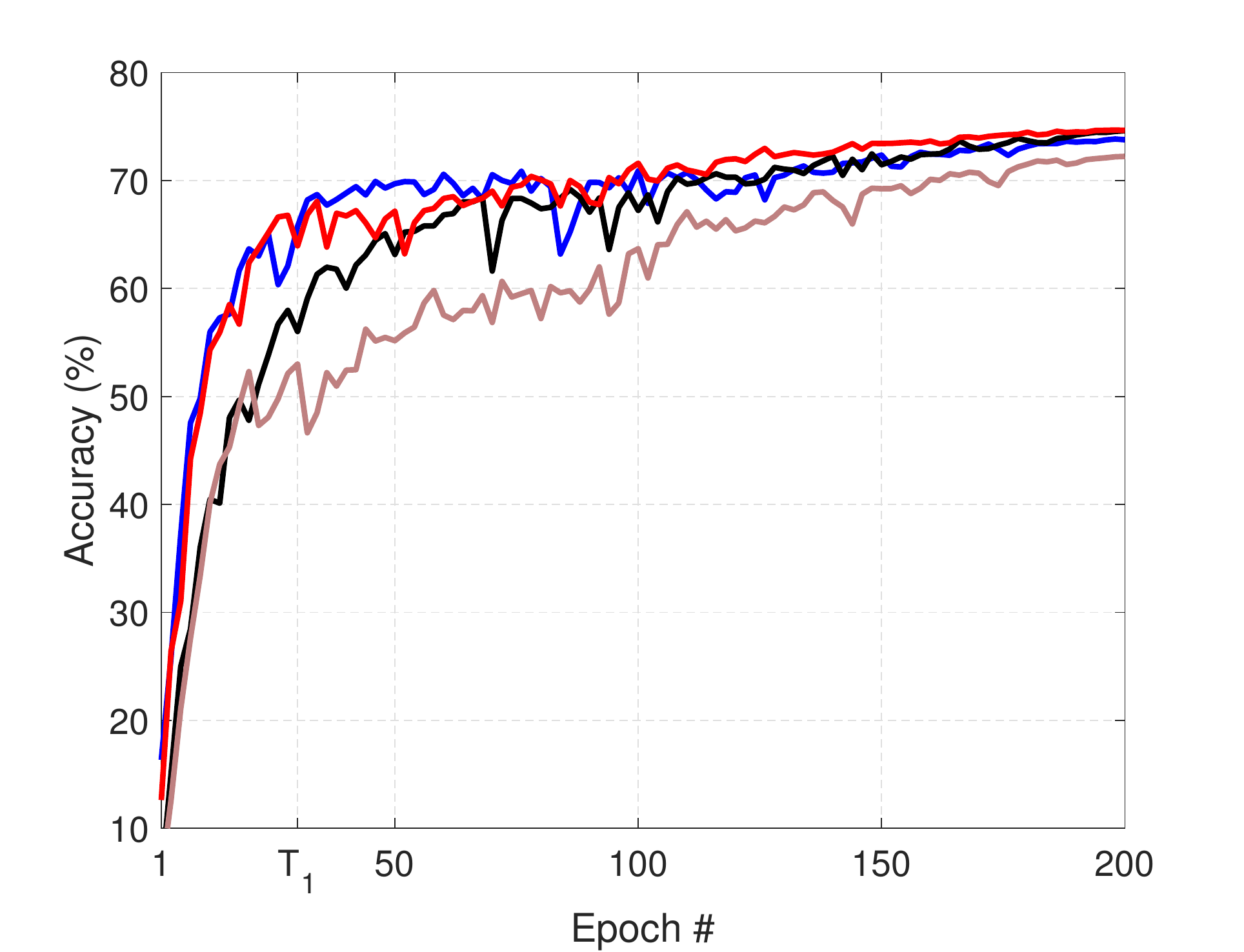}}
	\subfigure[Cloth1M/ResNet18]{\includegraphics[trim={0.5em 0em 3em 1em}, clip, width=0.22\textwidth]{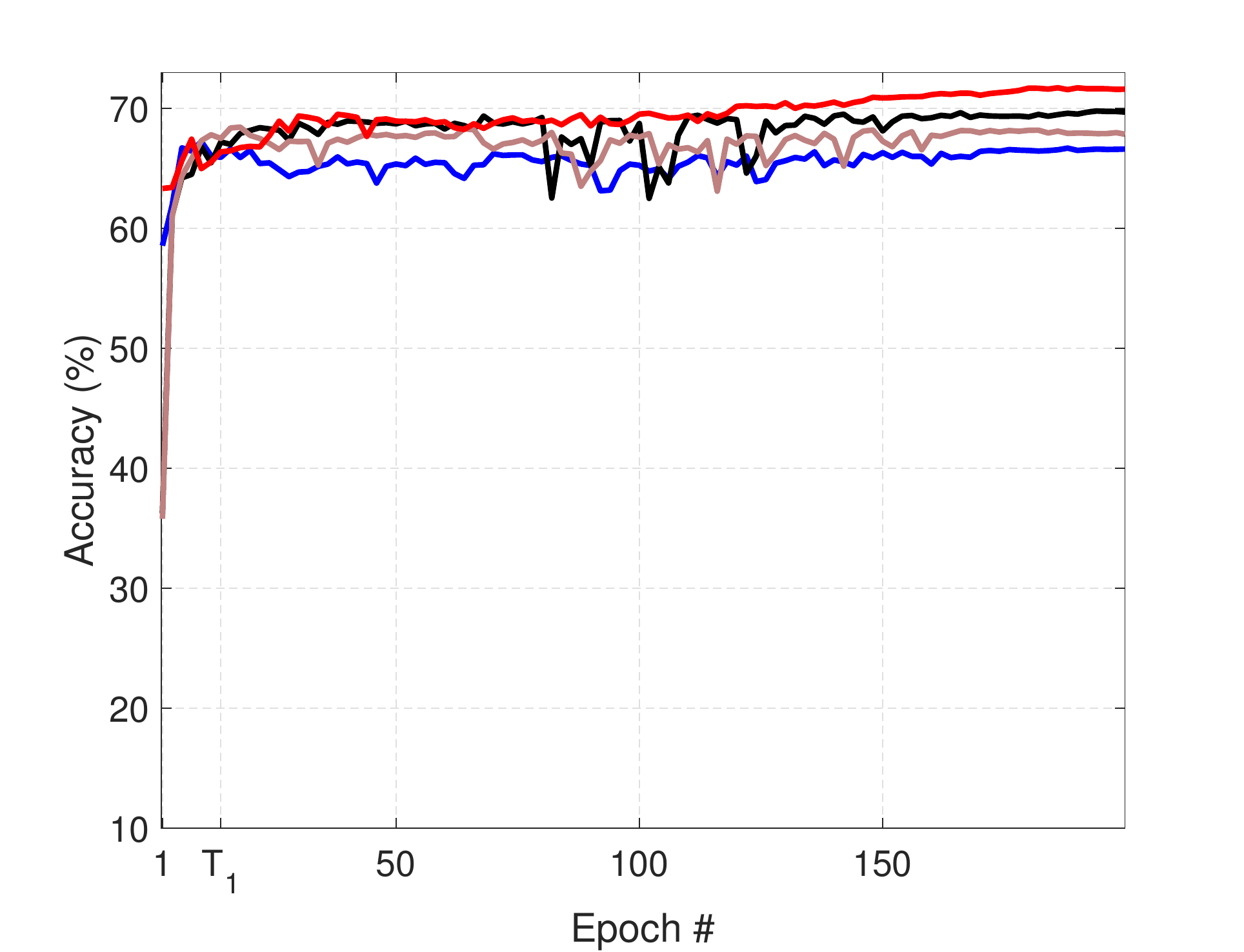}}
	\subfigure[Cloth1M/ConvNet]{\includegraphics[trim={0.5em 0em 3em 1em}, clip, width=0.22\textwidth]{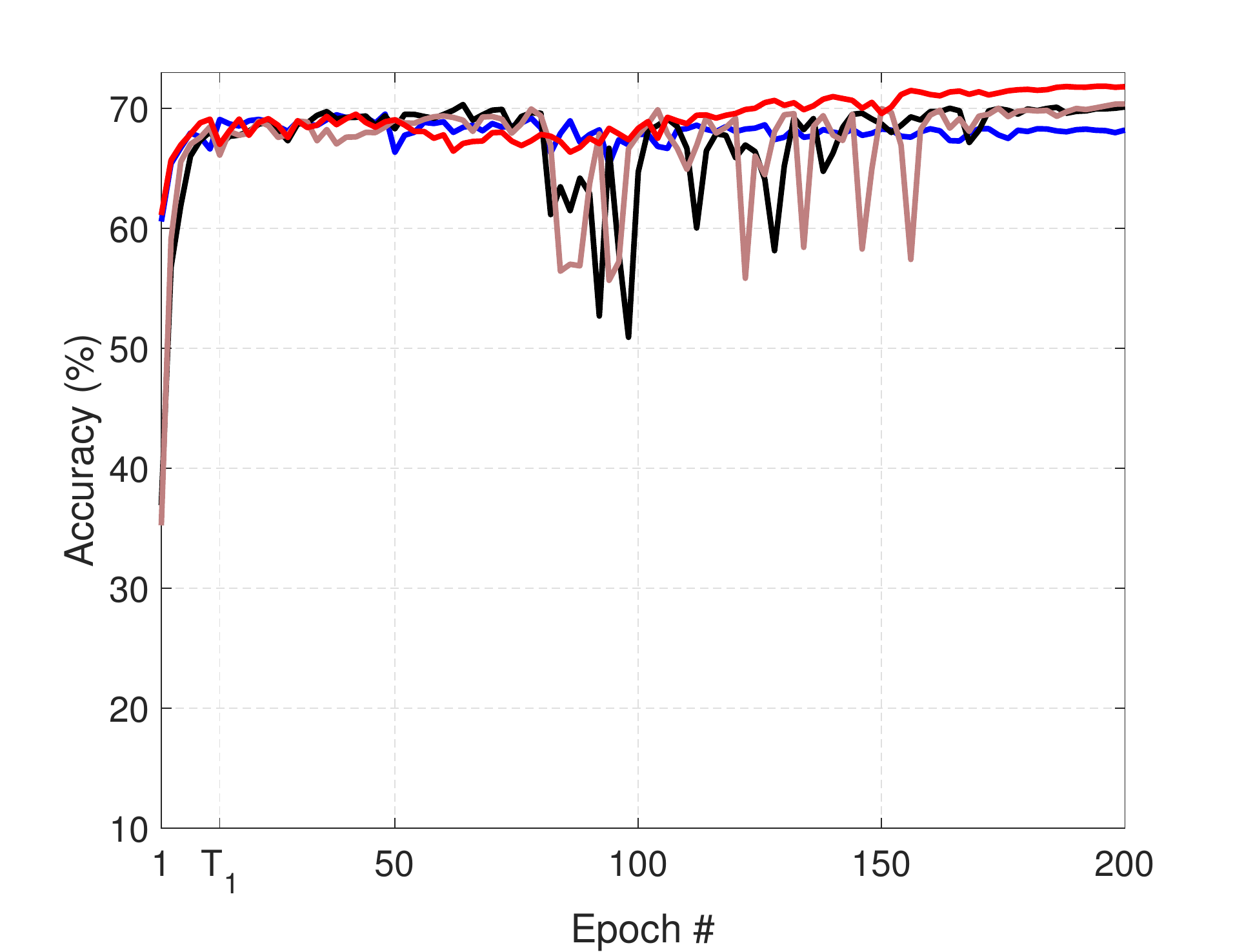}}
	\vspace{-1em}
	\caption{Testing accuracy of four methods with real-world noisy labels at different numbers of training epochs.}
	\label{fig:real}
	\vspace{-1em}
\end{figure*}

\section*{Testing Accuracy without Data Augmentation}
Table \ref{table:noaug} presents the average of testing accuracy (\%) of SPRL, Co-teaching and Co-teaching+ on CIFAR-10 and CIFAR-100 using ResNet18 and without using data augmentation. It further demonstrates that SPRL outperforms the best competitors Co-teaching and Co-teaching+ even without using data augmentation.

\section*{Parameter Settings}
Here, we present the settings of two essential parameters \(T_{1}\) and \(\gamma_d\) in SPRL. Table \ref{table:noisy} shows their values on MNIST, CIFAR-10, CIFAR-100 and Mini-ImageNet when using noisy labels generated by symmetric and pair flipping; Table \ref{table:semi} displays the values on CIFAR-10 and CIFAR-100 when using noisy labels generated by CNNs.

\section*{Testing Accuracy vs. Training Epochs}
Figs. \ref{fig:mnist}-\ref{fig:ImgNet}  present testing accuracy of the eight methods at different numbers of training epochs on MNIST, CIFAR-10, CIFAR-100 and ImageNet. It is worth noting that the accuracy of Co-teaching with ConvNet is drastically fluctuating during training on CIFAR-10 and CIFAR-100, while Co-teaching with ResNet18 can achieve stable accuracy on these two datasets. This might be caused by that we set the dropout rate to 0.5 instead of 0.25 in ConvNet.  Fig. \ref{fig:semi} displays testing accuracy of the eight methods at different  numbers of training epochs on CIFAR-10 and CIFAR-100 when using noisy labels generated by CNNs. Fig. \ref{fig:real} presents testing accuracy of Standard, Co-teaching, Co-teaching+ and SPRL on Food101 and Cloth1M with real-world noisy labels at different numbers of training epochs.

\section*{Accuracy of Selected Confident Samples}
Fig. \ref{fig:stds} shows the accuracy of selected confident samples of ResNet18, and it suggests that the network might first memorize the probably correct-label data
and then corrupt-label samples. Additionally, Fig. \ref{fig:sl} presents the accuracy of selected confident samples of Co-teaching on different levels
of noisy labels. It infers that the selection accuracy of Co-teaching is signifcantly decreased on extremely noisy labels.

\begin{table}[tb]\renewcommand\thetable{A3}
	\scriptsize 
	\centering
	\caption{Parameters for SPRL using symmetric and pair noisy labels on MNIST, CIFAR-10, CIFAR-100 and Mini-ImageNet.}
	\vspace{-0.5em}
	\begin{tabular}{|c|c|c|c|c|c|c|c|}
		\hline
		\multicolumn{8}{|c|}{ResNet18}  \\      
		\hline
		\multicolumn{6}{|c|}{Symmetry}   & \multicolumn{2}{c|}{Pair}    \\
		\hline
		\multicolumn{2}{|c|}{\(\epsilon=0.2\)} & \multicolumn{2}{c|}{\(\epsilon=0.5\)}  & \multicolumn{2}{c|}{\(\epsilon=0.8\)}   & \multicolumn{2}{c|}{\(\epsilon=0.45\)}   \\
		\hline
		\(T_1\) & \(\gamma_d\) & \(T_1\) & \(\gamma_d\) & \(T_1\) & \(\gamma_d\) & \(T_1\) & \(\gamma_d\)   \\
		\hline
		\multicolumn{8}{|c|}{MNIST}  \\
		\hline
		15 & 300 & 15 & 300 & 15 & 300 & 15 & 300  \\
		\hline
		\multicolumn{8}{|c|}{CIFAR-10}  \\
		\hline
		20 & 10 & 20 & 10 &20 &10 & 20 &50 \\
		\hline
		\multicolumn{8}{|c|}{CIFAR-100}  \\
		\hline
		20 & 10 & 20 & 10 &20 &10 &20 &50    \\
		\hline
		\multicolumn{8}{|c|}{Mini-ImageNet}  \\
		\hline
		15& 10 &15 &10 &20 &10 &10 &50  \\
		\hline
		\hline
		\multicolumn{8}{|c|}{ConvNet} \\
		\hline
		\multicolumn{6}{|c|}{Symmetry}   & \multicolumn{2}{c|}{Pair}  \\
		\hline
		\multicolumn{2}{|c|}{\(\epsilon=0.2\)}  & \multicolumn{2}{c|}{\(\epsilon=0.5\)}  & \multicolumn{2}{c|}{\(\epsilon=0.8\)}   & \multicolumn{2}{c|}{\(\epsilon=0.45\)}  \\
		\hline
		\(T_1\) & \(\gamma_d\) & \(T_1\) & \(\gamma_d\) & \(T_1\) & \(\gamma_d\) & \(T_1\) & \(\gamma_d\) \\
		\hline
		\multicolumn{8}{|c|}{MNIST}  \\
		\hline
		15 & 300 & 15 & 300 & 15  &300 & 15 & 300 \\
		\hline
		\multicolumn{8}{|c|}{CIFAR-10}  \\
		\hline
		40 &5 & 40 & 5 &40 &5  & 40 & 50 \\
		\hline
		\multicolumn{8}{|c|}{CIFAR-100}  \\
		\hline
		40 &5 &40 &5 &40 &5  &40 &50  \\
		\hline
		\multicolumn{8}{|c|}{Mini-ImageNet}  \\
		\hline
		40 &5 &40 &5 &40 &5  &40 &50  \\
		\hline
	\end{tabular}
	\label{table:noisy}
\end{table}

\begin{table}[htb] \renewcommand\thetable{A4}
	\scriptsize 
	\centering
	\caption{Parameters for SPRL using noisy labels generated by CNNs on CIFAR-10 and CIFAR-100.}
	\vspace{-0.5em}
	\begin{tabular}{|c|c|c|c|}
		\hline
		\multicolumn{2}{|c|}{ResNet18} & \multicolumn{2}{c|}{ConvNet}  \\      
		\hline
		\(T_1\) & \(\gamma_d\) & \(T_1\) & \(\gamma_d\)  \\
		\hline
		\multicolumn{4}{|c|}{CIFAR-10}  \\
		\hline
		10 & 10 & 10 & 10   \\
		\hline
		\multicolumn{4}{|c|}{CIFAR-100}  \\
		\hline
		20& 10 & 20 & 10   \\
		\hline
	\end{tabular}
	\label{table:semi}
\end{table}

\begin{figure}[tbp]\renewcommand\thefigure{A7}
	\center
	\includegraphics[trim={0em 0em 0em 0em}, clip, width=0.4\textwidth]{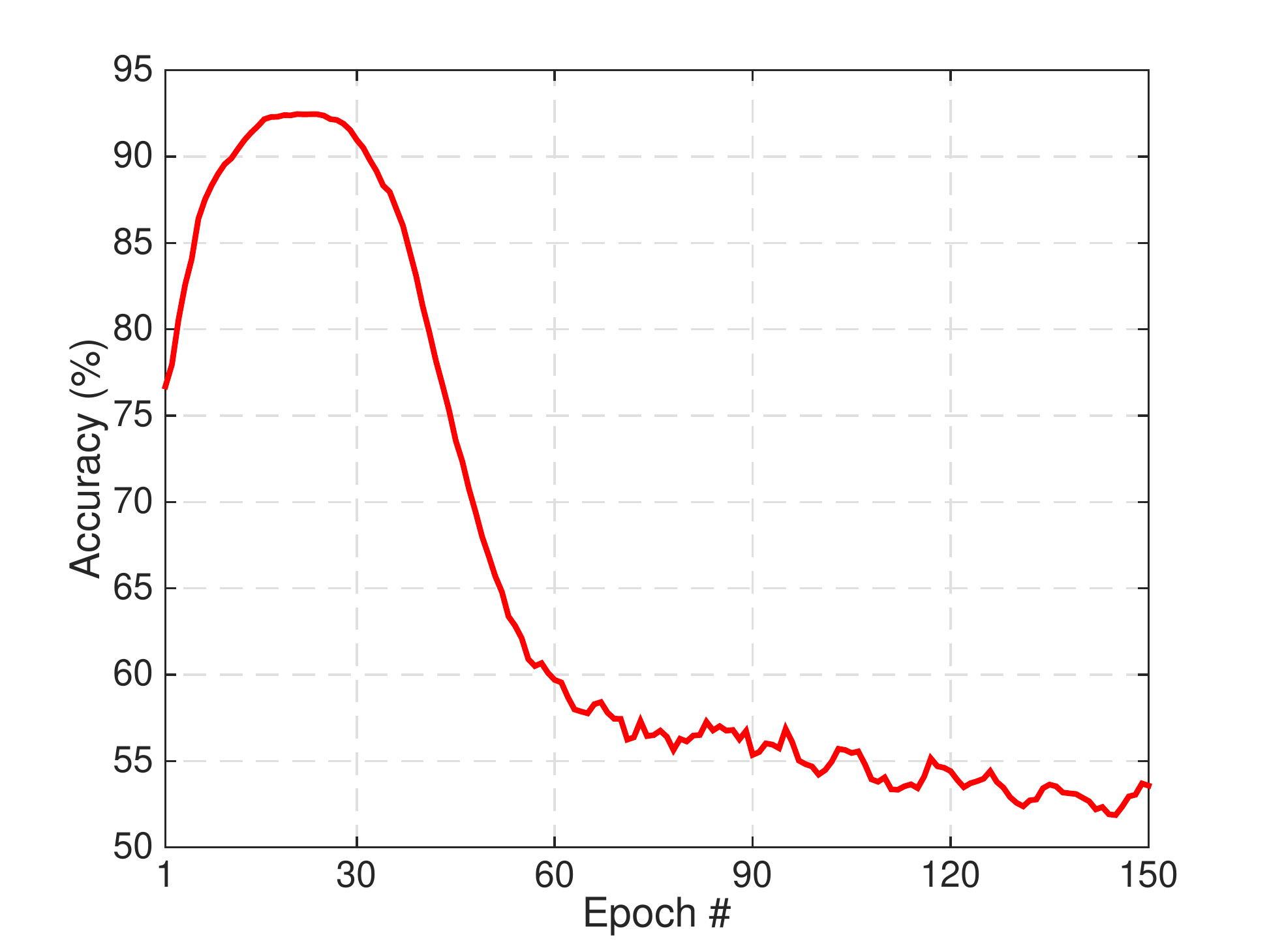}
	\vspace{-0.5em}
	\caption{The accuracy of selected confident samples from CIFAR-10 by using standard ResNet18 with symmetric label noise \(\epsilon=0.5\). Note that we only select 50\% training samples as confident ones. } 
	\label{fig:stds}
	\vspace{-1em}
\end{figure}

\begin{figure}[tbp]\renewcommand\thefigure{A8}
	\center
	\includegraphics[trim={0em 0em 0em 0em}, clip, width=0.4\textwidth]{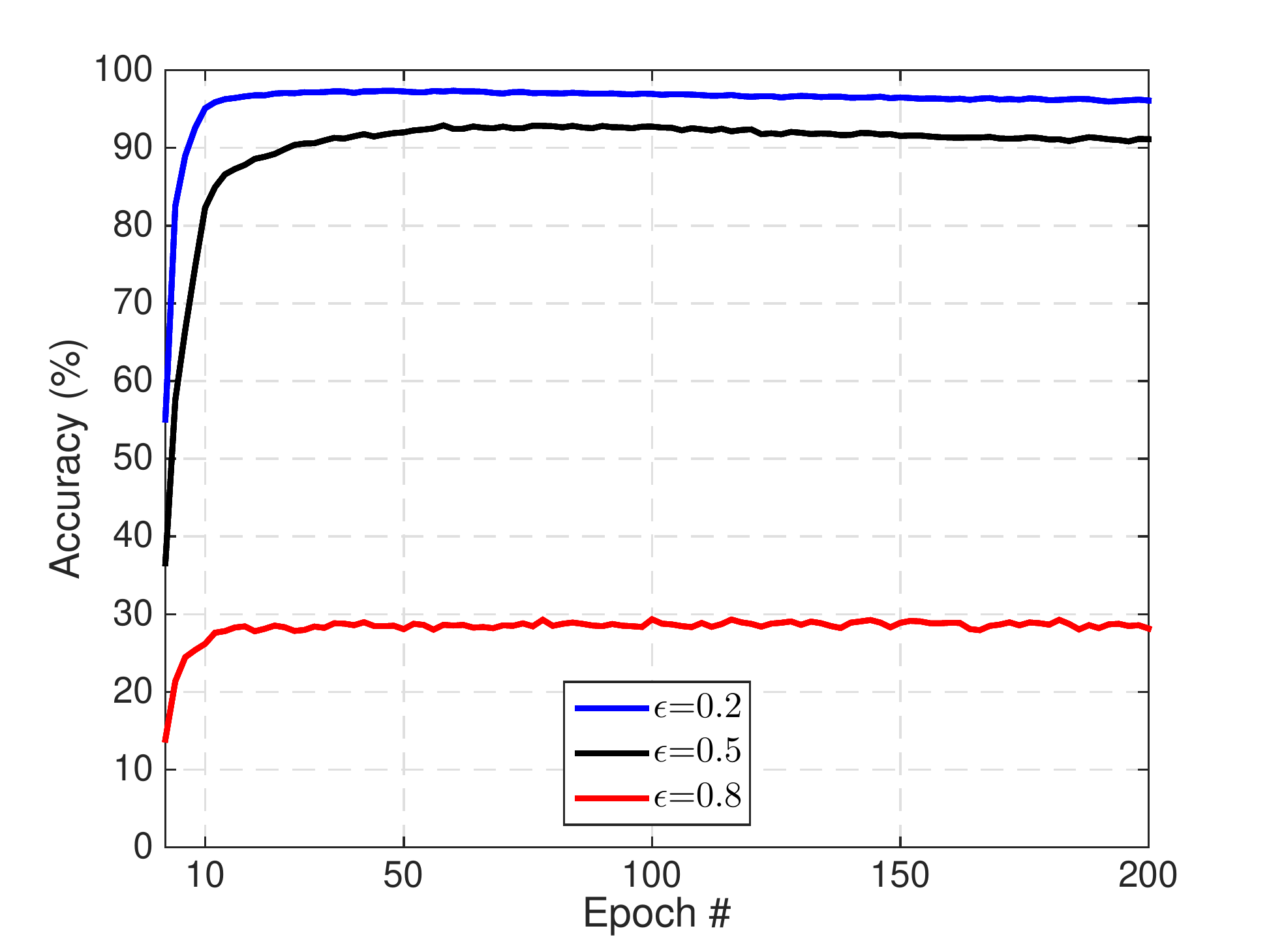}
	\vspace{-0.5em}
	\caption{The accuracy of selected confident samples from CIFAR-10 at different levels of symmetric label noise.} 
	\label{fig:sl}
	\vspace{-1em}
\end{figure}

\end{document}